\documentclass[10pt,journal,doublecolumn,twoside]{IEEEtran}

\ifCLASSOPTIONcompsoc

\usepackage[nocompress]{cite}
\else
\usepackage{cite}
\fi

\ifCLASSINFOpdf

\else

\fi
\usepackage{afterpage}
\usepackage{amssymb}
\usepackage{lipsum}
\usepackage{amsmath,lipsum}
\usepackage{cuted}
\usepackage{graphicx}
\usepackage{epstopdf}
\usepackage{amsmath}
\usepackage{subeqnarray}
\usepackage{cases}
\usepackage{algorithmic}
\usepackage{algorithm}
\usepackage{array}
\usepackage{multicol}
\usepackage{stfloats}

\usepackage{url}
\usepackage{harpoon}
\usepackage{booktabs}
\usepackage{multirow}
\usepackage{amsfonts}
\usepackage{xcolor}
\usepackage{tabularx}
\usepackage{amssymb,varwidth}
\usepackage{bm}
\usepackage{float}
\usepackage{graphicx}
\usepackage{subfigure}
\usepackage{ntheorem}
\usepackage{mathtools}
\allowdisplaybreaks[3]
\newenvironment{proof}{{\indent \it Proof:\quad}}{\hfill $\blacksquare$\par}
\theoremseparator{:}

\newtheorem{theorem}{Theorem}
\newtheorem{assumption}{\textbf{Assumption}}
\newtheorem{lemma}{\textbf{Lemma}}

\newtheorem{corollary}{\textbf{Corollary}}
\begin{document}
\begin{sloppypar}

\bibliographystyle{IEEEtran}

\title{Energy-Efficient Federated Learning for Edge Real-Time Vision via Joint Data, Computation, and Communication Design}

\author{Xiangwang Hou, \emph{Member, IEEE,}  Jingjing Wang, \emph{Senior Member, IEEE}, \\ Fangming Guan, \emph{Graduate Student Member, IEEE,} Jun Du, \emph{Senior Member, IEEE}, \\ Chunxiao Jiang, \emph{Fellow, IEEE},   Yong Ren, \emph{Senior Member, IEEE}

 \thanks{ This work of Xiangwang Hou was supported by the National Natural Science Foundation of China under grant No. 623B2060. This work of Jingjing Wang was partly supported by the National Natural Science Foundation of China under Grant No. 62222101 and No. U24A20213, partly supported by the Beijing Natural Science Foundation under Grants No. L232043 and No. L222039, partly supported by the Natural Science Foundation of Zhejiang Province under Grant No. LMS25F010007 and partly supported by the Fundamental Research Funds for the Central Universities. This work of Jun Du was partly supported by the National Natural Science Foundation China under Grants No. 62422109 and No.U23A20281.  This work of Chunxiao Jiang was partly supported by the National Natural Science Foundation of China under Grants No. 62325108 and No. 62341131. This work of Yong Ren  was partly supported by the National Natural Science Foundation of China under Grants No. 62127801 and No. U24A20213, and  in part by the project `The Verification Platform of Multi-tier Coverage Communication Network for Oceans (LZC0020)' of Peng Cheng Laboratory. (Corresponding author: Chunxiao Jiang.)}

\thanks{X. Hou and F. Guan are with the Department of Electronic Engineering, Tsinghua University, Beijing 100084, China. (e-mail: xiangwanghou@163.com, gfm23@mails.tsinghua.edu.cn)}

\thanks{J. Wang is with the School of Cyber Science and Technology, Beihang University, Beijing 100191, China, and also with Hangzhou Innovation Institute, Beihang University, Hangzhou 310051, China (e-mail: drwangjj@buaa.edu.cn).}
\thanks{J. Du is with the Department of Electronic Engineering and also with the State Key Laboratory of Space
Network and Communications, Tsinghua University, Beijing 100084, China
(e-mail: jundu@tsinghua.edu.cn).}
\thanks{C. Jiang is with the Beijing National Research Center for Information Science and Technology, and the State Key Laboratory of Space Network and Communications,  Tsinghua University, Beijing 100084, China (e-mail: jchx@tsinghua.edu.cn).}

\thanks{Y. Ren is with the Department of Electronic Engineering and the State Key Laboratory of Space Network and Communications, Tsinghua University, Beijing 100084, China, and also with the Network and Communication Research Center, Peng Cheng Laboratory, Shenzhen 518055, China (e-mail: reny@tsinghua.edu.cn).}
}

\markboth{}
{Hou \MakeLowercase{\textit{et al.}}:Energy-Efficient Federated Learning for Edge Real-Time Vision via Joint Data, Computation, and Communication Design}

\IEEEtitleabstractindextext{

\begin{abstract}
Emerging real-time computer vision (CV) applications on wireless edge devices demand energy-efficient and privacy-preserving learning. Federated learning (FL) enables on-device training without raw data sharing, yet remains challenging in resource-constrained environments due to energy-intensive computation and communication, as well as limited and non-i.i.d. local data. We propose FedDPQ, an ultra energy-efficient FL framework  for real-time CV over unreliable wireless networks. FedDPQ integrates diffusion-based data augmentation, model pruning, communication quantization, and transmission power control to enhance training efficiency. It expands local datasets using synthetic data, reduces computation through pruning, compresses updates via quantization, and mitigates transmission outages with adaptive power control. We further derive a closed-form energy–convergence model capturing the coupled impact of these components, and develop a Bayesian optimization(BO)-based algorithm to jointly tune data augmentation strategy, pruning ratio, quantization level, and power control. To the best of our knowledge, this is the first work to jointly optimize FL performance from the perspectives of data, computation, and communication under unreliable wireless conditions. Experiments on representative CV tasks show that FedDPQ achieves superior convergence speed and energy efficiency.
\end{abstract}

}
\maketitle
\IEEEdisplaynontitleabstractindextext
\IEEEpeerreviewmaketitle

\section{Introduction}

Performing real-time computer vision (CV) tasks on resource-constrained edge devices is increasingly critical for emerging applications such as autonomous driving, robotics, and virtual/augmented reality (VR/AR) \cite{9953092}. These applications require timely, accurate, and energy-efficient training of machine learning (ML) models using visual data collected at the wireless network edge to maintain continuous adaptation and high-quality service delivery. Traditional centralized ML approaches are impractical in these contexts due to significant communication overhead and privacy vulnerabilities associated with transmitting large amounts of raw visual data. Federated learning (FL) \cite{mcmahan2017communication} addresses these issues by enabling decentralized collaborative model training, reducing data transmission overhead, and preserving data privacy by keeping sensitive data locally on edge devices. However, deploying FL for real-time CV tasks faces substantial challenges, especially regarding energy efficiency.

A fundamental challenge in FL deployment is the considerable energy consumption associated with local model training and the subsequent transmission of model updates. Device energy consumption in FL can be analyzed from two perspectives: energy usage per training round and cumulative energy over multiple rounds. Specifically, per-round energy includes both computation and communication components. Computational energy consumption can be reduced by employing model pruning \cite{9762360} techniques, which eliminate redundant computations based on the widely observed phenomenon that many deep ML models are inherently over-parameterized. Communication energy consumption, though inherently reduced in FL due to transmitting model parameters rather than raw data, remains substantial because of the large model sizes used in contemporary ML tasks. Techniques like quantization \cite{10368103} and sparsification \cite{10026255} effectively reduce communication overhead by exploiting the intrinsic redundancy in parameter updates. Nevertheless, reducing per-round energy via pruning, quantization, or sparsification can inadvertently increase total energy consumption. This counterintuitive outcome arises because these techniques may adversely affect the convergence speed and accuracy of FL models, potentially requiring more training rounds to reach a desired level of performance. Thus, although these methods effectively reduce per-round energy, the increased number of training rounds can negate these benefits, sometimes even resulting in higher overall energy consumption.

From the perspective of global convergence, another critical factor influencing FL performance is the quantity and distribution of available training data. In real-time CV scenarios, training samples are not only limited in quantity due to their dependence on device-captured real-time observations, but also exhibit severe non-i.i.d. characteristics across devices. Generative artificial intelligence (GAI) methods \cite{10599123}, such as diffusion models \cite{wang2025generative}, generative adversarial networks (GAN) \cite{10.1145/3422622}, variational auto-encoders (VAE) \cite{10.1145/3663364}, and so on, offer promising solutions by generating additional synthetic training data based on learned data distributions, thereby improving sample diversity and accelerating convergence. Recent studies \cite{10333794,10463181,10454003,10811919} have begun integrating GAI into FL contexts for data augmentation purposes. However, GAI introduces additional computational overhead due to the energy required for data generation, potentially offsetting some of the convergence speed benefits.

Considering the intricate interactions among computational compression, communication compression, and data augmentation strategies, holistic optimization is necessary. Each of these techniques impacts both per-round energy consumption and the total number of training rounds required for convergence. Moreover, the intermittent connectivity common in wireless edge environments further complicates this optimization, necessitating sophisticated approaches to balance energy efficiency, model accuracy, and operational robustness.

To address the above mentioned challenges, we propose a new ultra energy-efficient FL framework, referred to as FedDPQ, which jointly integrates diffusion model-based data augmentation, model pruning, quantization, and transmission power control to minimize overall energy consumption in real-time CV applications. Specifically, we employ diffusion model-based data augmentation to both expand the amount of training samples and alleviate the non-i.i.d. characteristics of local datasets, thereby improving data availability. Subsequently, we incorporate model pruning to reduce computational overhead and gradient quantization to minimize communication overhead, while transmission power control mitigates transmission outages. Recognizing the inherent coupling among these three techniques and their collective impact on both per-round energy consumption and global convergence performance, we analytically derive their individual and combined influences on FL convergence under unreliable wireless network conditions. Based on this theoretical analysis, we design a low-complexity algorithm to jointly optimize the synthetic data generation strategies, model pruning rates, quantization bit widths, and transmission power. This joint optimization aims to achieve optimal long-term energy efficiency across edge devices.

Our main contributions are summarized as follows:
\begin{itemize}
\item We propose FedDPQ, an ultra energy-efficient FL framework for real-time CV tasks. To the best of our knowledge, this is the first work that jointly optimizes FL performance from the perspectives of data, computation, and communication under unreliable communication links.
\item We derive a closed-form analytical expression capturing how data augmentation, model pruning, and quantization jointly influence FL's overall energy efficiency in unreliable wireless network environments.
\item Guided by the analytical results, we design a low-complexity algorithm based on Bayesian optimization (BO), which systematically determines optimal synthetic data strategies, pruning rates, quantization levels, and transmission power configurations to minimize energy consumption. Experimental results show that FedDPQ enhances energy efficiency and accelerates convergence compared to baseline configurations.
\end{itemize}


The rest of the paper is structured as follows. Section II surveys related literature. Section III presents the system architecture of the proposed FedDPQ scheme. In Section IV, we analyze the convergence behavior and formulate the associated optimization problem. Section V elaborates on the algorithmic design and parameter coordination strategy. Section VI reports the experimental findings, and Section VII summarizes the key contributions and outlines directions for future research.

\section{Related Literature}
In real-world deployments, FL operates as a complex iterative process involving alternating local computation and global communication. Its energy consumption is significantly affected by wireless link conditions, compression strategies, and the design of learning mechanisms. Early studies have mainly focused on communication efficiency. For example, \cite{9488839} proposes an adaptive sparsification framework with convergence guarantees, where sparsity levels are dynamically adjusted based on the computing and communication capacities of edge devices to balance energy consumption. \cite{10168747} develops a quantized over-the-air computation framework that improves spectrum efficiency and reduces energy consumption by selecting participants with favorable channel conditions and applying gradient quantization compatible with quadrature amplitude modulation. Similarly, \cite{9916128} formulates a joint optimization problem to minimize energy consumption through adaptive quantization and wireless resource allocation under performance and latency constraints. More recent efforts begin to consider computation compression jointly with communication efficiency. The work in \cite{10258354} introduces a resource-aware optimization framework that jointly adjusts model pruning levels, CPU frequencies, uplink transmission power, and bandwidth allocation. An edge device selection strategy is further employed to improve system-level energy efficiency under heterogeneous conditions. Despite these advances, FL performance in practice remains constrained by limited data availability and non-i.i.d. data across devices, which impair model convergence and increase energy costs. To mitigate this, generative data augmentation has been explored. \cite{10333794} proposes a VAE-based FL framework for trajectory data synthesis that maintains the structural properties of the original data while preserving privacy. \cite{10463181} addresses data insufficiency in vertical FL by using a distributed GAN to generate pseudo-overlapping features across participants. However, GAN-based approaches often face difficulties in generating high-fidelity samples and typically overlook the computational burden introduced by data augmentation. To address this, the works presented in \cite{10454003,10811919} leverage diffusion models to enable resource-aware data synthesis, improving convergence and reducing local energy consumption in heterogeneous environments.

As discussed above, data quantity and quality, model compression, and communication compression all have intricate and interdependent effects on the convergence behavior and energy consumption of FL. Analyzing their coupling relationship from a unified perspective and jointly optimizing them within a single framework holds great potential for further improving energy efficiency. However, to date, no existing work has systematically addressed this joint optimization problem. Therefore, this paper aims to fill this gap.

\section{System Model}
In the FedDPQ system, as shown in Fig. \ref{ArchitectureJournalVersion}, there is a base station (BS) paired with an edge server, and a set of $U$ devices equipped with cameras, denoted as $\mathcal{U}=\{1,2,\dots,u,\dots, U\}$. The original data samples on device $u$ are represented by $\mathcal{D}_u^{\text{loc}} =\{(\boldsymbol{x}_{u, d}, \boldsymbol{y}_{u,d})\}^{D^{\text{loc}}_{u}}_{d=1}$, where $|\mathcal{D}^{\text{loc}}_u| = D^{\text{loc}}_u $ represents the total number of local training samples observed by device $u$, and $\boldsymbol{x}_{u, d}$ and $\boldsymbol{y}_{u,d}$ are the samples and their corresponding labels, respectively. To effectively leverage the distributed data on the devices for training a robust data-driven CV model with low communication overhead and high privacy, these devices participate in a distributed FL process coordinated by the BS. To fill the missing portion of local data and thereby improve the convergence rate of FL, FedDPQ employs a pre-trained generative model to synthesize new data for each device according to a data augmentation strategy. Considering the energy limitations of devices, FedDPQ also introduces model pruning and quantization techniques to reduce the computational overhead on devices and decrease communication costs, respectively. Moreover, FedDPQ takes into account the impact of transmission outages during communication. 

\begin{figure}[t]
  \centering
  \includegraphics[width=8.8cm]{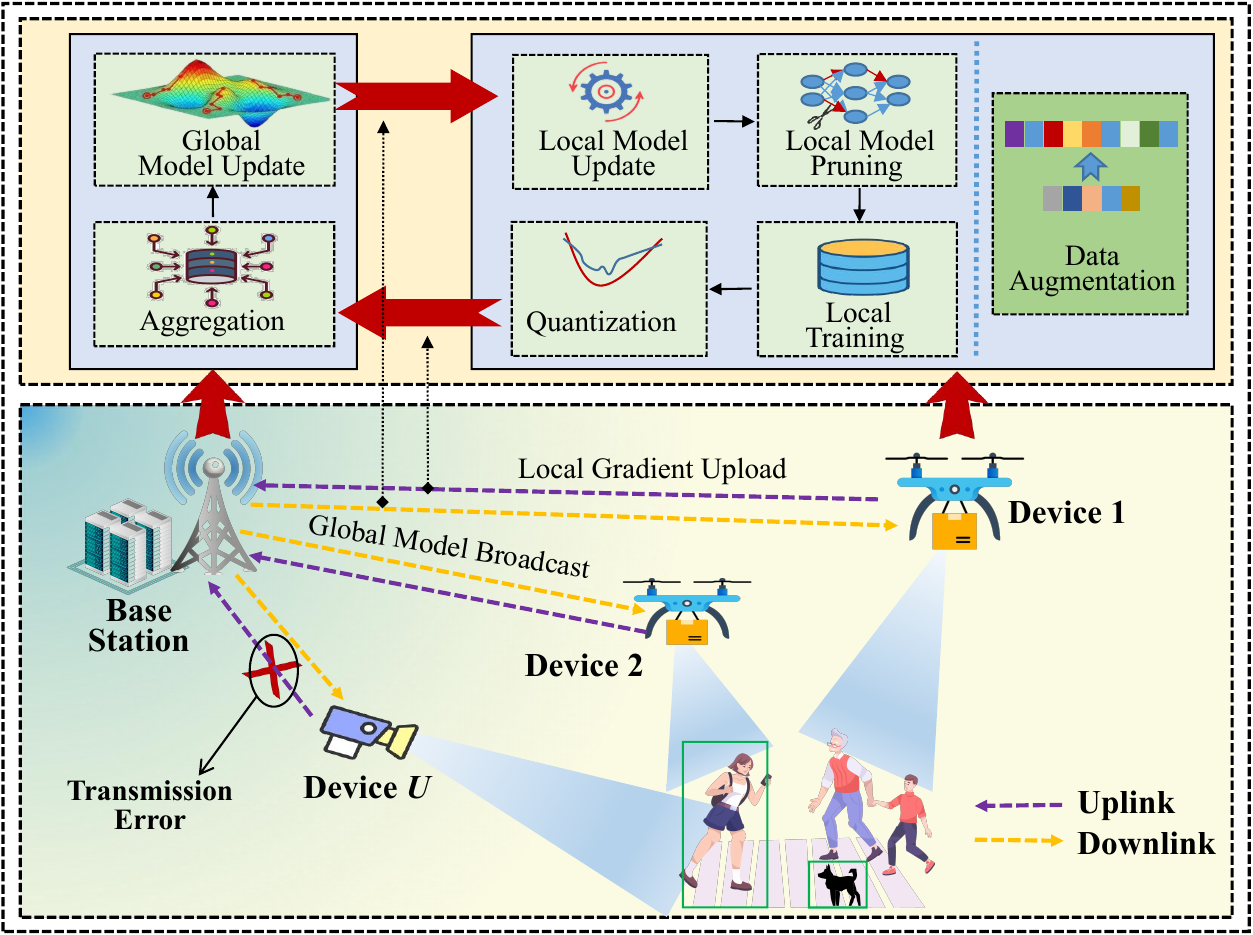}
  \caption{The architecture of FedDPQ.}\label{ArchitectureJournalVersion}
\end{figure}
\subsection{Data Augmentation}
The original dataset owned by the device $u$ can be reorganized by class as \( \mathcal{D}_u^{\text{loc}} = \{ \mathcal{D}^{\text{loc}}_{u,1}, \mathcal{D}^{\text{loc}}_{u,2}, \dots, \mathcal{D}^{\text{loc}}_{u,C} \} \), where \( \mathcal{D}^{\text{loc}}_{u,c} \) represents samples belonging to class \( c \). The pre-trained generative model deployed locally, based on the original data distribution of device $u$ and following a specific strategy, generates a synthetic dataset \( \mathcal{D}_u^{\text{gen}} = \{ \mathcal{D}^{\text{gen}}_{u,1}, \mathcal{D}^{\text{gen}}_{u,2}, \dots, \mathcal{D}^{\text{gen}}_{u,C} \} \), where \( |\mathcal{D}_{u,c}^{\text{gen}}| = D_{u,c}^{\text{gen}}\). Define a data augmentation factor \( \Delta_u \) for each device $u$, and the intermediate generation target \( D^\text{{gen}}_{u,c} \) is computed as
\begin{equation}
 D^{\text{gen}}_{u,c} = \max\{\Delta_u  D'_{u}-D^{\text{loc}}_{u,c},\ \  0\},
\end{equation}
where $D'_{u} = \max_c D^{\text{loc}}_{u,c}$. We denote the mixed dataset obtained by device $u$ after data generation as $\mathcal{D}_u^{\text{mix}} = \left\{ \mathcal{D}^{\text{mix}}_{u,1}, \mathcal{D}^{\text{mix}}_{u,2}, \dots, \mathcal{D}^{\text{mix}}_{u,C} \right\}$, where $\mathcal{D}^{\text{mix}}_{u,c}$ is represented as
\begin{equation}
\mathcal{D}^{\text{mix}}_{u,c} =
\begin{cases}
\mathcal{D}^{\text{loc}}_{u,c}, & \text{if } D^{\text{loc}}_{u,c} \ge \Delta_u \cdot D'_{u} \\[6pt]
\mathcal{D}^{\text{gen}}_{u,c}\cup \mathcal{D}^{\text{loc}}_{u,c}, & \text{if } D^{\text{loc}}_{u,c} < \Delta_u \cdot D'_{u}.
\end{cases}
\end{equation}
\noindent

Accordingly, the total number of generated data samples on device \( u \), denoted as \( D_u^{\text{gen}} \), can be computed as
\begin{equation}
D_u^{\text{gen}} = \sum_{c=1}^{C} D^{\text{gen}}_{u,c}.
\end{equation}


\subsection{Basic FL Framework}
The primary aim of  FL is to minimize the global loss function $F(\boldsymbol{w})$, which is associated with the globally shared learning model. Mathematically, the FL's objective is formulated as
\begin{equation}
\label{MinTheGlobalLossFunction}
\boldsymbol{w}^{*} = \arg \min F\left(\boldsymbol{w}\right),
\end{equation}
where  $F(\boldsymbol{w})=\sum_{u=1}^{U} \tau_uF_u\left(\boldsymbol{w}_u\right)$, and $\tau_u = \frac{D^{\text{loc}}_u+D^{\text{gen}}_u}{\sum_{u=1}^{U}\left(D^{\text{loc}}_u+D^{\text{gen}}_u\right)}$ denotes the normalized data proportion of device $u$. Let $\boldsymbol{\xi}_u$ represent the mini-batch samples of size $b$. We define $F_u\left( \boldsymbol{w}, \boldsymbol{\xi}_u \right) = \frac{1}{b}\sum_{m = 1}^b f(\boldsymbol{w}, \xi_{u,m})$, where $\xi_{u,m}$ denotes the $m$-th randomly chosen sample from device $u$’s dataset, and $f(\boldsymbol{w}, \xi_{u,m})$ is the sample-level loss function. 

The classical FL framework usually consists of the following steps.

\textbf{Step 1:} In the $t$-th communication round, the server samples $S$ devices, denoted as $\mathcal{S}^t$, and broadcasts the global model $\boldsymbol{w}^{t-1}$ in the last communication round to these selected devices.

\textbf{Step 2:} Each device $u \in\mathcal{S}^t$ updates its local model by 
\begin{equation} \label{LocalGradientUpdating} \boldsymbol{w}_u^t = \boldsymbol{w}^{t-1} -\eta\nabla F_u\left( \boldsymbol{w}^{t-1}, \boldsymbol{\xi}_u^t \right), 
\end{equation} where $\eta$ denotes the learning rate. 

\textbf{Step 3:} The devices in $\mathcal{S}^t$ upload their local models $\boldsymbol{w}_u^t$ \footnote{The transmitted information can also be the computed gradients, or transformed variants thereof, depending on the specific FL algorithm design.}  to the BS, where the server updates the global model according to specific aggregation strategy. 

It is important to note that FL typically adopts one of two aggregation strategies.

\begin{itemize}
    \item \textbf{(i) Full participation:} All devices participate in the aggregation process, and the global model is updated by
    \begin{equation}
    \label{FullP}
        \bar{\boldsymbol{w}}^t = \sum_{u=1}^{U} \tau_u \boldsymbol{w}_u^t,
    \end{equation}
    However, since the large number of devices in FL, this strategy is impractical under real-world communication constraints.
    
    \item \textbf{(ii) Partial participation:} The server samples $S$ devices $\left(S \ll U\right)$ in $\mathcal{S}^t$ with replacement according to the probability distribution $\{\tau_1,\dots, \tau_U\}$, the global model is updated by
    \begin{equation}
    \label{partialp}
        \boldsymbol{w}^t = \frac{1}{S}\sum_{u \in \mathcal{S}^t} \boldsymbol{w}_u^t.
    \end{equation}
\end{itemize}

It is worth noting that the average scheme in Eq. (\ref{partialp}) yields an unbiased estimate of $\bar{\boldsymbol{w}}^t$ in Eq. (\ref{FullP}), meaning $\mathbb{E}\left[  \boldsymbol{w}^t \right] = \bar{\boldsymbol{w}}^t$ \cite{Li2020On}.

\subsection{FL with Model Pruning, Gradient Quantization with Transmission Outage}
\subsubsection{Model Pruning}
To reduce the computational energy consumption at local training, model pruning is employed. Model pruning involves the removal of insignificant connections or filters to compress the model, thereby reducing computations with minimal loss in accuracy. According to \cite{9598845},  we can evaluate the importance of the $v$-th parameter by
\begin{equation}
\label{ImportanceEvaluation}
\bar{I}^t_{u,v}=\left(F_{u}\left(\boldsymbol{w}^t_{u}, \boldsymbol{\xi}_u^t\right)-F_{u}\left(\left.\boldsymbol{w}^t_{u},\boldsymbol{\xi}_u^t\right|_{w^t_{u,v}=0}\right)\right)^2,
\end{equation}
which is essentially the mean squared error with and without $w^t_{u,v}$. The smaller the value of $\bar{I}^t_{u,v}$ is, the less important the  $w^t_{u,v}$. Nevertheless, the calculation for obtaining the mean squared error $\bar{I}^t_{u,v}$ is computation-intensive, especially when the model is large. Therefore,  to alleviate the computational complexity, the approximate to Eq. (\ref{ImportanceEvaluation}) is introduced, which is denoted by
\begin{equation}
\bar{I}^t_{u,v}=\left\| w^t_{u,v}\right\|.
\end{equation}

After the importance evaluation,  model pruning is performed by zeroing out the parameters having relatively low importance values. Define the pruned model of device $u$ at the $t$-th round as $\widetilde{\boldsymbol{w}}_{u}^{t}$, and the pruning ratio $\rho_u$ of  device $u$ as
\begin{equation}
\rho_u=\frac{{V}_{u}}{V},
\end{equation}
where ${V}_{u}$ represents the number of the pruned parameters of the deployed model on device $u$.

\subsubsection{Gradient Quantization}

In practice, we transmit the local gradients $\boldsymbol{g}_{u}^{t} = \nabla F_u\left( \widetilde{\boldsymbol{w}}^{t}_{u}, \boldsymbol{\xi}_u^t \right)$ instead of the model $\boldsymbol{w}_u^t$, since these gradients fully characterize the updates to the model. Then, we apply stochastic quantization to further compress the updated local gradients before transmitting them to the BS. Let $g_{u,v}^t$ denote the $v$-th element of the local gradient, where $v \in \{1, 2, \ldots, V\}$ and $g_{u,v}^t \in \left[ \underline{g}_{u,v}^t, \overline{g}_{u,v}^t \right]$. Let $\mathcal{Q}$ denote the quantization function, and let $\delta_u$ represent the number of quantization bits allocated to device $u$ during the $t$-th round. The range $\left[ \underline{g}_{u,v}^t, \overline{g}_{u,v}^t \right]$ is uniformly divided into $2^{\delta_u}$ equal-length intervals, with boundary points denoted by $\{b_0, b_1, \dots, b_{2^{\delta_u}}\}$, where $b_j$ is given by
\begin{equation}
b_j = \underline{g}_{u,v}^t + j \cdot \frac{ \overline{g}_{u,v}^t - \underline{g}_{u,v}^t }{2^{\delta_u} - 1}, \quad j = 0, \dots, 2^{\delta_u} - 1.
\label{eq:bt}
\end{equation}
If $g_{u,v}^t \in [b_j, b_{j+1})$, it can be quantized as
\begin{equation}
\mathcal{Q}(g_{u,v}^t) =
\begin{cases}
\mathrm{sign}(g_{u,v}^t) \cdot b_j, & \text{with probability } \frac{b_{j+1} - |g_{u,v}^t|}{b_{j+1} - b_j}, \\
\mathrm{sign}(g_{u,v}^t) \cdot b_{j+1}, & \text{with probability } \frac{|g_{u,v}^t| - b_j}{b_{j+1} - b_j},
\end{cases}
\label{eq:quant}
\end{equation}
where $\mathrm{sign}(\cdot)$ returns the sign of the gradient, while the quantized value is chosen stochastically based on its proximity to neighboring quantization levels. Let $o$ denote the additional number of bits needed to encode the sign and the endpoints $\underline{g}_{u,v}^t$, $\overline{g}_{u,v}^t$. The total number of bits required to represent the quantized gradient vector becomes
\begin{equation}
\tilde{\delta}_u = V \delta_u + o.
\label{eq:totalbits}
\end{equation}


\subsubsection{Transmission Error}
Owing to unreliable communication links, some of the selected devices may fail to participate in the current training round. Consider that each device uploads its local gradients using the widely adopted orthogonal frequency division multiplexing (OFDM) \footnote{\textcolor{black}{Although this work assumes OFDM for uplink, the proposed scheme can be extended to other multiple access protocols such as time division multiple access (TDMA) and  non-orthogonal multiple access (NOMA) with minor adaptations.}}. The uplink data rate between device $u$ and the BS is given by
\begin{equation}
R_u\left(p_u\right) = B_u^{\mathrm{UL}} \mathbb{E}_{h_u} \left( \log_2 \left( 1 + \frac{p_u h_u}{I_u + B_u^{\mathrm{UL}} N_0} \right) \right),
\end{equation}
where $B_u^{\mathrm{UL}}$ denotes the uplink bandwidth allocated to device $u$, $p_u$ represents the transmission power of device $u$, $I_u$ represents the interference, and $N_0$ denotes the power spectral density of noise. In addition, the channel gain $h_u$ is calculated as
\begin{equation}
h_u = \frac{\zeta_u}{\left(d_u\right)^2},
\end{equation}
where $\zeta_u$ is the Rayleigh fading coefficient, while $d_u$ is the distance between device $u$ and the BS.

In practice, transmission errors are inevitable. Assuming that the local gradients from device $u$ are uploaded as data packets, the probability of a transmission error for these packets is represented as \cite{9210812}
\begin{equation}
q_u\left(p_u\right) = \mathbb{E}_{h_u} \left( 1 - \exp \left( -\frac{\Upsilon \left( I_u + B_u^{\mathrm{UL}} N_0 \right)}{p_u h_u} \right) \right),
\end{equation}
where $\Upsilon$ is the waterfall threshold. We use a binary variable $\alpha_u^t$ to indicate whether the local gradient of device $u$ is successfully received by the BS at the $t$-th iteration. The value of $\alpha_u^t$ is represented as
\begin{equation}
\alpha_u^t = \begin{cases}
1, & \text{if } 1 - q_u\left(p_u\right), \\
0, & \text{if } q_u\left(p_u\right).
\end{cases}
\end{equation}

Taking into account partial participation, model pruning, stochastic quantization, and transmission errors, the global model at the server is updated as
\begin{equation}
\label{globalgbar}
\boldsymbol{w}^t = \boldsymbol{w}^{t-1} -\eta \frac{\sum_{u \in \mathcal{S}^t}\alpha_{u}^{t}\mathcal{Q}(\boldsymbol{g}_{u}^{t})}{\sum_{u \in \mathcal{S}^t} \alpha_{u}^{t}}.
\end{equation}

\section{Convergence Analysis and  Problem Formulation}
In this section, we conduct a quantitative analysis of the impact of partial participation, data augmentation, model pruning, stochastic quantization, and transmission errors on the convergence of FL, and derive the minimum number of communication rounds required to achieve a given model convergence accuracy. Additionally, we perform an energy consumption analysis of the system and formulate an optimization problem.

\subsection{Basic Assumptions}
Before performing the convergence analysis, we introduce some widely adopted assumptions as follows:
\begin{itemize}
  \item \begin{assumption}\label{CompleteAssumption1}
        The local function $ F_u(\boldsymbol{w})$ is uniformly $L$-Lipschitz continuous with respect to any $\boldsymbol{w}$ and $\boldsymbol{w}^{\prime}$, which is represented by
        \begin{equation}
        \left\|\nabla F_u\left(\boldsymbol{w}\right) - \nabla F_u\left(\boldsymbol{w}'\right)\right\| \leqslant L\left\|\boldsymbol{w} - \boldsymbol{w}'\right\|,\ \ \forall u,
        \end{equation}
        where $L$ is the Lipschitz constant, depending on $F(\cdot)$.
    \end{assumption}

  \item \begin{assumption}\label{CompleteAssumption2}
        Unbiasedness and bounded variance of stochastic gradient
        \begin{equation}
        \label{Assumption21}
        \mathbb{E}\left[ \boldsymbol{g}_u\left(\boldsymbol{w}\right)\right] = \nabla F_u\left(\boldsymbol{w}\right),\forall u,\forall \boldsymbol{w},
        \end{equation}
        and
        \begin{equation}
        \label{Assumption22}
        \mathbb{E}\left[ \left\|\boldsymbol{g}_u\left(\boldsymbol{w}\right) - \nabla F_u\left(\boldsymbol{w}\right)\right\|^2\right] \leq \sigma^2, \ \ \forall u,\forall \boldsymbol{w}.
        \end{equation}
    \end{assumption}

  \item \begin{assumption}\label{CompleteAssumption3}
        Data heterogeneity with bounded gradient variance
        \begin{equation}
        \mathbb{E}\left[\left\|\nabla F_u\left(\boldsymbol{w}\right)-\nabla F\left(\boldsymbol{w}\right)\right\|\right]\leq Z_u^2, \ \ \forall u,\forall \boldsymbol{w}.
        \end{equation}
    \end{assumption}

  \item \begin{assumption}\label{CompleteAssumption4}
        The second moments of model parameters are bounded, given by
        \begin{equation}
        \mathbb{E}\left[\left\|\boldsymbol{w}\right\|^2\right] \leq \Gamma^2, \ \ \forall \boldsymbol{w}.
        \end{equation}
    \end{assumption}
\end{itemize}

\subsection{Convergence Analysis}
To facilitate the convergence analysis, we first introduce several essential lemmas that lay the groundwork for the subsequent theoretical development.

\begin{lemma}\label{lemma1}
The model pruning error with pruning ratio $\rho_u$ can be represented as
\begin{equation}
\begin{aligned}
\mathbb{E}\left\{\left\|\boldsymbol{w}_u^t-\tilde{\boldsymbol{w}}_u^t\right\|^2\right\} \leqslant \rho_u \Gamma^2.
\end{aligned}
\end{equation}
\end{lemma}
\begin{proof}
See the detailed proof  in  \cite{stich2018sparsified}.
\end{proof}

\begin{lemma} \label{lemma2}
By the stochastic quantization, each local gradient is unbiasedly estimated as
\begin{equation}
\label{quantize no bias}
\begin{aligned}
& \mathbb{E}\left[\mathcal{Q}(\boldsymbol{g}_u^t)  \right] = \boldsymbol{g}_u^t,
\end{aligned}
\end{equation}
and the associated quantization error is bounded by
\begin{equation}
\mathbb{E} \left[ \left\| \mathcal{Q}\left(\boldsymbol{g}_u^t\right) - \boldsymbol{g}_u^t \right\|^2 \right]
\leq
\frac{
\displaystyle \sum_{v=1}^{V} \left( \bar{\boldsymbol{g}}_{u,v}^t - \underline{\boldsymbol{g}}_{u,v}^t \right)^2
}{
4\left(2^{\delta_u} - 1\right)^2
}.
\end{equation}

\end{lemma}
\begin{proof}
    See the detail proof in \cite{zheng2020design}.
\end{proof}

\begin{lemma} \label{lemma3}
Under Assumptions and Lemma \ref{lemma1}, we can obtain
\begin{equation}
\begin{aligned}
&\mathbb{E}\left[ \langle \nabla F(\boldsymbol{w}^{t-1}), \boldsymbol{w}^{t} - \boldsymbol{w}^{t-1} \rangle \right] \leq -\frac{\eta}{2} \cdot \mathbb{E} \left[ \left\| \nabla F(\boldsymbol{w}^{t-1}) \right\|^2 \right]\\&+\eta\cdot\chi^2_{\boldsymbol{\beta}\| \boldsymbol{\tau}} \cdot \sum_{u=1}^{U} \tau_u Z_u^2+ \eta\sum_{u=1}^{U} \bar{\beta}_u^2 \cdot L^2 \Gamma^2 \sum_{u=1}^{U} \rho_u,
\end{aligned}
\end{equation}
where $\chi^2_{\boldsymbol{\beta}\| \boldsymbol{\tau}}=\sum_{u=1}^U\left(\bar{\beta}_u-\tau_u\right)^2$ is the chi-square divergence between $\boldsymbol{\beta} = \left[\bar{\beta}_1,\ldots,\bar{\beta}_U \right]$ and $\boldsymbol{\tau} = \left[\tau_1,\ldots,\tau_U \right]$\cite{wang2020tackling}, and $\bar{\beta}_u$ is a function of $\{q_u, \forall u \in\mathcal{U} \}$, as shown in Eq. \eqref{eqA1} in Appendix A. 
\end{lemma}
\begin{proof}
See the detailed proof in Appendix A.
\end{proof}
\begin{lemma} \label{lemma4}
According to Assumptions and Lemma \ref{lemma2}, it holds that
\begin{equation}
\begin{aligned}
&\mathbb{E}\left[\|\boldsymbol{w}^{t} - \boldsymbol{w}^{t-1}\|^2\right]\leq\; 2\eta^2\sum_{u=1}^{U} \bar{\alpha}_u\frac{
\displaystyle \sum_{v=1}^{V} \left( \bar{g}_{u,v}^t - \underline{g}_{u,v}^t \right)^2
}{
4\left(2^{\delta_u} - 1\right)^2
}\\&+8\eta^2L^2 \Gamma^2 \sum_{u=1}^{U} \bar{\beta}_u \rho_u+16\eta^2\, \mathbb{E} \left[
\left\| 
  \nabla F(\boldsymbol{w}^{t-1}) 
\right\|^2
\right]\\&+4\eta^2\sum_{u=1}^{U}\bar{\alpha}_u\sigma^2 + 16\eta^2\sum_{u=1}^{U} \bar{\alpha}_u Z_u^2\\&+16\eta^2\sum_{k=2}^{S}
\frac{
(q_{\max})^{S - k} \, \mathbb{C}_S^k
}{
1 - (q_{\max})^S
}
\sum_{u=1}^{U} \tau_u \left\| q_u - \bar{q} \right\|^2 Z_u^2,
\end{aligned}
\end{equation}
where $ \mathbb{C}_S^k = \frac{S!}{k!(S-k)!}$, $q_{\max} = \max\{q_1,\dots,q_U\}$ and $\bar{q}=\sum\limits_{u=1}^U\tau_uq_u$ represent maximum and average transmission error probability, respectively. $\bar{\alpha}_u$ is a function of $\{q_u, \forall u \in\mathcal{U}\}$, as shown in Eq. \eqref{eqB.1} in Appendix B.
\end{lemma}
\begin{proof}
See the detailed proof in Appendix B.
\end{proof}
Based on Lemmas \ref{lemma3} and \ref{lemma4}, the convergence rate of FL can be derived by  Theorem \ref{Theorem1}.
\begin{theorem}
\label{Theorem1}
\textcolor{black} {When the data generation strategy $\boldsymbol{\Delta} = \left[\Delta_1,\ldots,\Delta_U \right]$, model pruning strategy $\boldsymbol{\rho} = \left[\rho_1,\ldots,\rho_U \right]$, quantization level strategy $\boldsymbol{\delta} = \left[\delta_1, \ldots,\delta_u\right]$ and power configuration strategy $\boldsymbol{p} = \left[p_1,\ldots,p_U \right]$ are given, the upper bound on the average $l_2$-norm of the gradients after $\Omega$ rounds with the learning rate satisfying $0<\eta<\frac{1}{16L}$ is represented as}
\begin{equation}
\begin{aligned}
&\frac{1}{\Omega} \sum_{t = 1}^\Omega\mathbb{E} \left[ \left\| \nabla F(\boldsymbol{w}^{t-1}) \right\|^2 \right]
\leq
\frac{ \mathbb{E}[F(\boldsymbol{w}^{0})]-\mathbb{E}[F(\boldsymbol{w}^{*})] }{ \left(\frac{\eta}{2}-8L\eta^2\right) \Omega}\\
&\quad+ \frac{ \eta \cdot \chi^2_{\boldsymbol{\beta}\| \boldsymbol{\tau}} }{ \frac{\eta}{2}-8L\eta^2 } \cdot \sum_{u=1}^{U} \tau_u Z_u^2+ \frac{ 8L\eta^2 }{ \left(\frac{\eta}{2}-8L\eta^2\right) } \sum_{u=1}^{U} \bar{\alpha}_uZ_u^2\\
&\quad+ \frac{ \eta L^2 \Gamma^2 }{ \left(\frac{\eta}{2}-8L\eta^2\right) } \left(\sum_{u=1}^{U} \bar{\beta}_u^2 \cdot \sum_{u=1}^{U} \rho_u+ 4\eta L\sum_{u=1}^{U} \bar{\beta}_u \rho_u\right)\\
&\quad+ \frac{L\eta^2}{\left(\frac{\eta}{2}-8L\eta^2\right)\Omega} \sum_{t = 1}^\Omega\sum_{u=1}^{U} \bar{\alpha}_u\frac{
\displaystyle \sum_{v=1}^{V} \left( \bar{g}_{u,v}^t - \underline{g}_{u,v}^t \right)^2
}{
4\left(2^{\delta_u} - 1\right)^2
} \\
&\quad+ \frac{8L\eta^2}{ \frac{\eta}{2}-8L\eta^2 } 
\sum_{k=2}^{S}
\frac{
(q_{\max})^{S - k} \, \mathbb{C}_S^k
}{
1 - (q_{\max})^S
}
\sum_{u=1}^{U} \tau_u \left\| q_u - \bar{q} \right\|^2 Z_u^2\\&\quad + \frac{2 L\eta^2}{ \left(\frac{\eta}{2}-8L\eta^2\right) } \sum_{u=1}^{U} \bar{\alpha}_u\sigma^2.
\end{aligned}
\end{equation}
\end{theorem}
\begin{proof}
See the detailed proof in Appendix C.
\end{proof}
From Theorem 1, it can be seen that when the devices have uniform transmission error probabilities, that is, $q_u = q\ \forall u$, the convergence rate will be improved, and the training process will still converge to a proper stationary solution. Under this condition, we have \(\bar{\beta}_u = \tau_u\) and \(\bar{\alpha}_u = \tau_u / \bar{S}\), 
with $\bar{S} = \frac{1 - q^S}{\sum_{k=1}^{S} \frac{1}{k} \mathbb{C}_S^k (1 - q)^{k} q^{S-k}}$\cite{wang2021quantized}.
\begin{corollary}\label{corollary1}
Under the same conditions as Theorem \ref{Theorem1}, if all
devices have uniform transmission error probabilities, we have

\begin{equation}
\begin{aligned}
&\frac{1}{\Omega} \sum_{t = 1}^\Omega\mathbb{E} \left[ \left\| \nabla F(\boldsymbol{w}^{t-1}) \right\|^2 \right]
\leq
\frac{ \mathbb{E}[F(\boldsymbol{w}^{0})]-\mathbb{E}[F(\boldsymbol{w}^{*})] }{ \left(\frac{\eta}{2}-8L\eta^2\right) \Omega}\\
&\quad+ \frac{ 2L\eta^2 }{ \left(\frac{\eta}{2}-8L\eta^2\right) } \left(4\sum_{u=1}^{U} \frac{\tau_u}{\bar{S}}Z_u^2+  \frac{\sigma^2}{\bar{S}}\right)\\
&\quad+ \frac{ \eta L^2 \Gamma^2 }{ \left(\frac{\eta}{2}-8L\eta^2\right) } \left(\sum_{u=1}^{U} \tau_u^2 \cdot \sum_{u=1}^{U} \rho_u+ 4\eta L\sum_{u=1}^{U} \tau_u \rho_u\right)\\
&\quad+ \frac{L\eta^2}{\left(\frac{\eta}{2}-8L\eta^2\right)\Omega} \sum_{t = 1}^\Omega\sum_{u=1}^{U} \frac{\tau_u}{\bar{S}}\frac{
\displaystyle \sum_{v=1}^{V} \left( \bar{g}_{u,v}^t - \underline{g}_{u,v}^t \right)^2
}{
4\left(2^{\delta_u} - 1\right)^2
},
\end{aligned}
\end{equation}
\end{corollary}

\begin{corollary}\label{corollary2}
\textcolor{black} {The number of training rounds $\Omega$ required for achieving the convergence target that \( \frac{1}{\Omega} \sum_{t = 1}^\Omega\mathbb{E} \left[ \left\| \nabla F(\boldsymbol{w}^{t-1}) \right\|^2 \right] \leq \boldsymbol{\varepsilon} \) is represented as \cite{9916128}}
\begin{equation}
\Omega 
\geq \frac{ \mathbb{E}[F(\boldsymbol{w}^{0})] - \mathbb{E}[F(\boldsymbol{w}^{*})] }
{ \left( \frac{\eta}{2} - 8 L \eta^2 \right) 
\boldsymbol{\varepsilon} - \Psi  },
\end{equation}
with
\begin{equation}
\begin{aligned}
\Psi =\;&
\eta L^2 \Gamma^2
\left( 
\sum_{u=1}^{U} \tau_u^2 \cdot \sum_{u=1}^{U} \rho_u 
+ 4\eta L \sum_{u=1}^{U} \tau_u \rho_u 
\right) \\[8pt]
&+ L\eta^2
\sum_{u=1}^{U} \frac{ \tau_u }{ \bar{S} } 
\frac{ \displaystyle \sum_{v=1}^V \left( \bar{g}^{\prime}_{u,v} - \underline{g}^{\prime}_{u,v} \right)^2 }
{  4\left( 2^{\delta_u} - 1 \right)^2 } \\[8pt]
&+ 2L \eta^2
\left( 
\frac{ \sigma^2 }{ \bar{S} } 
+ 4 \sum_{u=1}^{U} \frac{ \tau_u }{ \bar{S} } Z_u^2 
\right),
\end{aligned}
\end{equation}
\end{corollary}
where $\bar{g}^{\prime}_{u,v}=\max\{\bar{g}^{1}_{u,v},\dots,\bar{g}^{\Omega}_{u,v}\}$ and $\underline{g}^{\prime}_{u,v}=\min\{\underline{g}^{1}_{u,v},\dots,\underline{g}^{\Omega}_{u,v}\}$.
\subsection{Energy Consumption Analysis}
In practice, the BS typically has a continuous power supply, so only the energy consumption of each device is considered \footnote{\textcolor{black} {In this paper, we focus on a CPU-based computing scenario; however, this framework can be readily adapted to GPU-based scenarios by modifying Eqs. (\ref{GenComputingEnergyConsumption})-(\ref{TraningComputingTime}) following the methodology presented in \cite{9488839}.}}. The primary energy consumption for devices arises from data augmentation, local training and gradient upload. Firstly, the data augmentation energy consumption can be represented as 
\begin{equation}\label{GenComputingEnergyConsumption}
E_{u, gen} = \varrho \left(f_u\right)^\gamma T_{u, gen},
\end{equation}
where $f_u$ represents the available computational resources of device $u$, while $\varrho$ and $\gamma$ are constant parameters \cite{10454003}. $T_{u, gen}$ is the data generation time of device $u$, which is expressed as
\begin{equation}\label{GenComputingTime}
T_{u, gen} = \frac{D_u^{\text{gen}} c^{\text{gen}}_0 }{f_u},
\end{equation}
where $c^{\text{gen}}_0$ denotes the requisite CPU cycles for generating one data sample. Then, the training energy consumption of device $u$ is represented as
\begin{equation}\label{TraningComputingEnergyConsumption}
E_{u, tr} = \varrho \left(f_u\right)^\gamma T_{u, tr}.
\end{equation}
The local training time $T_{u, tr}^t$ of device $u$ is expressed as \cite{10368103}
\begin{equation}\label{TraningComputingTime}
T_{u, tr} = \frac{b c^{\text{tr}}_0 \left(1 - \rho_u\right)}{f_u},
\end{equation}
where $c^{\text{tr}}_0$ represents the total number of CPU cycles required for training one data sample \cite{9598845}.
Similarly, the energy consumption for gradient upload by device $u$ is
\begin{equation}
E_{u, cu} = p_u T_{u, cu},
\end{equation}
where $T_{u, cu}$ represents the communication time between device $u$ and the BS, calculated as
\begin{equation}
T_{u, cu} = \frac{\tilde{\delta}_u}{R_u\left(p_u\right)}.
\end{equation}

Thus, the overall expected energy consumption can be calculated as
\begin{equation}\label{TotalEnergyConsumption}
\mathcal{H}\left(\boldsymbol{\Delta,}\boldsymbol{\rho},\boldsymbol{\delta},\boldsymbol{p}\right)=\Omega\sum_{u\in \mathcal{U}} \tau_u\left(E_{u, tr} + E_{u, cu}\right) + \sum_{u\in \mathcal{U}} E_{u, gen} .
\end{equation}

\subsection{Problem Formulation}
Given the limited energy budget of edge devices, it is essential to balance training performance and energy consumption. To this end, we formulate a joint optimization problem that minimizes the accumulated average energy consumption while ensuring convergence requirements. The optimization jointly considers the data augmentation strategy $\boldsymbol{\Delta}$, model pruning strategy $\boldsymbol{\rho}$, quantization strategy $\boldsymbol{\delta}$, and power control strategy $\boldsymbol{p}$. The problem is defined as follows
\begin{subequations}
\begin{align}
\mathcal{P}1: \quad & \min_{\boldsymbol{\Delta},\boldsymbol{\delta},\boldsymbol{\rho},\boldsymbol{p}} 
\; \mathcal{H}\left(\boldsymbol{\Delta}, \boldsymbol{\rho}, \boldsymbol{\delta}, \boldsymbol{p}\right) 
\label{Za} \\
\text{s.t.} \quad 
& \quad\quad\Delta^{\min} \leqslant \Delta_u \leqslant \Delta^{\max}, 
&& \forall u,  \label{Zb} \\
& \quad\quad\delta_u \in \mathbb{Z}_+, 
&& \forall u,  \label{Zc} \\
& \quad\quad \delta^{\min} \leqslant \delta_u \leqslant \delta^{\max}, 
&& \forall u,  \label{Zd} \\
& \quad\quad p^{\min} \leqslant p_u \leqslant p^{\max}, 
&& \forall u, \label{Ze} \\
& \quad\quad\rho^{\min} \leqslant \rho_u \leqslant \rho^{\max}, 
&& \forall u,  \label{Zf}\\
& \quad\quad\quad q_u  = q, 
&& \forall u,  \label{Zg}
\end{align}
\end{subequations}
where Eq. (\ref{Zb}) represents the data generation constraints. $\mathbb{Z}_+$ in Eq. (\ref{Zc}) is the positive
integer set. Eqs. (\ref{Zd})-(\ref{Zf}) specify the valid ranges of $\delta_u$, $p_u$ and $\rho_u$, respectively. As suggested by Corollary 1, it is crucial to maintain uniform transmission error probabilities across all the devices, so we enforce the constraints in Eq. (\ref{Zg}).

\section{Algorithm Design}
In problem $\mathcal{P}1$, the objective function $\mathcal{H}(\boldsymbol{\Delta}, \boldsymbol{\rho}, \boldsymbol{\delta}, \boldsymbol{p})$ involves both continuous variables $\boldsymbol{\Delta}, \boldsymbol{\rho}, \boldsymbol{p}$ and discrete variables $\boldsymbol{\delta}$, with the power control variable $\boldsymbol{p}$ being uniquely determined by the common variable $q$ across all devices. This coupling, along with the non-convex and mixed-variable nature of the objective function, significantly increases the complexity of the optimization problem. Thus, to efficiently solve $\mathcal{P}1$, we reformulate it and design a block coordinate descent (BCD) framework, where each variable group is optimized individually using BO while the others are held fixed.

\subsection{Problem Reformulation}
Given the constraint $q_u = q,\ \forall u$, the power control vector $\boldsymbol{p}$ becomes an implicit function of $q$ and can be omitted from the optimization variables. Therefore, $\mathcal{P}1$ can be reformulated as
\begin{equation} 
\begin{aligned} 
\mathcal{P}2: \quad & \min_{q, \boldsymbol{\Delta}, \boldsymbol{\rho}, \boldsymbol{\delta}} \mathcal{H}\left(q, \boldsymbol{\Delta}, \boldsymbol{\rho},\boldsymbol{\delta}\right) \\
& \text{s.t.} \quad (\ref{Zb}), (\ref{Zc}), (\ref{Zd}), (\ref{Ze}), (\ref{Zf}), \\
& \quad \quad 0 \leq q \leq 1.
\end{aligned} 
\end{equation}

To solve it in a low-complexity manner, we decompose this mixed continuous-discrete, non-convex optimization into four subproblems based on variable blocks as
\begin{equation} 
\begin{aligned} 
\mathcal{P}_{2.1}: & \quad \min_{q} \mathcal{H}(q, \boldsymbol{\Delta}, \boldsymbol{\rho}, \boldsymbol{\delta}), \\
\mathcal{P}_{2.2}: & \quad \min_{\boldsymbol{\Delta}} \mathcal{H}(q, \boldsymbol{\Delta}, \boldsymbol{\rho}, \boldsymbol{\delta}), \\
\mathcal{P}_{2.3}: & \quad \min_{\boldsymbol{\rho}} \mathcal{H}(q, \boldsymbol{\Delta}, \boldsymbol{\rho}, \boldsymbol{\delta}), \\
\mathcal{P}_{2.4}: & \quad \min_{\boldsymbol{\delta}} \mathcal{H}(q, \boldsymbol{\Delta}, \boldsymbol{\rho}, \boldsymbol{\delta}).
\end{aligned} 
\end{equation}

\subsection{BO within Variable Blocks}
In each block, the subproblem is treated as a black-box function optimization task and solved via BO.

\subsubsection{Surrogate Model}
The unknown objective function $\mathcal{H}(\boldsymbol{x})$ is modeled as a Gaussian process (GP). Before any observations, the surrogate model $\hat{\mathcal{H}}(\boldsymbol{x})$ is defined by the GP prior
\begin{equation}
\hat{\mathcal{H}}(\boldsymbol{x}) \sim \mathcal{GP}(0, \kappa(\boldsymbol{x}, \boldsymbol{x}')),
\end{equation}
where $\kappa(\cdot, \cdot)$ denotes the covariance kernel. In our implementation, we adopt the radial basis function (RBF) kernel
\begin{equation}
\kappa(\boldsymbol{x}, \boldsymbol{x}') = \exp\left(-\frac{\|\boldsymbol{x} - \boldsymbol{x}'\|^2}{2l^2}\right),
\end{equation}
with length scale hyperparameter $l$ controlling the smoothness.

After collecting $M$ observations $\Xi_M = \{(\boldsymbol{x}_i, \mathcal{H}_i)\}_{i=1}^M$, the posterior distribution of the surrogate model at any new point $\boldsymbol{x}$ becomes
\begin{equation}
\hat{\mathcal{H}}(\boldsymbol{x}) \sim \mathcal{GP}(\mu_M(\boldsymbol{x}), \sigma_M^2(\boldsymbol{x})),
\end{equation}
where
\begin{equation}
\mu_M(\boldsymbol{x}) = \boldsymbol{k}_M^\top(\boldsymbol{x}) \boldsymbol{K}_M^{-1} \mathcal{H}_M,
\end{equation}
\begin{equation}
\sigma_M^2(\boldsymbol{x}) = \kappa(\boldsymbol{x}, \boldsymbol{x}) - \boldsymbol{k}_M^\top(\boldsymbol{x}) \boldsymbol{K}_M^{-1} \boldsymbol{k}_M(\boldsymbol{x}),
\end{equation}
and $\boldsymbol{k}_M(\boldsymbol{x}) = [\kappa(\boldsymbol{x}_1, \boldsymbol{x}), \ldots, \kappa(\boldsymbol{x}_M, \boldsymbol{x})]^\top$.

\subsubsection{Acquisition Function}
To balance exploration and exploitation, we use the Probability of improvement (PI) acquisition function
\begin{equation}
\theta(\boldsymbol{x}) = P(\hat{\mathcal{H}}(\boldsymbol{x}) \leq \mathcal{H}_M^* + \varsigma) = 1 - \Phi\left(\frac{\mu_M(\boldsymbol{x}) - \mathcal{H}_M^* - \varsigma}{\sigma_M(\boldsymbol{x})}\right),
\end{equation}
where $\mathcal{H}_M^* = \min_{1 \leq i \leq M} \mathcal{H}_i$ is the best observed value, $\varsigma > 0$ is a trade-off parameter, and $\Phi(\cdot)$ is the cumulative distribution function of the standard normal distribution. The next sampling point is obtained by
\begin{equation}
\boldsymbol{x}_{M+1} = \arg\max_{\boldsymbol{x}} \theta(\boldsymbol{x}).
\end{equation}
Once the acquisition function is maximized and a new sampling point $\boldsymbol{x}_{M+1}$ is selected, the true objective function $\mathcal{H}(\boldsymbol{x}_{M+1})$ is evaluated and appended to the dataset. The surrogate model is then updated using all available observations $\Xi_{M+1} = \Xi_M \cup \{(\boldsymbol{x}_{M+1}, \mathcal{H}_{M+1})\}$. This procedure is repeated iteratively, allowing the optimizer to refine its search towards the optimal region of the objective landscape. The process terminates when a predefined number of evaluations is reached or when the improvement over recent iterations falls below a certain threshold. The BO routine used for each variable block is detailed in Algorithm~\ref{alg:bo}.

\subsection{Joint Optimization via Block Coordinate Descent}

The BCD algorithm is employed to iteratively optimize the objective function $\mathcal{H}(q, \boldsymbol{\Delta}, \boldsymbol{\rho}, \boldsymbol{\delta})$ by updating one group of variables at a time while holding the others fixed. Each iteration consists of a complete cycle over all variable blocks, and the optimization proceeds in a predetermined order.

Let $(q^r, \boldsymbol{\Delta}^r, \boldsymbol{\rho}^r, \boldsymbol{\delta}^r)$ denote the solution at the $r$-th iteration. The update procedure for the $r$-th iteration $r$ follows these steps:

\begin{itemize}
    \item Fix $\boldsymbol{\Delta}^{r-1}, \boldsymbol{\rho}^{r-1}, \boldsymbol{\delta}^{r-1}$ and update $q^r$;
    \item Fix $q^r, \boldsymbol{\rho}^{r-1}, \boldsymbol{\delta}^{r-1}$ and update $\boldsymbol{\Delta}^r$;
    \item Fix $q^r, \boldsymbol{\Delta}^r, \boldsymbol{\delta}^{r-1}$ and update $\boldsymbol{\rho}^r$;
    \item Fix $q^r, \boldsymbol{\Delta}^r, \boldsymbol{\rho}^r$ and update $\boldsymbol{\delta}^r$.
\end{itemize}

Each subproblem is treated as a black-box optimization task and is solved using a GP-based BO with the PI acquisition function strategy for both continuous and discrete variables. After all variable blocks are updated, the objective value is re-evaluated to monitor progress. The iterative procedure continues until a convergence criterion is satisfied. Specifically, let $\epsilon_{\text{tol}}>0$ be the prescribed relative-tolerance threshold, and $r_{\max}\in\mathbb{N}$ denote the maximum number of BCD iterations. The iteration stops once
\[
    \frac{\lvert \mathcal{H}^{(r)} - \mathcal{H}^{(r-1)} \rvert}{\lvert \mathcal{H}^{(r-1)} \rvert} 
    < \epsilon_{\text{tol}}
    \quad\text{or}\quad
    r \ge r_{\max}.
\]

 This coordinated update framework allows efficient exploration of the mixed-variable search space and ensures compatibility with a wide range of non-convex and heterogeneous optimization problems. The procedure of the BCD framework is summarized in Algorithm~\ref{alg:bcd}.
\begin{algorithm}[t]
    \caption{BO for Subproblem}
    \label{alg:bo}
    \begin{algorithmic}
        \REQUIRE Variable block $x$ to optimize; other variables fixed; $\mathcal{H}(x)$ as black-box function
        \ENSURE Approximate optimal value $x^*$
        \STATE Initialize dataset $\Xi_1 = \{(x_1, \mathcal{H}(x_1))\}$ with a random sample
        \FOR{$m = 1$ to $M_{\rm max}$}
            \STATE Fit GP surrogate model $\hat{\mathcal{H}}(x)$ using $\Xi_M$
            \STATE Compute acquisition function $\theta(x)$ based on PI
            \STATE $x_{M+1} \gets \arg\max_x \theta(x)$
            \STATE Evaluate $\mathcal{H}(x_{M+1})$ and update $\Xi_{M+1} = \Xi_M \cup \{(x_{M+1}, \mathcal{H}(x_{M+1}))\}$
        \ENDFOR
        \STATE $x^* \gets \arg\min_{(x, \mathcal{H}(x)) \in \Xi_{M}} \mathcal{H}(x)$
    \end{algorithmic}
\end{algorithm}
\begin{algorithm}[t]
    \caption{BCD Framework}
    \label{alg:bcd}
    \begin{algorithmic}
        \REQUIRE $\epsilon_{\rm tol}, r_{\rm max}$; initial values $q_0$, $\boldsymbol{\Delta}_0$, $\boldsymbol{\rho}_0$, $\boldsymbol{\delta}_0$
        \ENSURE Optimized solution $(q^*, \boldsymbol{\Delta}^*, \boldsymbol{\rho}^*, \boldsymbol{\delta}^*)$
        \STATE Initialize $r = 0$, compute $\mathcal{H}_0$
        \WHILE{$r < r_{\rm max}$ \AND gap $\geq \epsilon_{\rm tol}$}
            \STATE $q_r^* \gets \textsc{BO}(q,\ \boldsymbol{\Delta}_{r-1}, \boldsymbol{\rho}_{r-1}, \boldsymbol{\delta}_{r-1})$
            \STATE $\boldsymbol{\Delta}_r^* \gets \textsc{BO}(\boldsymbol{\Delta},\ q_r^*, \boldsymbol{\rho}_{r-1}, \boldsymbol{\delta}_{r-1})$
            \STATE $\boldsymbol{\rho}_r^* \gets \textsc{BO}(\boldsymbol{\rho},\ q_r^*, \boldsymbol{\Delta}_r^*, \boldsymbol{\delta}_{r-1})$
            \STATE $\boldsymbol{\delta}_r^* \gets \textsc{BO}(\boldsymbol{\delta},\ q_r^*, \boldsymbol{\Delta}_r^*, \boldsymbol{\rho}_r^*)$
            \STATE Evaluate objective $\mathcal{H}_r = \mathcal{H}(q_r^*, \boldsymbol{\Delta}_r^*, \boldsymbol{\rho}_r^*, \boldsymbol{\delta}_r^*)$
            \STATE $r \gets r + 1$
        \ENDWHILE
    \end{algorithmic}
\end{algorithm}

\section{Experimental Results}
\textcolor{black}{We conducted experiments on the CIFAR-10 dataset \cite{krizhevsky2009learning}, which consists of 60,000 color images across 10 categories, with 50,000 samples used for training and 10,000 for testing. Each image has a resolution of $32 \times 32$ pixels. The dataset is partitioned in a non-i.i.d. and unbalanced manner across 100 devices, where each device holds a varying number of samples drawn from a limited subset of classes. We employed the ResNet-18 architecture \cite{he2016deep} to perform image classification\footnote{\textcolor{black}{The proposed approach is compatible with a wide range of ML models and datasets. Its advantages become more significant as model size and complexity grow.}}. Furthermore, we utilized a pre-trained diffusion model \cite{yang2023denoising} to augment local training data, aiming to mitigate data scarcity and address non-i.i.d. challenges in FL. Unless otherwise specified, each training round randomly selects 10 devices for participation. The default experimental settings are summarized in TABLE~\ref{ParameterSetting}.}

\textcolor{black}{We conduct a series of experiments to evaluate the effectiveness of our approach. Depending on the specific experimental objective, a subset of the following baselines is selected for comparison in each figure.}
\begin{itemize}
    \item \textcolor{black}{\textbf{Traditional FL (TFL)}: After performing local training, the devices directly upload the local gradients to the BS without any optimization techniques.}
    
    \item \textcolor{black}{\textbf{FedDPQ(proposed)}: The proposed scheme that integrates diffusion-based data augmentation, model pruning, quantization, and transmission power control.}
    
    \item \textcolor{black}{\textbf{FedDPQ-noDA}: A degraded version of FedDPQ in which the data augmentation module is disabled, while pruning, quantization, and power control are retained.}
    
    \item \textcolor{black}{\textbf{FedDPQ-noPQ}: A reduced version of FedDPQ where the pruning and quantization modules are removed, while the data augmentation and transmission power control components are preserved.}
    
    \item \textcolor{black}{\textbf{FedDPQ-noPC}: A variant of FedDPQ where the transmission power control module is disabled, while data augmentation, model pruning, and quantization are retained.}
\end{itemize}

\begin{table}[]
\caption{Parameter Settings}
\label{ParameterSetting}
\centering
{\color{black} 
\begin{tabular}{cccc}
\toprule
\textbf{Parameter}      & \textbf{Value}             & \textbf{Parameter}        & \textbf{Value}                         \\ \midrule
$p^{\rm max}$           & 0.1W                       & $c^{\text{tr}}_0$         & $2.7\times10^8$ cycles/sample         \\
$p^{\rm min}$           & 0.01W                      & $N_0$                     & $-174$ dBm/Hz                         \\
$B_u^{\rm UL}$          & 1MHz                       & $c^{\text{gen}}_0$        & $2.2\times10^8$ cycles/sample         \\
$\eta$                  & 0.001                      & $\varrho$                 & $1.25\times10^{-26}$                  \\
$\gamma$                & 3                          & $I_u$                     & $\mathrm{U}[10^{-8}, 2\times10^{-8}]$ \\
$f_u$                   & $\mathrm{U}[20, 50]$ MHz   & $d_u$                     & $\mathrm{U}[100, 300]$ m              \\
$\rho^{\rm max}$        & 0.3                        & $\rho^{\rm min}$          & 0.1                                   \\
$\Delta^{\rm min}$      & 0.1                        & $\Delta^{\rm max}$        & 0.4                                   \\
$\delta^{\rm min}$      & 6                          & $\delta^{\rm max}$        & 16                                    \\ \bottomrule
\end{tabular}
}
\end{table}

\begin{figure*}[t]
\centering
\subfigure[Energy consumption ($\pi = 0.6$).]{
\begin{minipage}{0.30\linewidth}
\centering
\includegraphics[width=\linewidth]{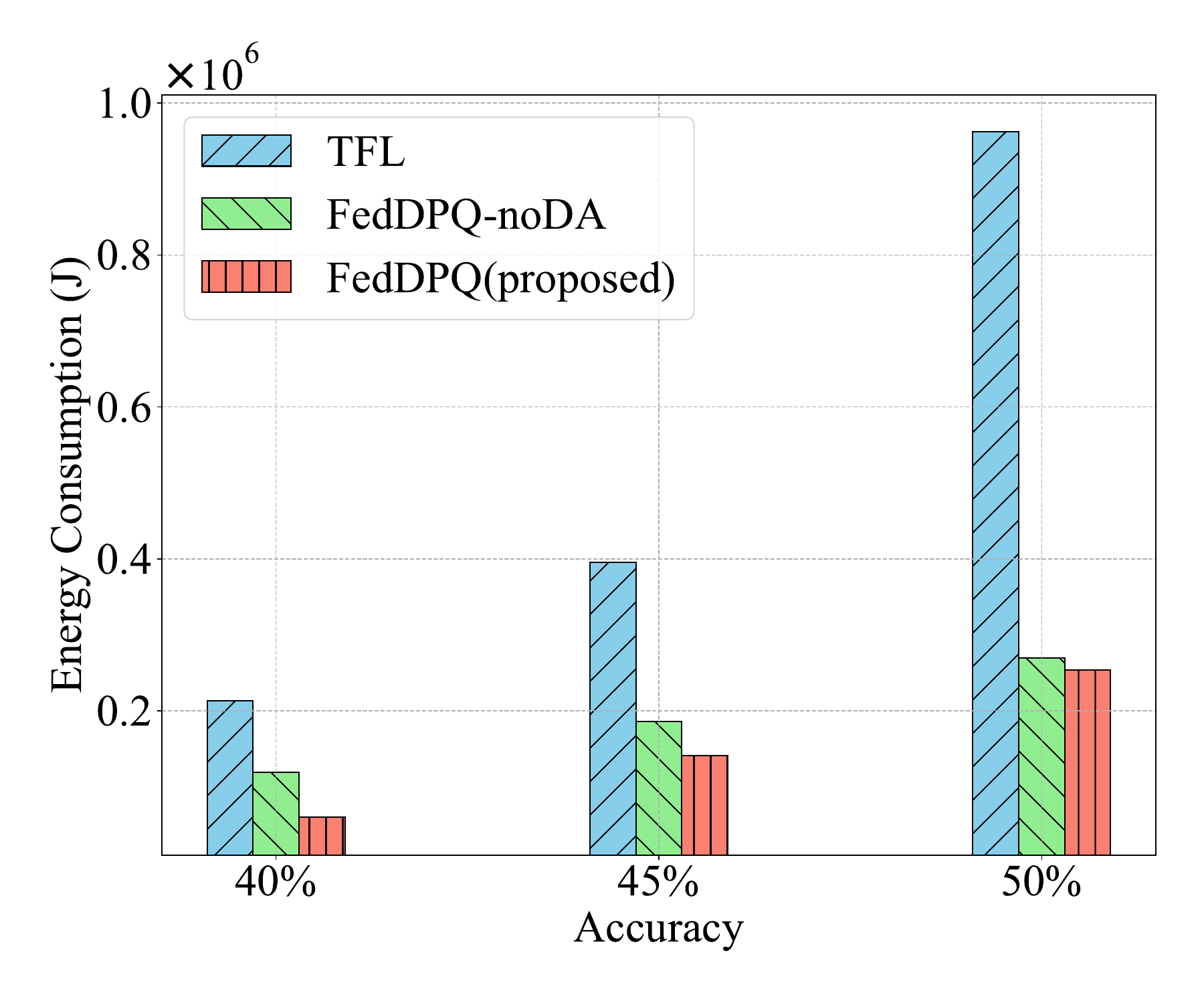}
\label{fig:dir06_energy}
\end{minipage}}
\hfill
\subfigure[Energy consumption ($\pi  = 1.2$).]{
\begin{minipage}{0.30\linewidth}
\centering
\includegraphics[width=\linewidth]{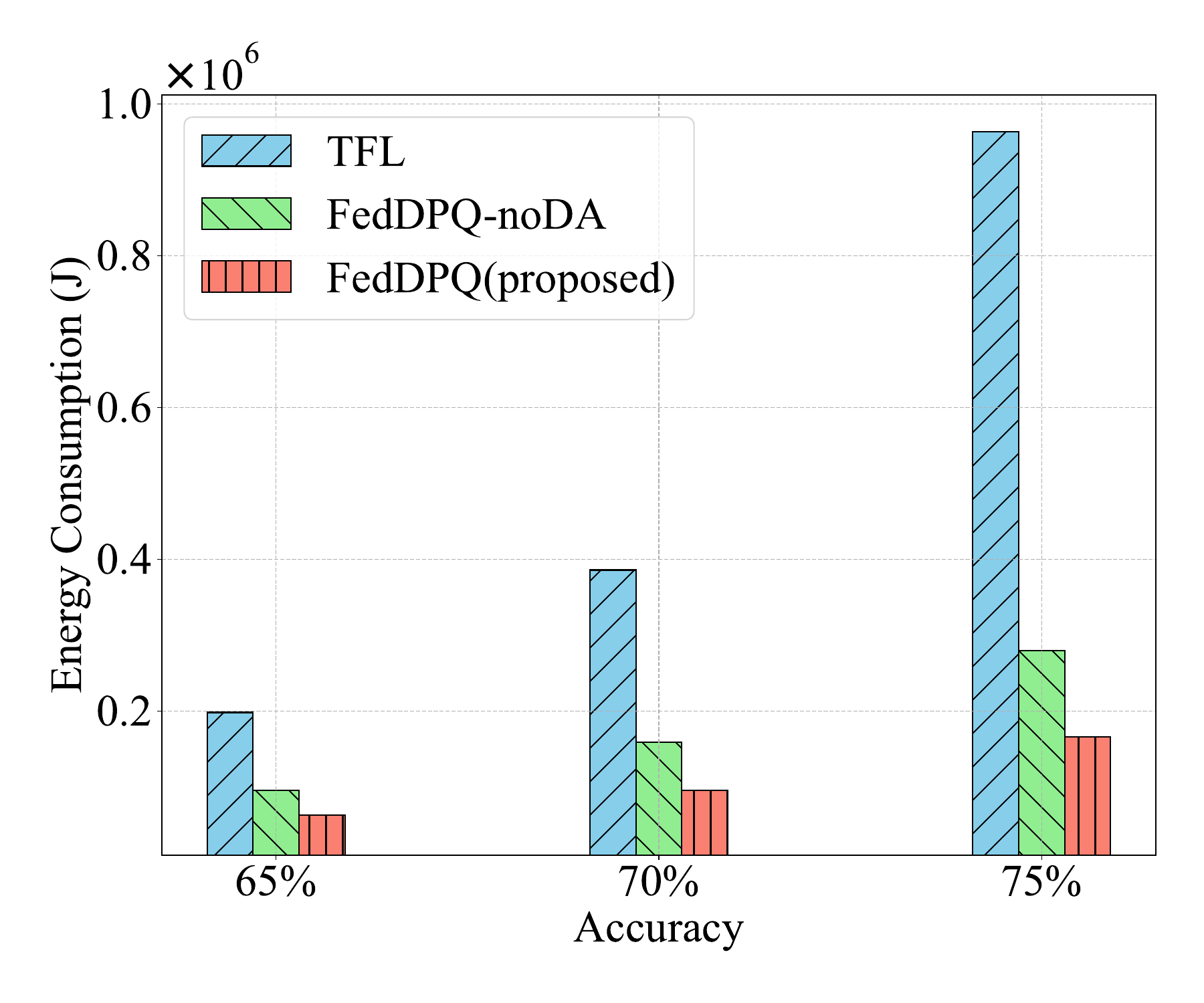}
\label{fig:dir12_energy}
\end{minipage}}
\hfill
\subfigure[Energy consumption ($\pi = 1.5$).]{
\begin{minipage}{0.30\linewidth}
\centering
\includegraphics[width=\linewidth]{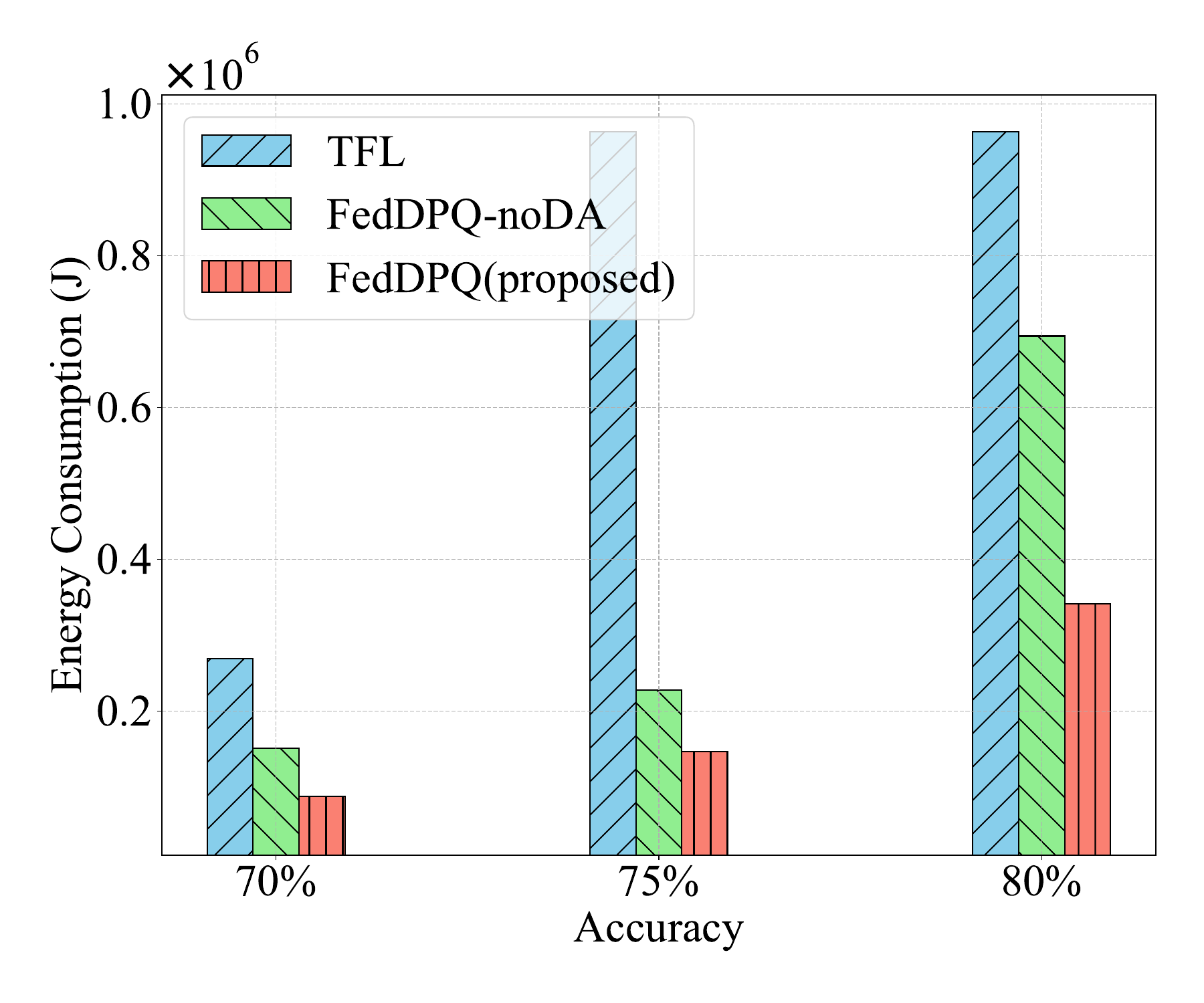}
\label{fig:dir15_energy}
\end{minipage}}

\subfigure[Test accuracy ($\pi = 0.6$).]{
\begin{minipage}{0.30\linewidth}
\centering
\includegraphics[width=\linewidth]{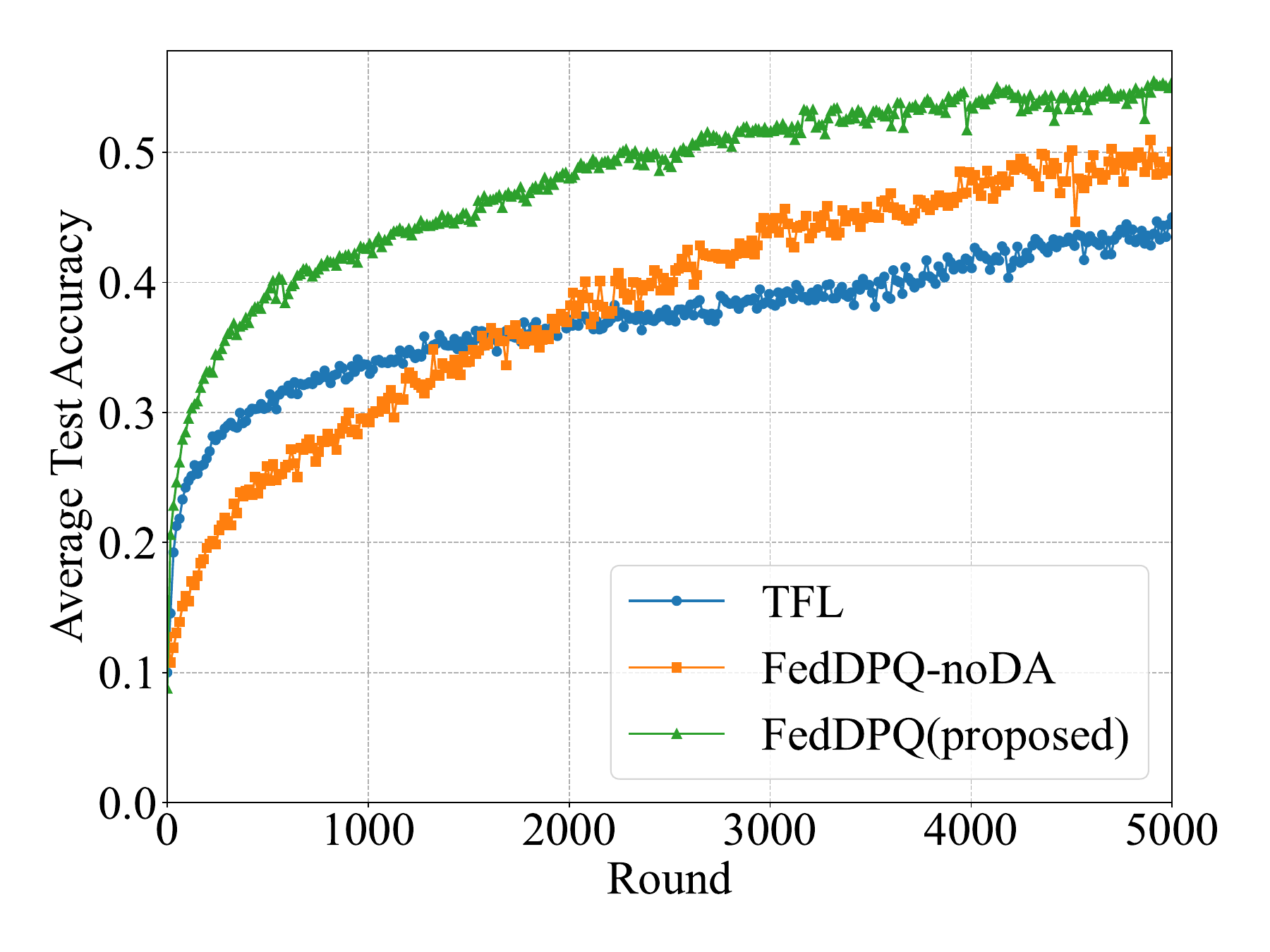}
\label{fig:dir06_accuracy}
\end{minipage}}
\hfill
\subfigure[Test accuracy ($\pi = 1.2$).]{
\begin{minipage}{0.30\linewidth}
\centering
\includegraphics[width=\linewidth]{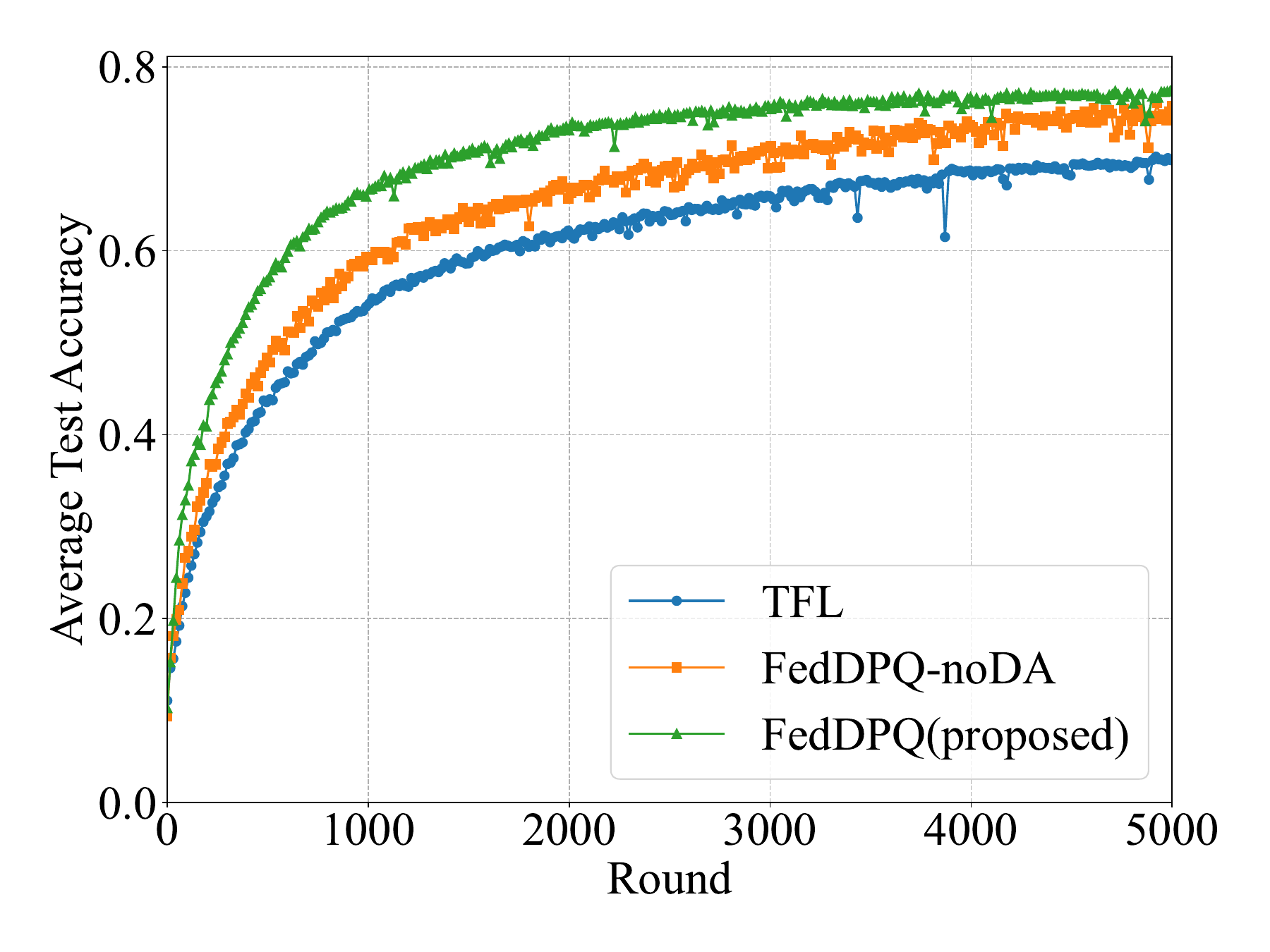}
\label{fig:dir12_accuracy}
\end{minipage}}
\hfill
\subfigure[Test accuracy ($\pi = 1.5$).]{
\begin{minipage}{0.30\linewidth}
\centering
\includegraphics[width=\linewidth]{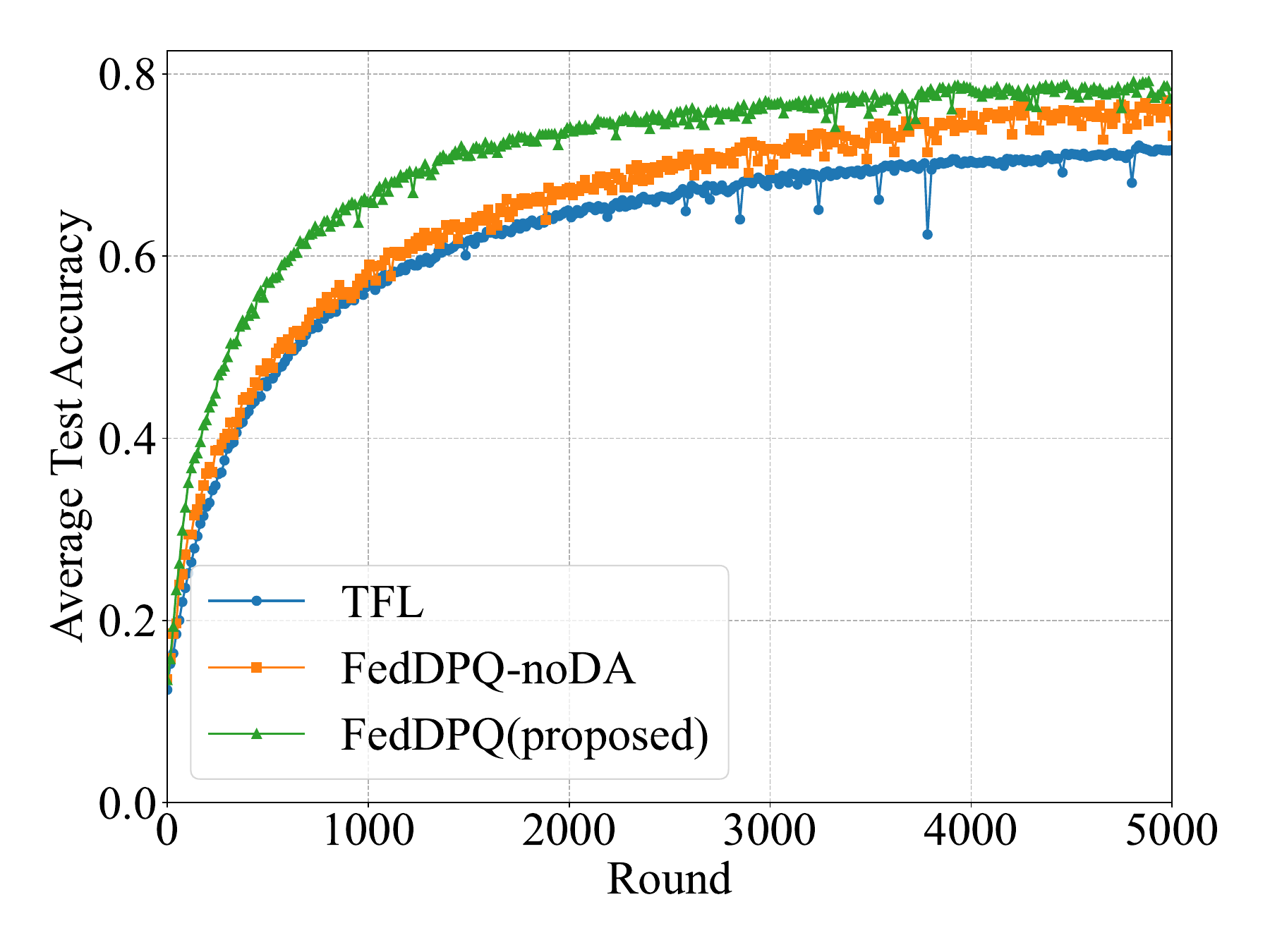}
\label{fig:dir15_accuracy}
\end{minipage}}

\subfigure[Training loss ($\pi = 0.6$).]{
\begin{minipage}{0.30\linewidth}
\centering
\includegraphics[width=\linewidth]{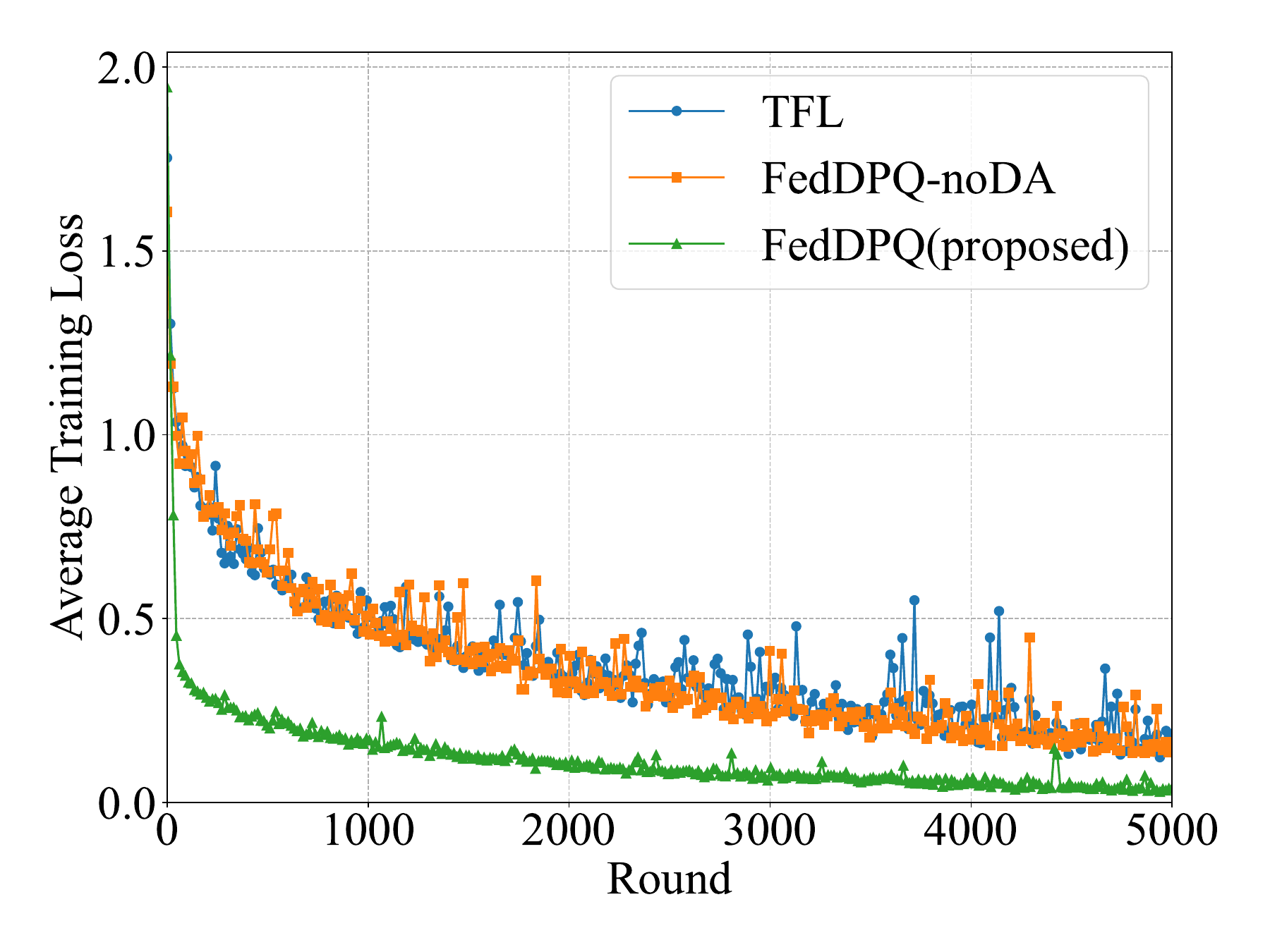}
\label{fig:dir06_loss}
\end{minipage}}
\hfill
\subfigure[Training loss ($\pi= 1.2$).]{
\begin{minipage}{0.30\linewidth}
\centering
\includegraphics[width=\linewidth]{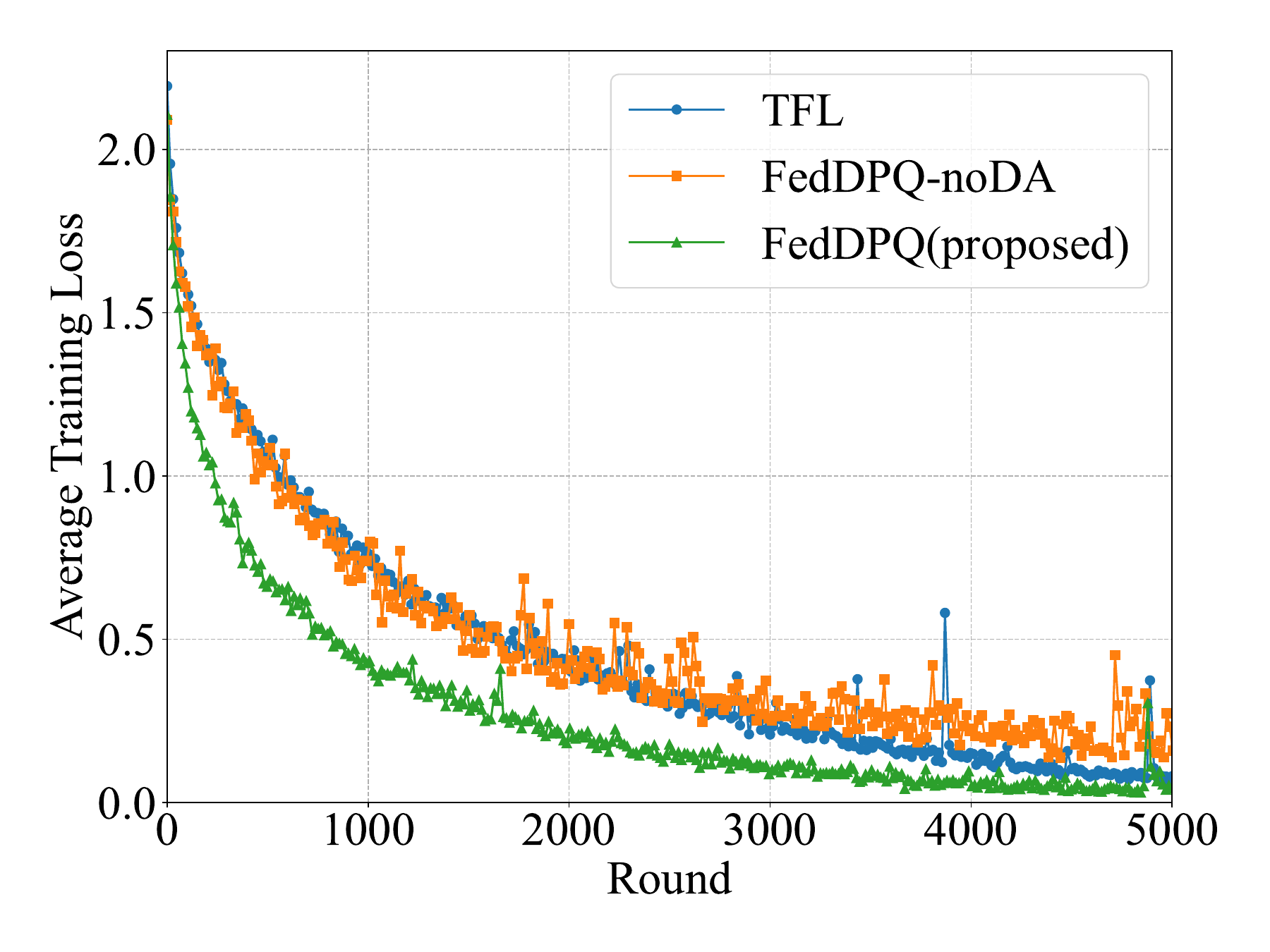}
\label{fig:dir12_loss}
\end{minipage}}
\hfill
\subfigure[Training loss ($\pi = 1.5$).]{
\begin{minipage}{0.30\linewidth}
\centering
\includegraphics[width=\linewidth]{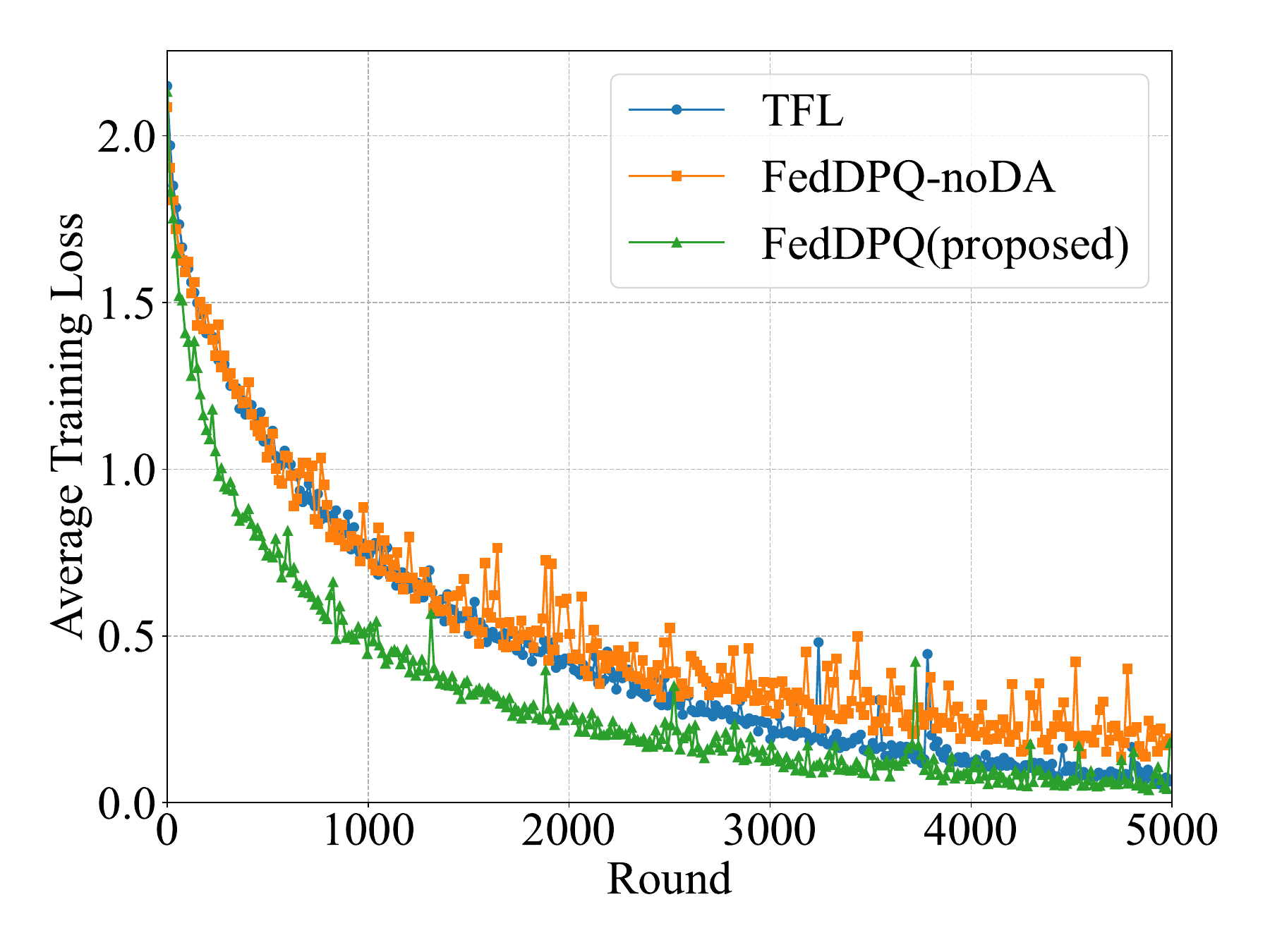}
\label{fig:dir15_loss}
\end{minipage}}

\caption{Performance comparison of FedDPQ and baselines under different levels of data heterogeneity controlled by Dirichlet coefficient $\pi \in \{0.6, 1.2, 1.5\}$. }
\label{fig:dirichlet_comparison}
\end{figure*}

\textcolor{black}{To evaluate the performance of the proposed FedDPQ framework under varying degrees of data heterogeneity, we conduct three groups of experiments by adjusting the Dirichlet distribution parameter $\pi \in \{0.6, 1.2, 1.5\}$, where a smaller $\pi$ indicates a more skewed and non-i.i.d. local data distribution. As shown in Fig.~\ref{fig:dirichlet_comparison}, increasing $\pi$ results in a more balanced data distribution across devices, which improves the convergence speed and reduces energy consumption for all schemes. This highlights the substantial impact of data heterogeneity on federated training: the more skewed the data, the more communication rounds are required to reach the target accuracy, leading to increased energy overhead. Among all settings, FedDPQ consistently outperforms all baselines in terms of energy efficiency, test accuracy, and convergence speed. In contrast, the performance of FedDPQ-noDA and TFL, which do not incorporate data augmentation, degrades significantly under severe heterogeneity (e.g., $\pi=0.6$), exhibiting slower convergence and higher energy consumption \footnote{\textcolor{black}{If a baseline fails to reach the target accuracy within the predefined communication budget, it continues training until the maximum number of rounds (set to 5000 in our experiments), resulting in saturated energy and delay values, i.e., total energy = per-round energy/delay $*$ 5000.}}. The superiority of FedDPQ is largely attributed to the integration of diffusion-based data augmentation. This component plays a critical role in two aspects: (i) it increases the volume of local training data, thereby mitigating underfitting issues in small-sample regimes, and (ii) it introduces synthetic samples covering more diverse classes, effectively alleviating the statistical heterogeneity across devices. These enhancements improve the generalization capability of the global model, enabling FedDPQ to achieve consistently better performance across different levels of data heterogeneity.}

\begin{figure*}[t]
\centering
\subfigure[Energy consumption with $U=15$.]{
\begin{minipage}{0.30\linewidth}
\centering
\includegraphics[width=\linewidth]{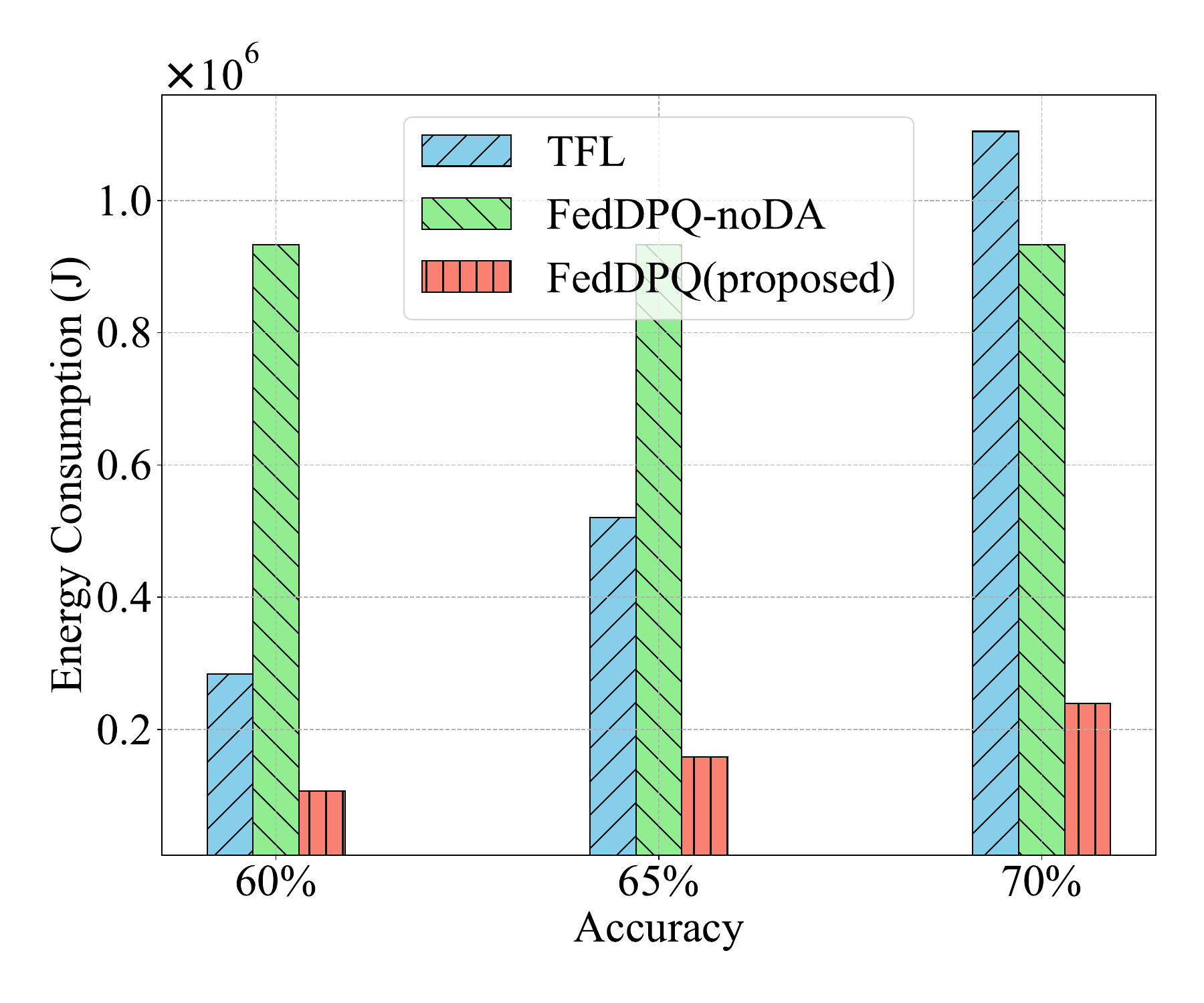}
\label{fig:u15_energy}
\end{minipage}}
\hfill
\subfigure[Energy consumption with $U=20$.]{
\begin{minipage}{0.30\linewidth}
\centering
\includegraphics[width=\linewidth]{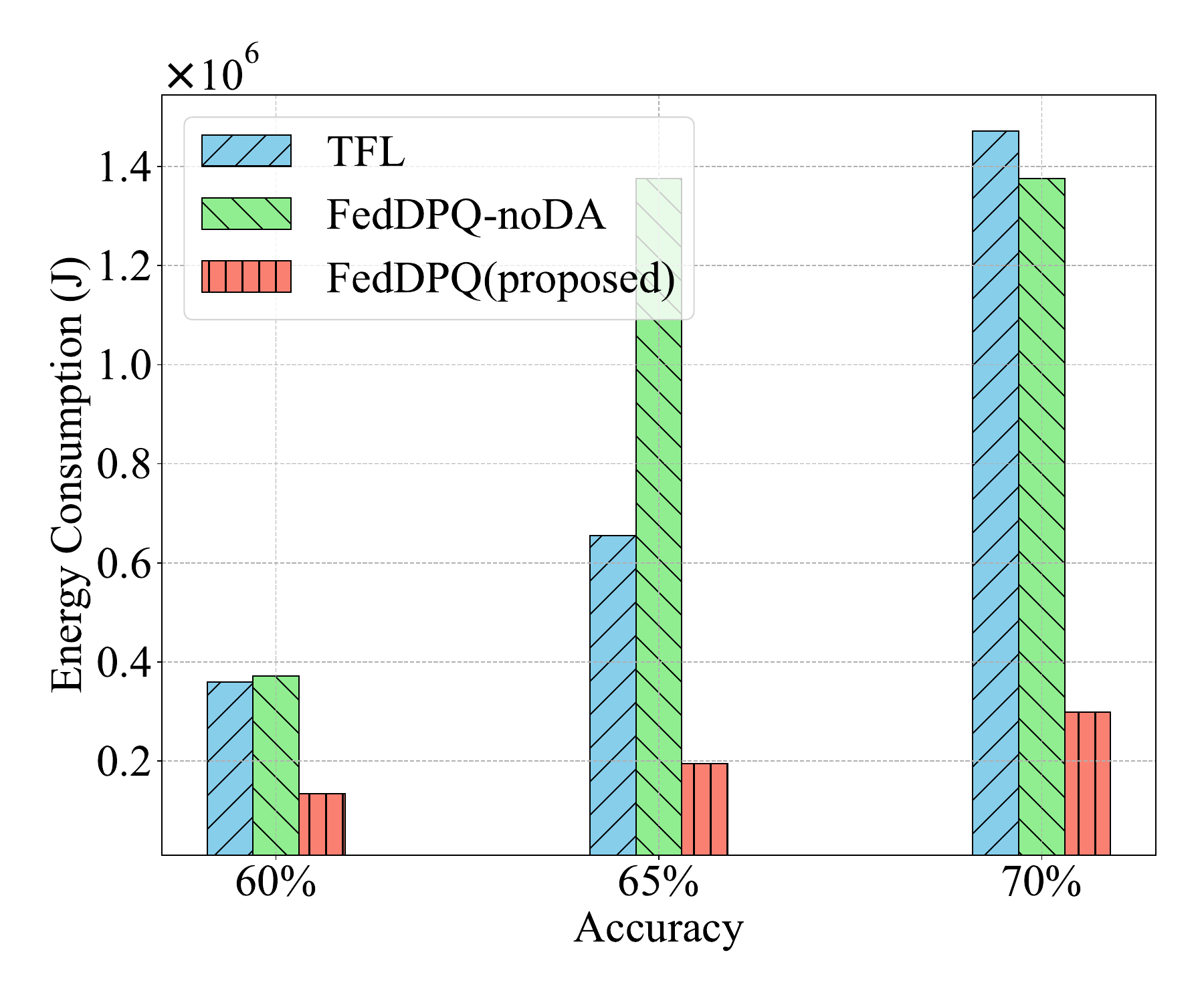}
\label{fig:u20_energy}
\end{minipage}}
\hfill
\subfigure[Energy consumption with $U=30$.]{
\begin{minipage}{0.30\linewidth}
\centering
\includegraphics[width=\linewidth]{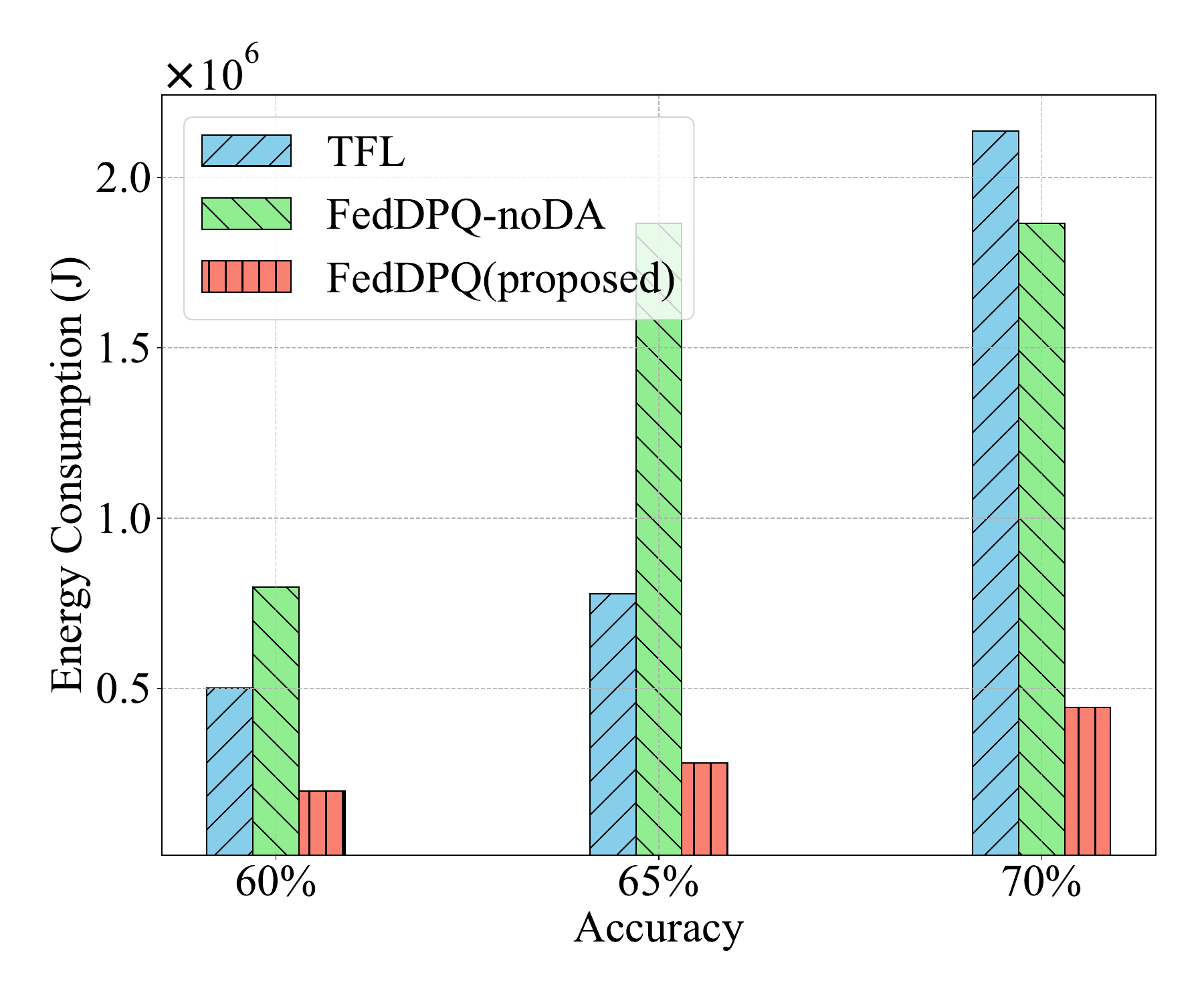}
\label{fig:u30_energy}
\end{minipage}}

\subfigure[Model accuracy with $U=15$.]{
\begin{minipage}{0.30\linewidth}
\centering
\includegraphics[width=\linewidth]{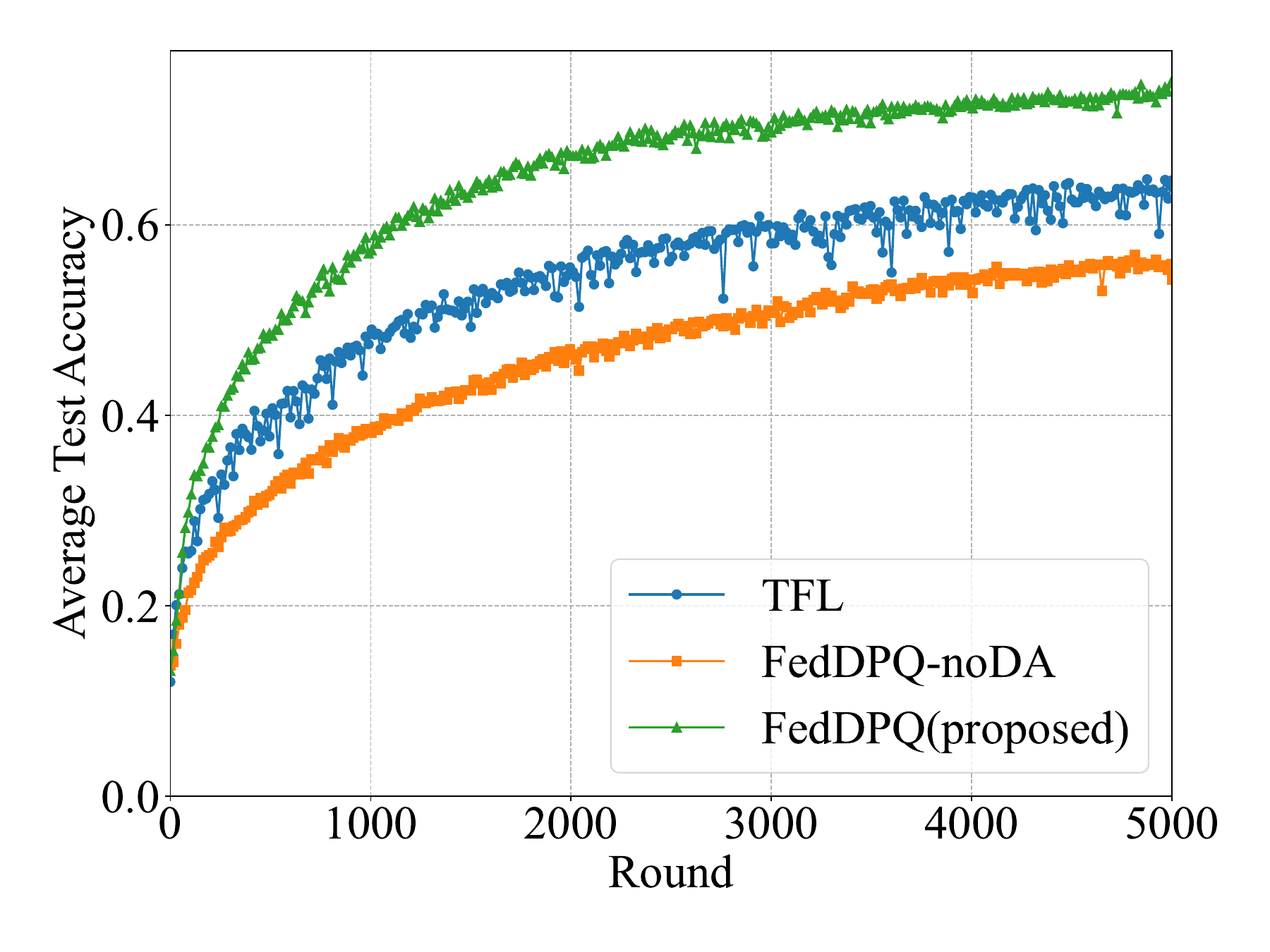}
\label{fig:u15_accuracy}
\end{minipage}}
\hfill
\subfigure[Model accuracy with $U=20$.]{
\begin{minipage}{0.30\linewidth}
\centering
\includegraphics[width=\linewidth]{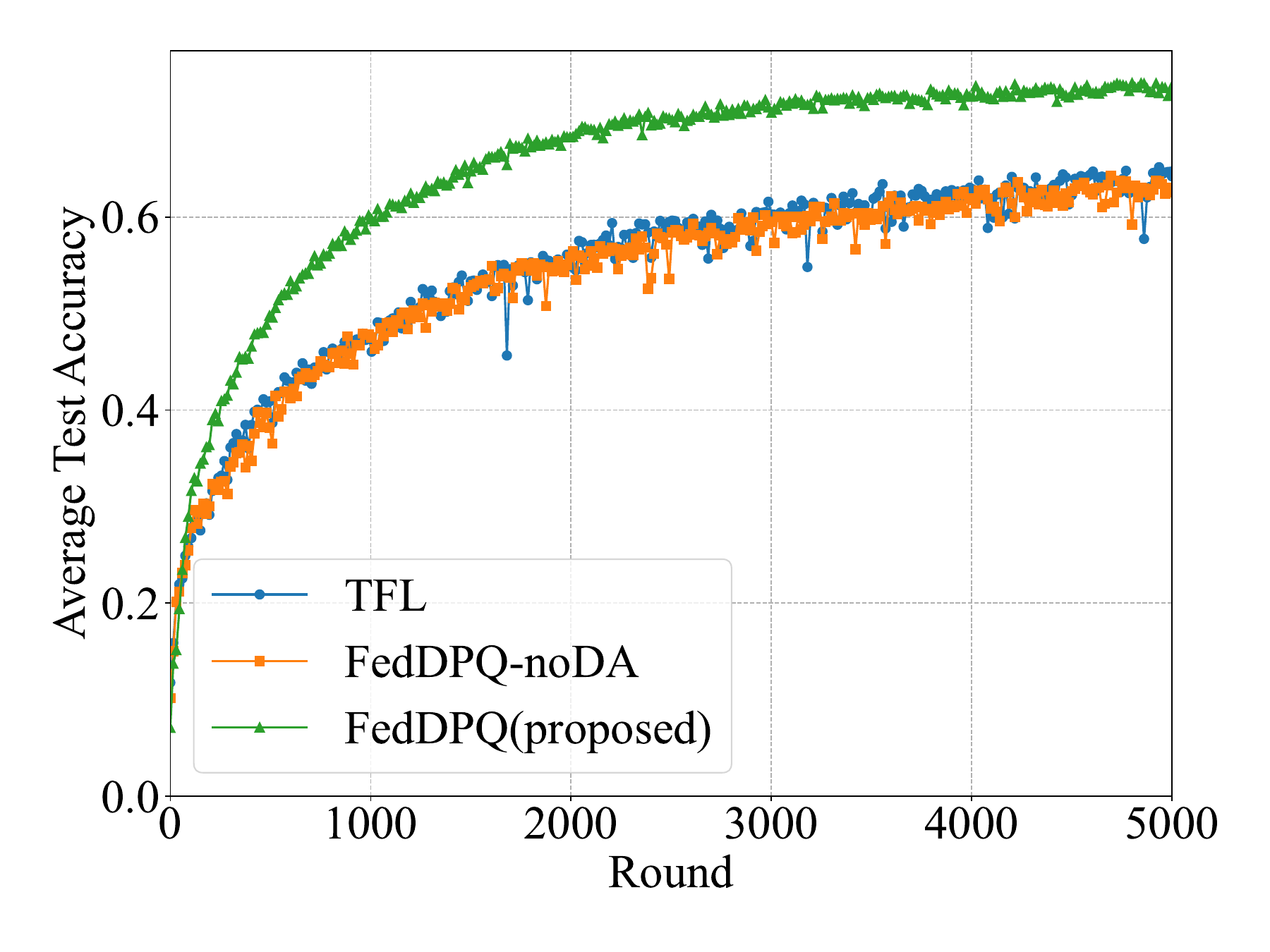}
\label{fig:u20_accuracy}
\end{minipage}}
\hfill
\subfigure[Model accuracy with $U=30$.]{
\begin{minipage}{0.30\linewidth}
\centering
\includegraphics[width=\linewidth]{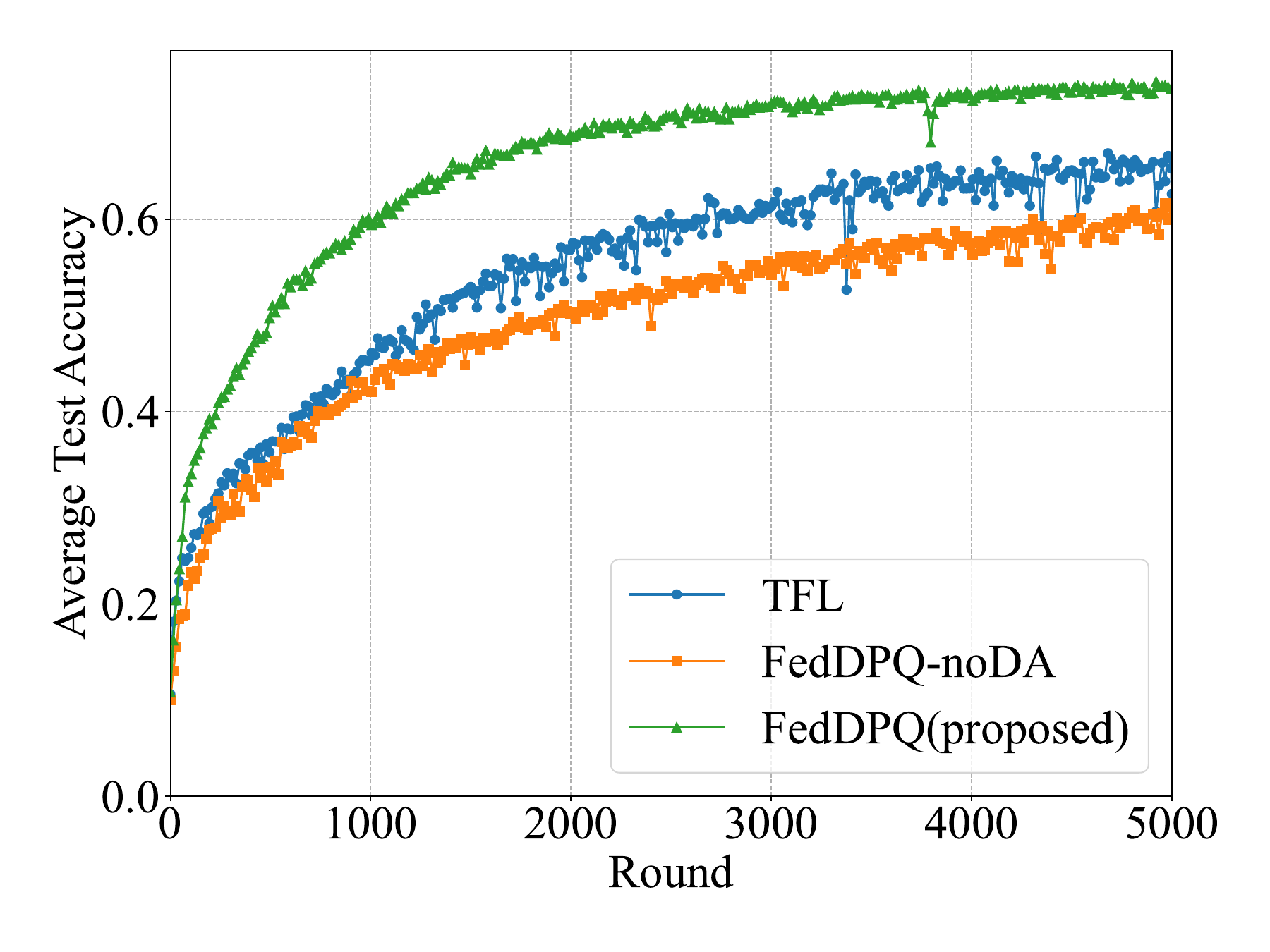}
\label{fig:u30_accuracy}
\end{minipage}}

\subfigure[Loss curve with $U=15$.]{
\begin{minipage}{0.30\linewidth}
\centering
\includegraphics[width=\linewidth]{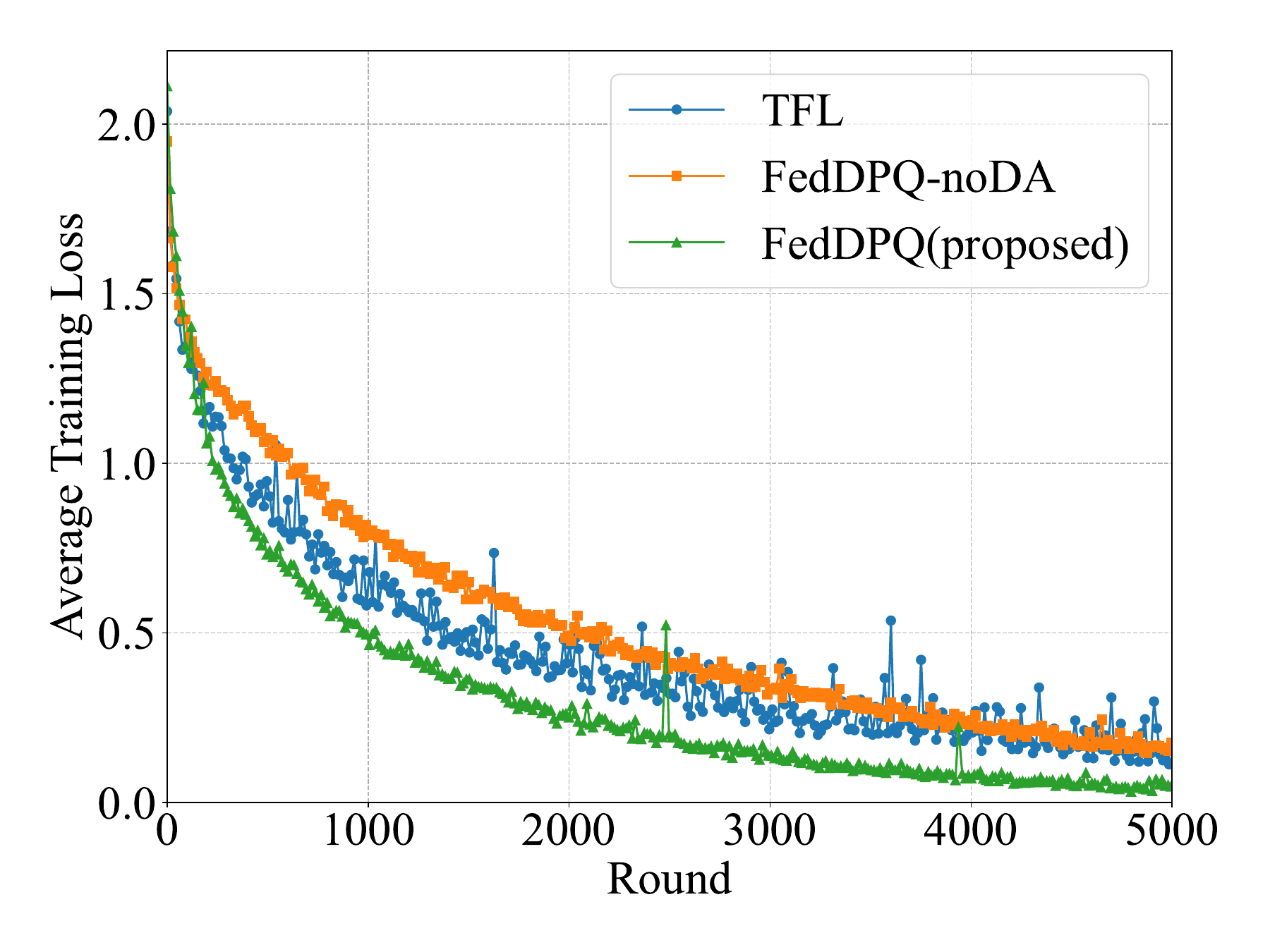}
\label{fig:u15_loss}
\end{minipage}}
\hfill
\subfigure[Loss curve with $U=20$.]{
\begin{minipage}{0.30\linewidth}
\centering
\includegraphics[width=\linewidth]{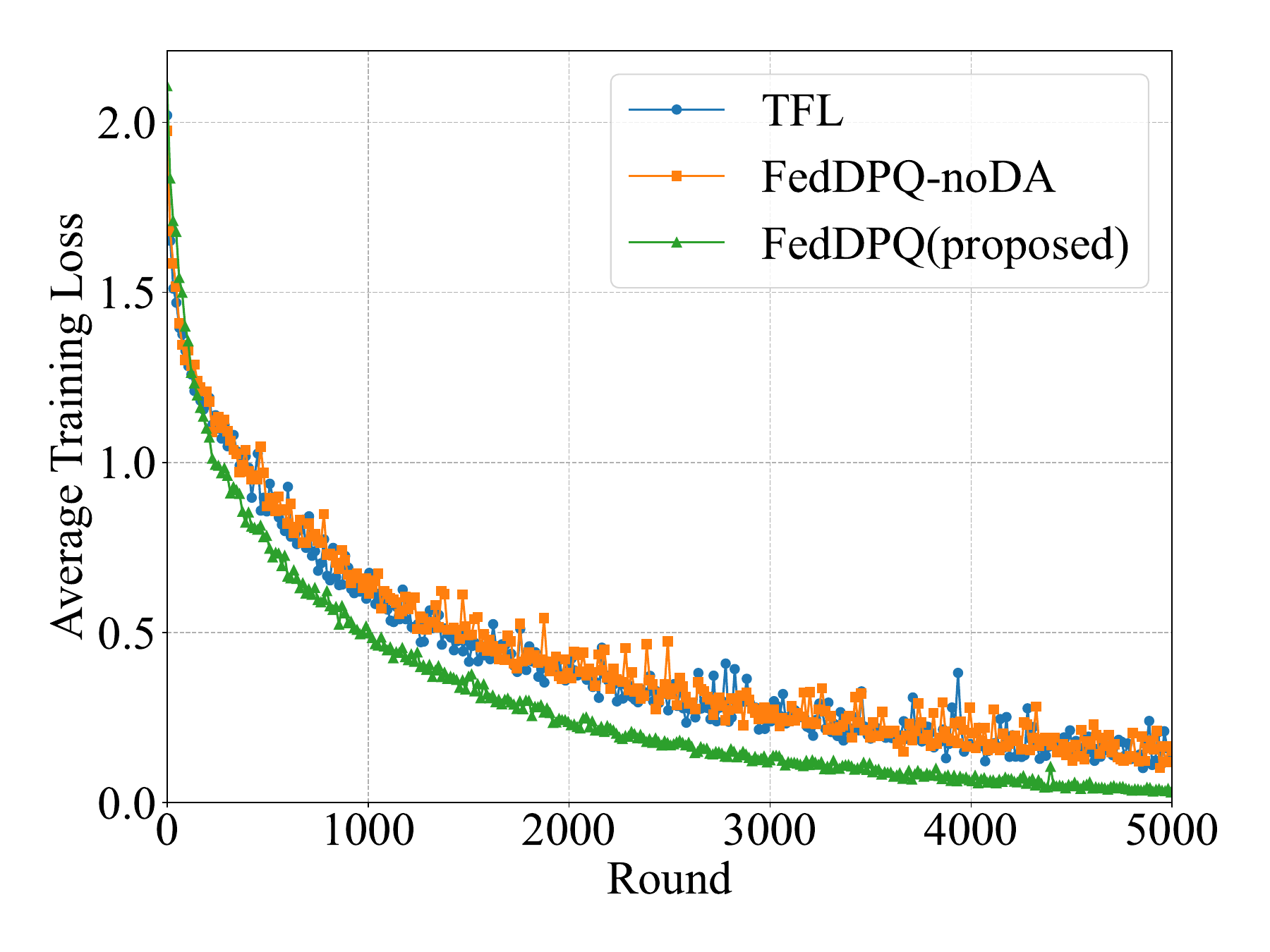}
\label{fig:u20_loss}
\end{minipage}}
\hfill
\subfigure[Loss curve with $U=30$.]{
\begin{minipage}{0.30\linewidth}
\centering
\includegraphics[width=\linewidth]{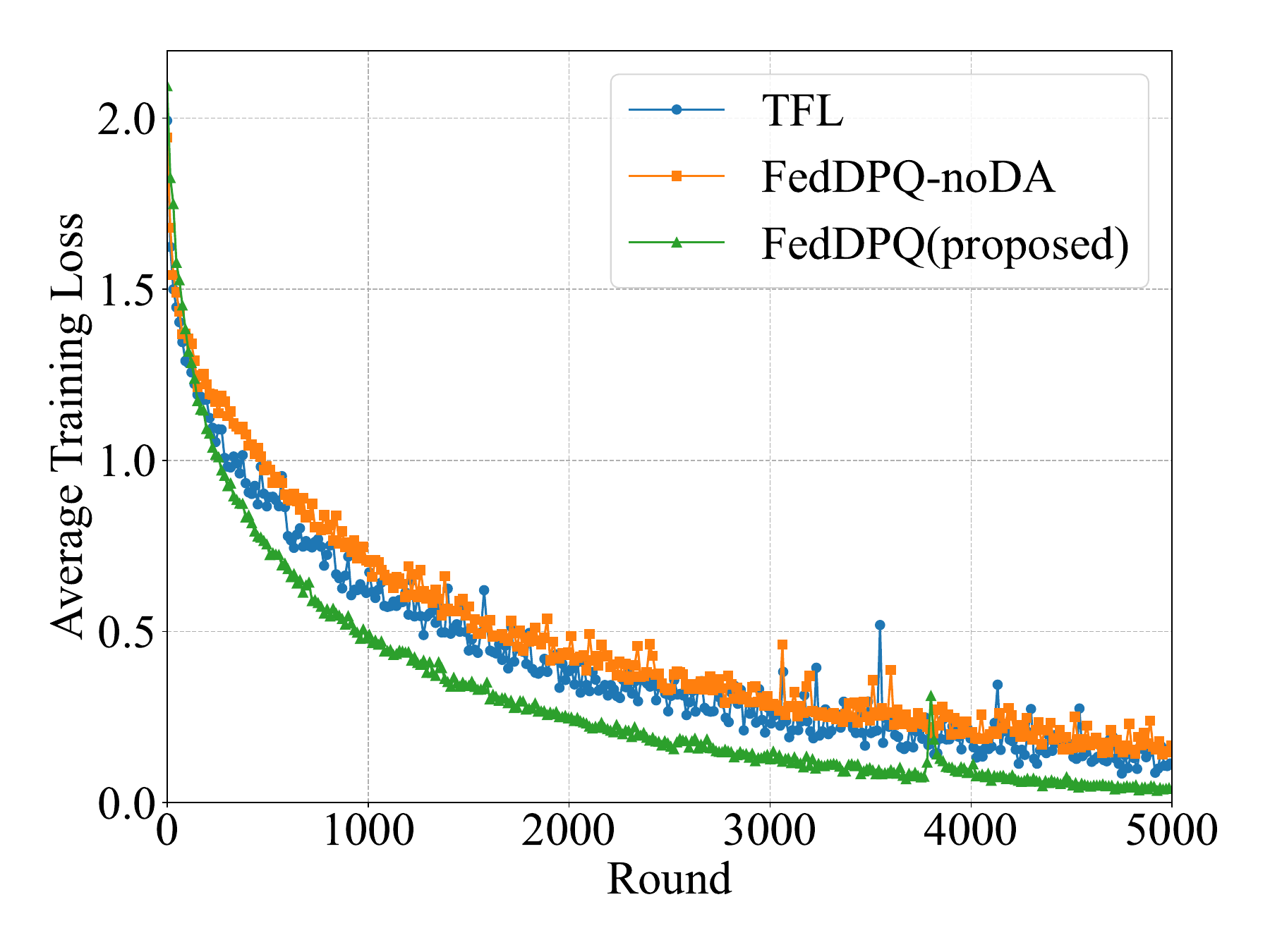}
\label{fig:u30_loss}
\end{minipage}}

\caption{Performance comparison of the proposed scheme under varying numbers of participating devices per round. The number of local participants $U$ is set to 15, 20, and 30.}
\label{fig:device_number_comparison}
\end{figure*}

\textcolor{black}{
In Fig.~\ref{fig:device_number_comparison}, we evaluate the performance of various FL schemes by examining their total energy consumption, test accuracy, and training loss under different numbers of participating devices per training round, specifically $U \in \{15, 20, 30\}$. As shown in Figs.~\ref{fig:u15_energy}–\ref{fig:u30_energy}, increasing the number of participating devices results in higher overall energy consumption, primarily due to the concurrent execution of multiple devices, which leads to a greater aggregate energy cost per round. Although involving more devices provides access to larger aggregated datasets, which may facilitate faster convergence, Figs.~\ref{fig:u15_accuracy}–\ref{fig:u30_accuracy} and \ref{fig:u15_loss}–\ref{fig:u30_loss} reveal that the reduction in the total number of communication rounds is relatively marginal. Consequently, the increased energy expenditure does not yield a proportional improvement in convergence efficiency. Despite the rising energy cost with more participants, our proposed FedDPQ consistently outperforms baseline schemes across all evaluation metrics. Specifically, FedDPQ achieves higher accuracy and faster convergence while maintaining lower total energy consumption. These results demonstrate the scalability and effectiveness of FedDPQ when adapting to varying numbers of participating devices.}

\begin{figure}[hpt]
\centering
\subfigure[Energy consumption.]{
\begin{minipage}{0.7\linewidth}
\centering
\includegraphics[width=\linewidth]{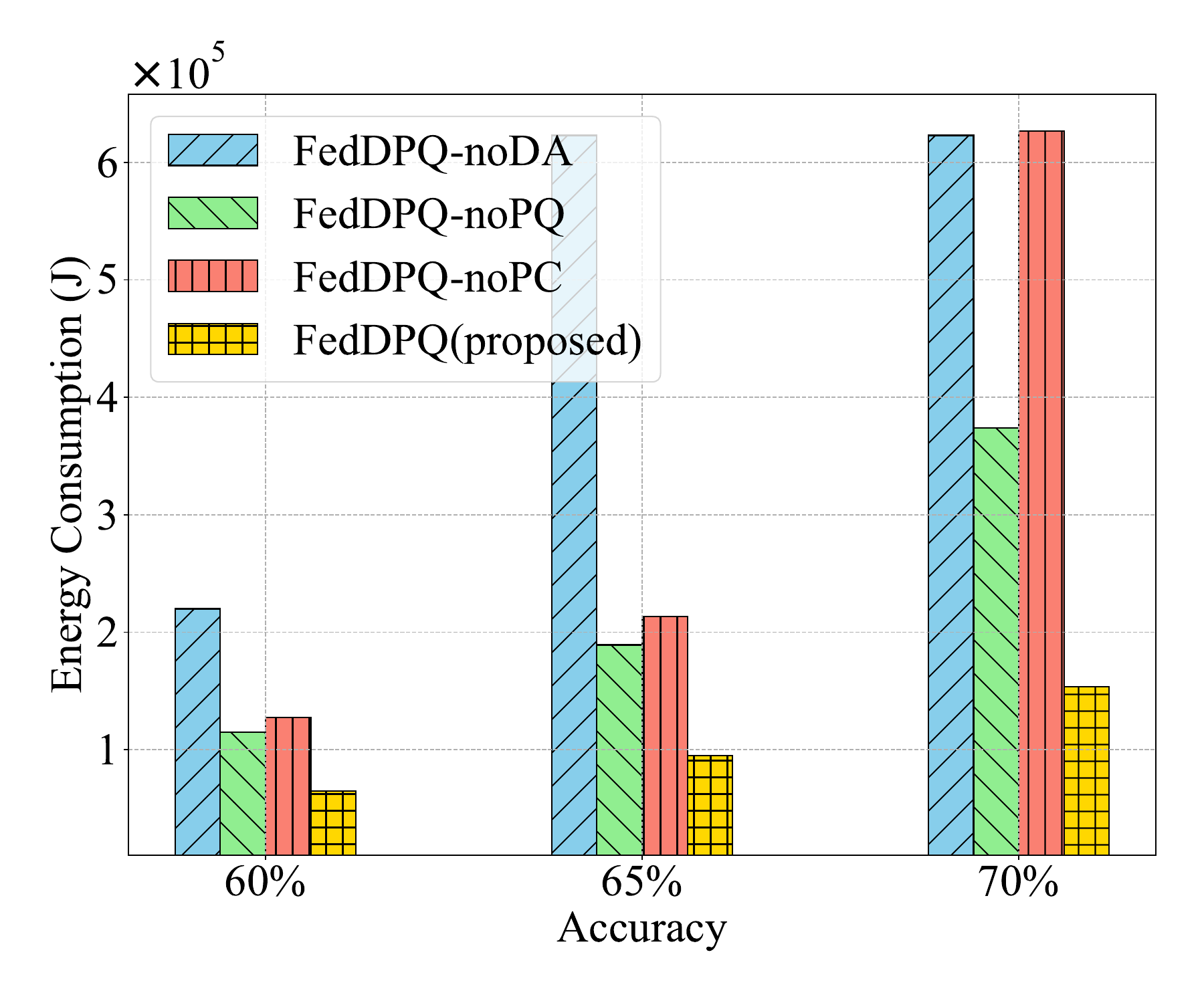}
\label{fig:xiaorong_energy}
\end{minipage}}

\subfigure[Model accuracy.]{
\begin{minipage}{0.7\linewidth}
\centering
\includegraphics[width=\linewidth]{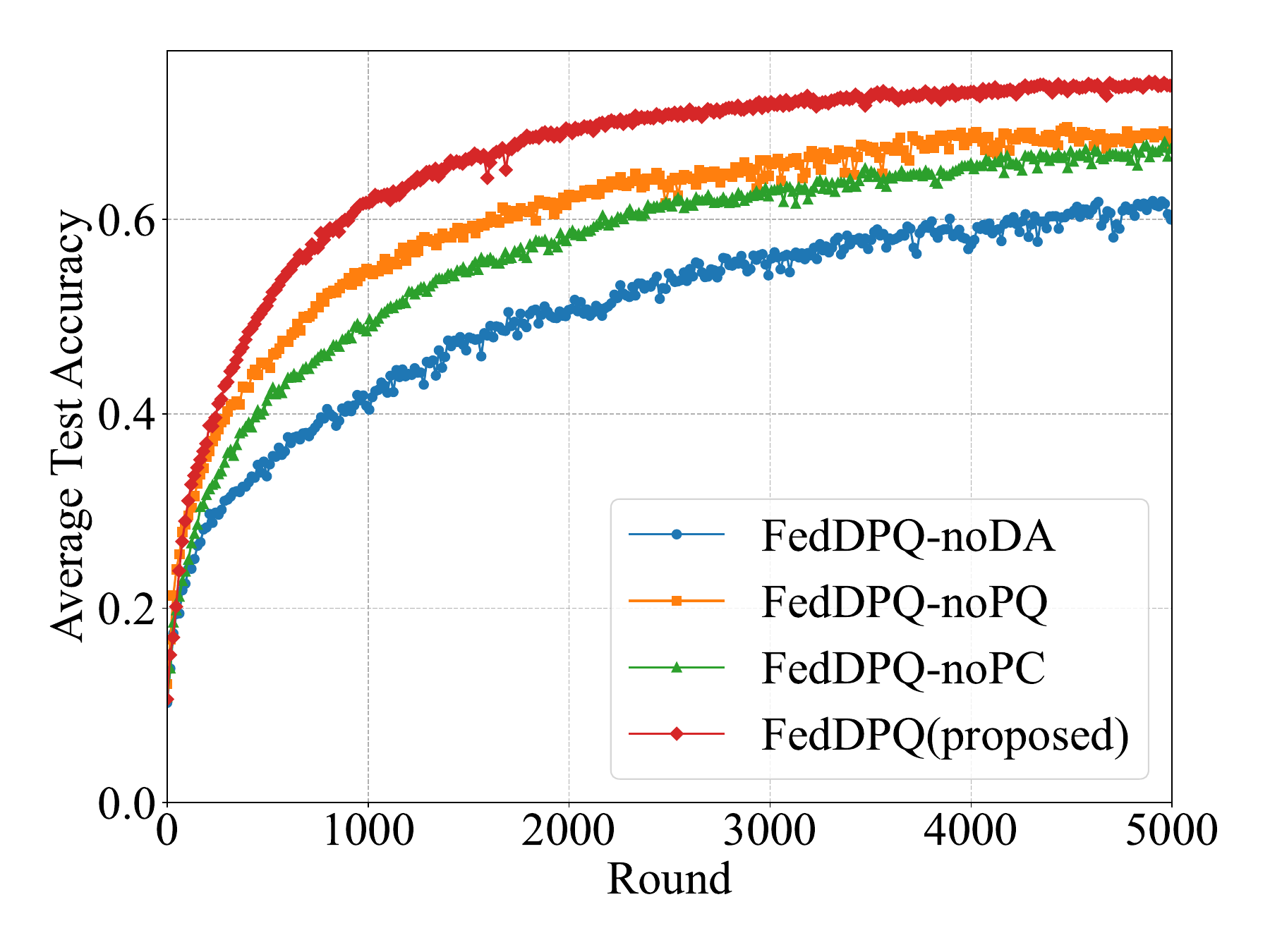}
\label{fig:xiaorong_accuracy}
\end{minipage}}

\subfigure[Training loss.]{
\begin{minipage}{0.7\linewidth}
\centering
\includegraphics[width=\linewidth]{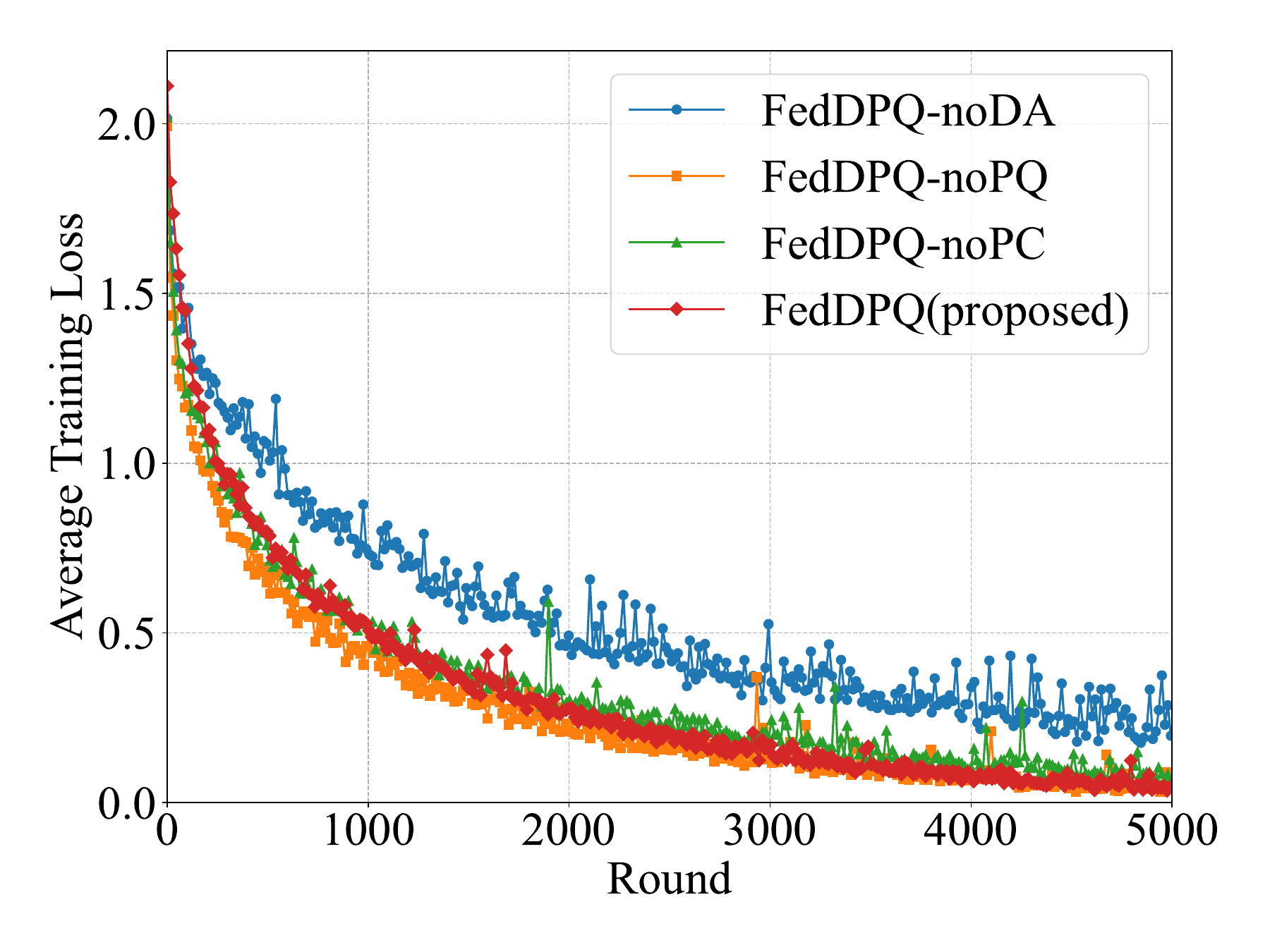}
\label{fig:xiaorong_loss}
\end{minipage}}

\subfigure[Training delay.]{
\begin{minipage}{0.7\linewidth}
\centering
\includegraphics[width=\linewidth]{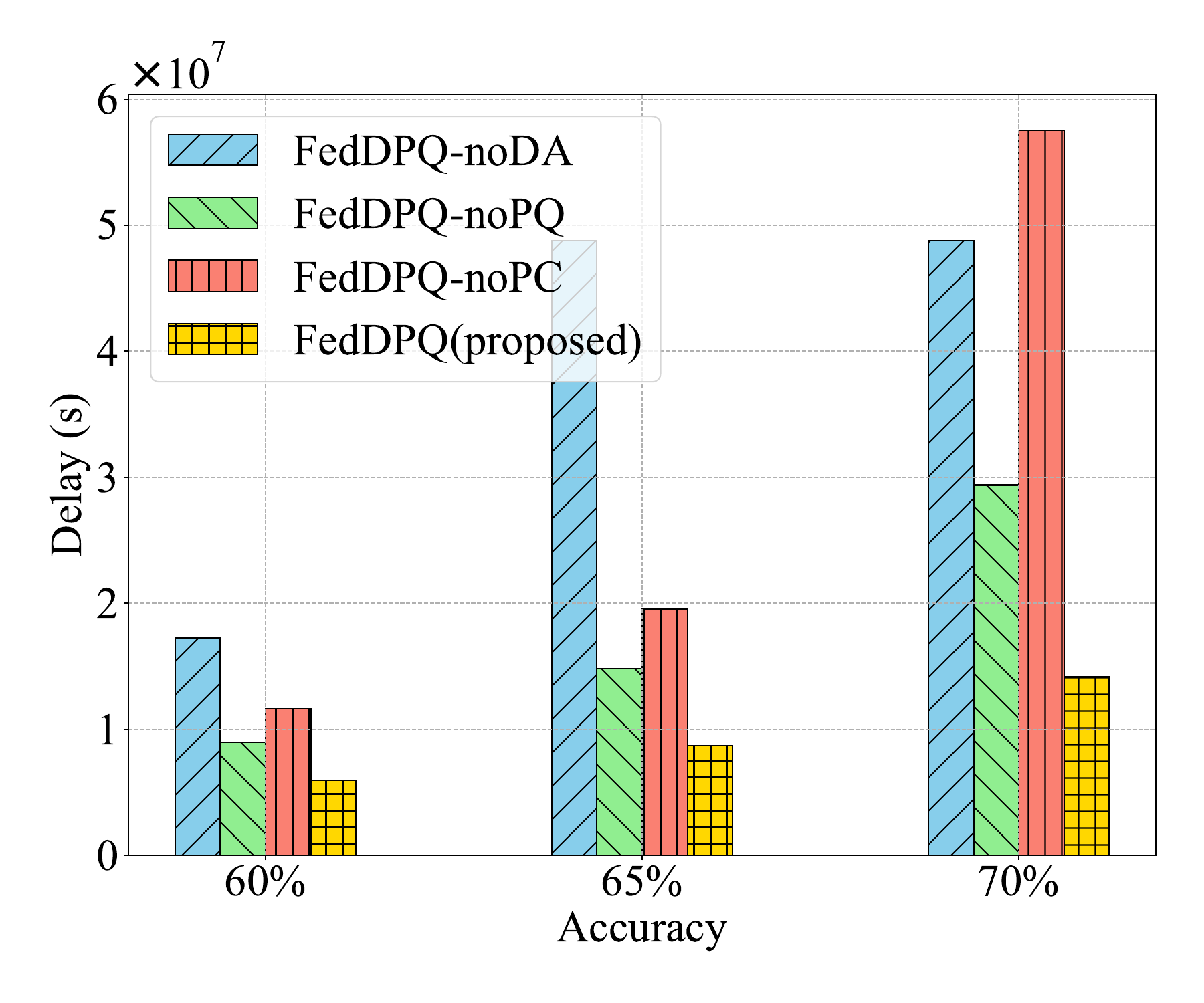}
\label{fig:xiaorong_delay}
\end{minipage}}

\caption{Ablation study results of the proposed scheme.}
\label{fig:ablation_study}
\end{figure}

\textcolor{black}{To further verify the individual contributions of each optimization module within the FedDPQ framework, we conduct an ablation study by comparing the complete FedDPQ scheme with three degraded variants: FedDPQ-noDA, FedDPQ-noPQ, and FedDPQ-noPC. Fig.~\ref{fig:ablation_study} presents a comprehensive comparison of these four schemes in terms of energy consumption, test accuracy, training loss, and training delay. As shown in Fig.~\ref{fig:xiaorong_energy}, FedDPQ consistently achieves the lowest total energy consumption across all accuracy targets, demonstrating that the joint optimization of data augmentation, pruning, quantization, and power control significantly reduces training energy costs. The removal of any single module leads to performance degradation. In particular, FedDPQ-noPC incurs a notable increase in energy consumption, highlighting the importance of power control in mitigating device dropout and enhancing energy efficiency in FL. With regard to training delay (Fig.~\ref{fig:xiaorong_delay}), the ranking of the schemes mirrors that of energy consumption. Power control effectively reduces delay by preventing disconnections, while pruning and quantization decrease the per-round training time. Data augmentation also contributes to faster training by accelerating convergence. In terms of model accuracy (Fig.~\ref{fig:xiaorong_accuracy}), FedDPQ again delivers the best performance. Notably, FedDPQ-noPQ achieves faster convergence than FedDPQ-noDA, indicating that diffusion-based data augmentation plays a vital role in enhancing model performance. By increasing both the volume and diversity of local training data, data augmentation improves the generalization capability of the global model. The accuracy degradation observed in FedDPQ-noPC further confirms that power control enhances training stability and convergence by ensuring more consistent participation from edge devices. In summary, each module in FedDPQ contributes complementary benefits across energy efficiency, delay reduction, and accuracy improvement. The proposed integrated optimization framework demonstrates superior performance across diverse FL scenarios.
}

\section{Conclusion} 
\textcolor{black}{This paper presented FedDPQ, a novel ultra energy-efficient FL framework designed for real-time CV applications over unreliable wireless networks. By jointly integrating diffusion-based data augmentation, model pruning, communication quantization, and adaptive transmission power control, FedDPQ effectively reduces energy consumption from both computation and communication while addressing data heterogeneity and transmission unreliability. We derived a closed-form convergence–energy model that captures the joint impact of these techniques and developed a low-complexity BO-based algorithm to optimize data generation, pruning, quantization, and power control strategies. Experimental evaluations demonstrated that FedDPQ significantly improves energy efficiency and accelerates convergence, offering a practical solution for deploying FL in energy-constrained edge CV scenarios.}

\textcolor{black}{Future research directions include extending FedDPQ to integrate model-splitting techniques (split learning) for scenarios with extreme computational constraints, and exploring fine-tuning strategies with low parameter overhead, such as the low-rank adaptation (LoRA) method, to facilitate federated adaptation of large-scale vision models. These extensions promise further improvements in energy efficiency and adaptability, enhancing the applicability of FL in emerging real-time edge applications.}

\begin{appendices}
\setcounter{equation}{0}
\renewcommand\theequation{A.\arabic{equation}}
\section{Proof of Lemma 3}
We begin by deriving the expectation of the aggregated gradient computed in each round as
\begin{equation}
\begin{aligned}
\label{eqA1}
& \mathbb{E}\left[\frac{\sum\limits_{u \in \mathcal{S}^t}\alpha_u^{t} \boldsymbol{g}_u^t}
{\sum\limits_{u \in \mathcal{S}^t}\alpha_u^{t}}\Bigg|\sum\limits_{u \in \mathcal{S}^t}\alpha_u^{t}\neq 0\right] \\
&= \mathbb{E}_{\mathcal{S}^t}\Bigg[\sum\limits_{k=1}^S\sum_{\substack{
\mathcal{B}^t \cup \bar{\mathcal{B}}^t = \mathcal{S}^t \\
|\mathcal{B}^t| = k,\; |\bar{\mathcal{B}}^t| = S - k
}} Pr\Big( \alpha_{u_1}^t = 1\; \forall u_1 \in \mathcal{B}^t ,\\& \quad \alpha_{u_2}^t=0 \forall u_2 \in \bar{\mathcal{B}}^t \Big|\sum\limits_{u \in S^t}\alpha_u^{t}\neq 0\Big)
\cdot \frac{\sum\limits_{u_1 \in \mathcal{B}^t} \boldsymbol{g}^{t}_{u_1}}{k} \Bigg]\\&
=\sum\limits_{g=1}^{U^S}\left(\prod_{u \in \mathcal{S}_g^t} \tau_u\right) \cdot \Bigg(\sum\limits_{k=1}^S\sum_{\substack{
\mathcal{B}_g^t \cup \bar{\mathcal{B}}_g^t = \mathcal{S}_g^t \\
|\mathcal{B}_g^t| = k,\; |\bar{\mathcal{B}}_g^t| = S - k
}} \frac{\sum\limits_{u_1 \in \mathcal{B}_g^t} \boldsymbol{g}^{t}_{u_1}}{k} \cdot \\& \quad \frac{\prod_{u_1\in \mathcal{B}_g^t}(1-q_{u_1})\prod_{u_2\in \bar{\mathcal{B}}_g^t}q_{u_2}}{1-\prod_{u\in \mathcal{S}_g^t}q_u}\Bigg)
\\
&\overset{\triangle}{=} \sum\limits_{u=1}^U \bar{\beta}_u \boldsymbol{g}_u^t,
\end{aligned}
\end{equation}
where $\mathcal{B}^t$ represents the set of devices which are selected without transmission error, while $\bar{\mathcal{B}}^t$ denotes the set of devices with transmission error. Since in each round, $S$ devices are selected independently and with replacement according to the probability distribution $\{\tau_1,\dots, \tau_U\}$, the set $\mathcal{S}^t$ has $U^S$ different possibilities (denoted by $\mathcal{S}_g^t$, $g=1,\dots,U^S$ ). The probability of each set $\mathcal{S}_g^t$ occurring is $\prod_{u \in \mathcal{S}_g^t} \tau_u$. $\sum\limits_{u=1}^U\bar{\beta}_u=1$, this can be observed by setting $\boldsymbol{g}_u^t=1$.

Next, we prove Lemma 3 as follows
\begin{equation}
\label{eqA2}
\begin{aligned}
&\mathbb{E}\left[ \langle \nabla F(\boldsymbol{w}^{t-1}), \boldsymbol{w}^{t} - \boldsymbol{w}^{t-1} \rangle \right]\\
&=\mathbb{E} \left[ \left\langle \nabla F(\boldsymbol{w}^{t-1}),\ 
    -\eta \cdot \frac{\sum_{u \in \mathcal{S}^t} \alpha_u^t \mathcal{Q}_u(\boldsymbol{g}_u^t)}{\sum_{u \in \mathcal{S}^t} \alpha_u^t} \right\rangle \right]\\
&\overset{\left(a\right)}{=} -\eta \cdot \mathbb{E} \left[ \left\langle \nabla F(\boldsymbol{w}^{t-1}),\ 
    \frac{\sum_{u \in \mathcal{S}^t} \alpha_u^t \boldsymbol{g}_u^t}{\sum_{u \in \mathcal{S}^t} \alpha_u^t} \right\rangle \right]\\
&\overset{\left(b\right)}{=} -\eta \cdot \mathbb{E} \left[ \left\langle \nabla F(\boldsymbol{w}^{t-1}),\ 
    \frac{\sum_{u \in \mathcal{S}^t} \alpha_u^t \nabla F_u(\widetilde{\boldsymbol{w}}^{t}_{u})}{\sum_{u \in \mathcal{S}^t} \alpha_u^t} \right\rangle \right]\\
&\overset{\left(c\right)}{=} -\eta \cdot \mathbb{E} \left[ \left\langle \nabla F(\boldsymbol{w}^{t-1}),\ 
    \sum_{u=1}^{U} \bar{\beta}_u \nabla F_u(\widetilde{\boldsymbol{w}}^{t}_{u}) \right\rangle \right]
\end{aligned}
\end{equation}
\begin{equation*}
\begin{aligned}
&\overset{\left(d\right)}{=} \frac{\eta}{2} \cdot \mathbb{E} \left[ \left\| \nabla F(\boldsymbol{w}^{t-1})  
    - \sum_{u=1}^{U} \bar{\beta}_u  \nabla F_u(\widetilde{\boldsymbol{w}}^{t}_{u}) \right\|^2 \right.\\&\left.\quad\quad\quad\quad -\left\| \nabla F(\boldsymbol{w}^{t-1}) \right\|^2 - \left\|\sum_{u=1}^{U} \bar{\beta}_u  \nabla F_u(\widetilde{\boldsymbol{w}}^{t}_{u}) \right\|^2  \right]\\
&\leq \frac{\eta}{2} \cdot \mathbb{E} \left[ \left\| \nabla F(\boldsymbol{w}^{t-1})  
    - \sum_{u=1}^{U} \bar{\beta}_u  \nabla F_u(\widetilde{\boldsymbol{w}}^{t}_{u}) \right\|^2 \right]\\&\quad\quad\quad\quad -\frac{\eta}{2} \cdot\mathbb{E} \left[\left\| \nabla F(\boldsymbol{w}^{t-1}) \right\|^2\right]\\
&\overset{\left(e\right)}{\leq} -\frac{\eta}{2} \cdot \mathbb{E} \left[ \left\| \nabla F(\boldsymbol{w}^{t-1}) \right\|^2 \right] 
    \\&\quad\quad+ \eta \cdot \underbrace{\mathbb{E} \left[ \left\| \nabla F(\boldsymbol{w}^{t-1})-\sum_{u=1}^{U} \bar{\beta}_u \nabla F_u(\boldsymbol{w}^{t-1}) \right\|^2 \right]}_{A_1}\\
    &\quad\quad\quad+ \eta \cdot \underbrace{\mathbb{E} \left[ \left\| \sum_{u=1}^{U} \bar{\beta}_u 
    \left( \nabla F_u(\boldsymbol{w}^{t-1}) - \nabla F_u(\widetilde{\boldsymbol{w}}^{t}_{u}) \right) \right\|^2 \right]}_{A_2}, 
\end{aligned}
\end{equation*}
where equality $\left(a\right)$ stems from the unbiased quantization in Eq. \eqref{quantize no bias}, equality $\left(b\right)$ follows from Assumption \ref{CompleteAssumption2}, equality $\left(c\right)$ is due to Eq. \eqref{eqA1}, equality $\left(d\right)$ is obtained by the basic identity $\langle x_1,x_2\rangle = \frac{1}{2}\left(\|x_1\|^2+\|x_2\|^2-\|x_1-x_2\|^2\right)$, and inequality $\left(e\right)$ follows from  $\|x_1+x_2\|^2 \leq 2\|x_1\|^2+2\|x_2\|^2$.

\begin{equation}
\label{eqA3}
\begin{aligned}
&A_1 = \mathbb{E} \left[ \left\| \sum_{u=1}^{U} \tau_u \nabla F_u(\boldsymbol{w}^{t-1}) - \sum_{u=1}^{U} \bar{\beta}_u \nabla F_u(\boldsymbol{w}^{t-1}) \right\|^2 \right] \\
&\overset{\left(a\right)}{=} \mathbb{E} \left[ \left\| \sum_{u=1}^{U} (\tau_u - \bar{\beta}_u) \nabla F_u(\boldsymbol{w}^{t-1}) \right.\right.  \\
&\left.\left.\quad\quad\quad\quad -\sum_{u=1}^{U} (\tau_u - \bar{\beta}_u) \nabla F(\boldsymbol{w}^{t-1}) \right\|^2 \right]\\&= \mathbb{E} \left[ \left\| \sum_{u=1}^{U} \frac{\tau_u - \bar{\beta}_u}{\sqrt{\tau_u}} \cdot \sqrt{\tau_u} \left( \nabla F_u(\boldsymbol{w}^{t-1}) - \nabla F(\boldsymbol{w}^{t-1}) \right) \right\|^2 \right] \\
&\overset{\left(b\right)}{\leq} \left( \sum_{u=1}^{U} \frac{\left(\tau_u - \bar{\beta}_u\right)^2}{\tau_u} \right) \cdot \sum_{u=1}^{U} {\tau_u} \, \mathbb{E} \left[ \left\| \nabla F_u(\boldsymbol{w}^{t-1}) \right.\right.  \\
&\left.\left.\quad\quad- \nabla F(\boldsymbol{w}^{t-1}) \right\|^2 \right]
\overset{\left(c\right)}{\leq} \chi^2_{\boldsymbol{\beta}\| \boldsymbol{\tau}} \cdot \sum_{u=1}^{U} \tau_u Z_u^2,
\end{aligned}
\end{equation}
where equality $\left(a\right)$ is obtained by $\sum_{u=1}^U\left(\tau_u-\bar{\beta}_u\right)=0$, inequality $\left(b\right)$ stems from Cauchy-Schwarz Inequality, and inequality $\left(c\right)$ is because of Assumption 3.

\begin{equation}
\label{eqA4}
\begin{aligned}
A_2 &\overset{\left(a\right)}{\leq} \sum_{u=1}^{U} \bar{\beta}_u^2\sum_{u=1}^{U}\mathbb{E} \left[ \left\| \nabla F_u(\boldsymbol{w}^{t-1}) - \nabla F_u(\widetilde{\boldsymbol{w}}^{t}_{u}) \right\|^2 \right] \\
&\overset{\left(b\right)}{\leq} \sum_{u=1}^{U} \bar{\beta}_u^2 L^2\sum_{u=1}^{U}\mathbb{E} \left[ \left\| \boldsymbol{w}^{t-1} - \widetilde{\boldsymbol{w}}^{t}_{u} \right\|^2\right]
\end{aligned}
\end{equation}
\begin{equation*}
\begin{aligned}
&\overset{\left(c\right)}{\leq} \sum_{u=1}^{U} \bar{\beta}_u^2 \cdot L^2 \Gamma^2 \sum_{u=1}^{U} \rho_u,
\end{aligned}
\end{equation*}
where inequality $\left(a\right)$ is due to $\|\sum\limits_{u=1}^U x_uy_u\|^2 \leq \sum\limits_{u=1}^U\|x_u\|^2 \sum\limits_{u=1}^U\|y_u\|^2$, inequality $\left(b\right)$ stems from Assumption \ref{CompleteAssumption1}, and inequality $\left(c\right)$ follows from Lemma \ref{lemma1}. Finally, by substituting Eq. \eqref{eqA3} and \eqref{eqA4} into Eq. \eqref{eqA2}, we can obtain Lemma \ref{lemma3}.

\setcounter{equation}{0}
\renewcommand\theequation{B.\arabic{equation}}
\section{Proof of Lemma 4}
In the same fashion as Eq. \eqref{eqA1}, we can obtain
\begin{equation}
\begin{aligned}
\label{eqB.1}
& \mathbb{E}\left[\frac{\sum\limits_{u \in \mathcal{S}^t}\alpha_u^{t} \boldsymbol{g}_u^t}
{\left(\sum\limits_{u \in \mathcal{S}^t}\alpha_u^{t}\right)^2}\Bigg|\sum\limits_{u \in \mathcal{S}^t}\alpha_u^{t}\neq 0\right] \\&
=\sum\limits_{g=1}^{U^S}\left(\prod_{u \in \mathcal{S}_g^t} \tau_u\right) \cdot \Bigg(\sum\limits_{k=1}^S\sum_{\substack{
\mathcal{B}_g^t \cup \bar{\mathcal{B}}_g^t = \mathcal{S}_g^t \\
|\mathcal{B}_g^t| = k,\; |\bar{\mathcal{B}}_g^t| = S - k
}} \frac{\sum\limits_{u_1 \in \mathcal{B}_g^t} \boldsymbol{g}^{t}_{u_1}}{k^2} \cdot \\& \quad \frac{\prod_{u_1\in \mathcal{B}_g^t}(1-q_{u_1})\prod_{u_2\in \bar{\mathcal{B}}_g^t}q_{u_2}}{1-\prod_{u\in \mathcal{S}_g^t}q_u}\Bigg)
\\
&\overset{\triangle}{=} \sum\limits_{u=1}^U \bar{\alpha}_u \boldsymbol{g}_u^t. 
\end{aligned}
\end{equation}
Next, we prove Lemma 4 as follows
\begin{equation}
\begin{aligned}
\label{eqB}
&\mathbb{E}\left[\|\boldsymbol{w}^{t} - \boldsymbol{w}^{t-1}\|^2\right]\\ &= \eta^2 \, \mathbb{E} \left[ \left\| \frac{\sum_{u \in \mathcal{S}^t} \alpha_u^t \mathcal{Q}\left(\boldsymbol{g}_u^t\right)}{\sum_{u \in \mathcal{S}^t} \alpha_u^t} \right\|^2 \right] \\
&\overset{\left(a\right)}{\leq} 2 \eta^2 \, \underbrace{\mathbb{E} \left[ \left\| \frac{\sum_{u \in \mathcal{S}^t} \alpha_u^t \left(\mathcal{Q}(\boldsymbol{g}_u^t) - \boldsymbol{g}_u^t\right)}{\sum_{u \in \mathcal{S}^t} \alpha_u^t} \right\|^2 \right]}_{B_1} \\
&\quad + 2 \eta^2 \, \underbrace{\mathbb{E} \left[ \left\| \frac{\sum_{u \in \mathcal{S}^t} \alpha_u^t \boldsymbol{g}_u^t}{\sum_{u \in \mathcal{S}^t} \alpha_u^t} \right\|^2 \right]}_{B_2},
\end{aligned}
\end{equation}
where inequality $\left(a\right)$ is due to the basic identity $\|x_1+x_2\|^2 \leq 2\|x_1\|^2+2\|x_2\|^2$. Next, we separately derive the upper bounds of $B_1$ and $B_2$.

\subsubsection{Bound of $B_1$}
\begin{equation}
\label{eqB1}
\begin{aligned}
B_1 &\overset{\left(a\right)}{=} \mathbb{E} \left[ 
\frac{  \sum_{u \in \mathcal{S}^t} \alpha_u^t \left\|\left( \mathcal{Q}(\boldsymbol{g}_u^t) - \boldsymbol{g}_u^t \right) \right\|^2 }
{ \left( \sum_{u \in \mathcal{S}^t} \alpha_u^t \right)^2 }
\right] \\
&\overset{\left(b\right)}{=} \sum_{u=1}^{U} \bar{\alpha}_u \, \mathbb{E} \left[ \left\| \mathcal{Q}(\boldsymbol{g}_u^t) - \boldsymbol{g}_u^t \right\|^2 \right] \\
& \overset{\left(c\right)}{\leq} \sum_{u=1}^{U} \bar{\alpha}_u\frac{
\displaystyle \sum_{v=1}^{V} \left( \bar{g}_{u,v}^t - \underline{g}_{u,v}^t \right)^2
}{
4\left(2^{\delta_u} - 1\right)^2
},
\end{aligned}
\end{equation}
where equality $\left(a\right)$ is obtained by $\|x_1+x_2+\dots+x_U\|^2=\sum\limits_{u=1}^U\|x_u\|^2+\sum\limits_{i=1}^U\sum\limits_{j=1,j\neq i}^Ux_ix_{j}$, equality $\left(b\right)$ is due to Eq. \eqref{eqB.1} and inequality $\left(c\right)$ follows from Lemma \ref{lemma2}.

\subsubsection{Bound of $B_2$}
\begin{equation}
\begin{aligned}
B_2&\overset{\left(a\right)}{\leq} 2 \, \underbrace{\mathbb{E} \left[
\left\| \frac{ \sum_{u \in \mathcal{S}^t} \alpha_u^t \left( \boldsymbol{g}_u^t - \nabla F_u(\widetilde{\boldsymbol{w}}^{t}_{u}) \right) }{ \sum_{u \in \mathcal{S}^t} \alpha_u^t } \right\|^2 \right]}_{B_{21}} \\
&\quad + 2 \, \underbrace{\mathbb{E} \left[
\left\| \frac{ \sum_{u \in \mathcal{S}^t} \alpha_u^t \nabla F_u(\widetilde{\boldsymbol{w}}^{t}_{u}) }{ \sum_{u \in \mathcal{S}^t} \alpha_u^t } \right\|^2 \right]}_{B_{22}},
\end{aligned}
\end{equation}
where inequality $\left(a\right)$ stems from $\|x_1+x_2\|^2 \leq 2\|x_1\|^2+2\|x_2\|^2$. Then, we derive the upper bound of $B_{21}$ and $B_{22}$. To begin with, the upper bound of $B_{21}$ can be derived by

\begin{equation}
\label{eq:B21}
\begin{aligned}
B_{21}
&\overset{\left(a\right)}{=} \mathbb{E}\!\Biggl[
      \frac{%
        \displaystyle
        \sum\nolimits_{u\in\mathcal{S}^{t}}
        \alpha_{u}^{t}\;
        \Bigl\lVert
            \boldsymbol{g}_{u}^t
          - \nabla F_{u}\!\bigl(\widetilde{\boldsymbol{w}}^{t}_{u}\bigr)
        \Bigr\rVert^{2}}
      {\Bigl(\,\sum\nolimits_{u\in\mathcal{S}^{t}} \alpha_{u}^{t}\Bigr)^{2}}
    \Biggr] \\[4pt]
&= \sum_{u=1}^{U} \bar{\alpha}_{u}\;
   \mathbb{E}\!\Bigl[
      \bigl\lVert
          \boldsymbol{g}_{u}^t
        - \nabla F_{u}\!\bigl(\widetilde{\boldsymbol{w}}^{t}_{u}\bigr)
      \bigr\rVert^{2}
   \Bigr] \\[2pt]
&\overset{\left(b\right)}{\le} \sum_{u=1}^{U} \bar{\alpha}_{u}\,\sigma^{2},
\end{aligned}
\end{equation}
where equality $\left(a\right)$ is due to the basic identity $\|x_1+x_2+\dots+x_U\|^2=\sum\limits_{u=1}^U\|x_u\|^2+\sum\limits_{i=1}^U\sum\limits_{j=1,j\neq i}^Ux_ix_{j}$, and inequality $\left(b\right)$ follows from Assumption \ref{CompleteAssumption2}. 

Secondly, we derive the upper bound of $B_{22}$ as follows:
\begin{equation}
\label{eq:B22}
\begin{aligned}
B_{22}
\;\le\;& 2\,
\underbrace{\mathbb{E}\!\Biggl[
  \Biggl\lVert
    \frac{%
      \displaystyle\sum\nolimits_{u\in\mathcal{S}^{t}}
      \alpha_{u}^{t}\!
      \Bigl(
        \nabla F_{u}\!\bigl(\widetilde{\boldsymbol{w}}^{t}_{u}\bigr)
        - \nabla F_u\!\bigl(\boldsymbol{w}^{t-1}\bigr)
      \Bigr)}
      {\displaystyle\sum\nolimits_{u\in\mathcal{S}^{t}} \alpha_{u}^{t}}\!
  \Biggr\rVert^{2}
\Biggr]}_{B_{221}} \\[4pt]
&\quad
+ 2\,
\underbrace{\mathbb{E}\!\Biggl[
  \Biggl\lVert
    \frac{%
      \displaystyle\sum\nolimits_{u\in\mathcal{S}^{t}}
      \alpha_{u}^{t}\,\nabla F_{u}\!\bigl(\boldsymbol{w}^{t-1}\bigr)}
      {\displaystyle\sum\nolimits_{u\in\mathcal{S}^{t}} \alpha_{u}^{t}}\!
  \Biggr\rVert^{2}
\Biggr]}_{B_{222}}.
\end{aligned}
\end{equation}

The upper bound of $B_{221}$ can be derived by
\begin{equation}
\begin{aligned}
&B_{221} = \mathbb{E} \left[
\frac{
   \left\|\sum\nolimits_{u\in \mathcal{S}^t} \left(\alpha_u^t \right)^2 
      \left( 
        \nabla F_u(\widetilde{\boldsymbol{w}}^{t}_{u}) - \nabla F_u(\boldsymbol{w}^{t-1}) 
      \right)
  \right\|^2
}{
  \left( \sum\nolimits_{u\in \mathcal{S}^t} \alpha_u^t \right)^2
}
\right] \\[6pt]
&\leq \mathbb{E} \left[
\frac{
  \sum\limits_{u\in \mathcal{S}^t} \left(\alpha_u^t \right)^2 \cdot \sum\limits_{u\in \mathcal{S}^t} \left(\alpha_u^t \right)^2 
  \left\| 
    \nabla F_u(\widetilde{\boldsymbol{w}}^{t}_{u}) - \nabla F_u(\boldsymbol{w}^{t-1}) 
  \right\|^2
}{
  \left(\sum\nolimits_{u\in \mathcal{S}^t} \alpha_u^t \right)^2
}
\right] \\[6pt]
&= \mathbb{E} \left[
\frac{
  \sum_{u\in \mathcal{S}^t} \alpha_u^t 
  \left\| 
    \nabla F_u(\widetilde{\boldsymbol{w}}^{t}_{u}) - \nabla F_u(\boldsymbol{w}^{t-1}) 
  \right\|^2
}{
  \sum\nolimits_{u\in \mathcal{S}^t} \alpha_u^t
}
\right] \\[6pt]
&\leq \sum_{u=1}^{U} \bar{\beta}_u \,
\mathbb{E} \left[
\left\| 
  \nabla F_u(\widetilde{\boldsymbol{w}}^{t}_{u}) - \nabla F_u(\boldsymbol{w}^{t-1}) 
\right\|^2
\right]\\[6pt]
&\leq L^2 \sum_{u=1}^{U} \bar{\beta}_u \,
\mathbb{E} \left[
\left\| 
  \widetilde{\boldsymbol{w}}^{t}_{u} - \boldsymbol{w}^{t-1} 
\right\|^2
\right]\\
&\leq  L^2 \Gamma^2 \sum_{u=1}^{U} \bar{\beta}_u \rho_u.
\end{aligned}
\end{equation}

The upper bound of $B_{221}$ can be derived by
\begin{equation}
\begin{aligned}
B_{222} \leq\; &2\; \underbrace{\mathbb{E} \left[
\frac{
\sum\nolimits_{u \in \mathcal{S}^t} \alpha_u^t \|
 \nabla F_u(\boldsymbol{w}^{t-1}) - \nabla F(\boldsymbol{w}^{t-1}) 
\|^2
}{
\left( \sum\nolimits_{u \in \mathcal{S}^t} \alpha_u^t \right)^2
}
\right]}_{C_1} \\[6pt]
&+\, 2\underbrace{\mathbb{E} \left[
\frac{
\begin{array}[t]{@{}l@{}}
\sum\nolimits_{u \in \mathcal{S}^t} \sum\nolimits_{\substack{u' \in \mathcal{S}^t \\ u \ne u'}} 
\alpha_{u}^t\alpha_{u'}^t \left(
\nabla F_{u}(\boldsymbol{w}^{t-1}) 
- \right.\\[4pt]
\left.\quad
\nabla F(\boldsymbol{w}^{t-1})\right)
\left(
\nabla F_{u'}(\boldsymbol{w}^{t-1}) - \nabla F(\boldsymbol{w}^{t-1})
\right)
\end{array}
}{
\left( \sum\nolimits_{u \in \mathcal{S}^t} \alpha_u^t \right)^2
}
\right]}_{C_2} \\[6pt]
&+ 2\, \mathbb{E} \left[
\left\| 
  \nabla F(\boldsymbol{w}^{t-1}) 
\right\|^2
\right].
\end{aligned}
\end{equation}

Furthermore, the upper bound of $C_1$ and $C_2$ can be derived by
\begin{equation}
\begin{aligned}
C_1 &= \sum_{u=1}^{U} \bar{\alpha}_u \, 
\mathbb{E} \left[
\left\|
\nabla F_u(\boldsymbol{w}^{t-1}) - \nabla F(\boldsymbol{w}^{t-1})
\right\|^2
\right]
\;\\&\le\;
\sum_{u=1}^{U} \bar{\alpha}_u Z_u^2
\end{aligned}
\end{equation}
and
\begin{equation}
C_2 \overset{\left(a\right)}{\leq} 
\sum_{k=2}^{S}
\frac{
(q_{\max})^{S - k} \, \mathbb{C}_S^k
}{
1 - (q_{\max})^S
}
\sum_{u=1}^{U} \tau_u \left\| q_u - \bar{q} \right\|^2 Z_u^2,
\end{equation}
where inequality $\left(a\right)$ is referred to \cite{wang2021quantized}. Thus, the upper bound of $B_2$ can be represented by
\begin{equation}
\label{eqB2}
\begin{aligned}
 B_2 &\leq 8\sum_{u=1}^{U} \bar{\alpha}_u Z_u^2+
8\, \mathbb{E} \left[
\left\| 
  \nabla F(\boldsymbol{w}^{t-1}) 
\right\|^2
\right] +4L^2 \Gamma^2 \sum_{u=1}^{U} \bar{\beta}_u \rho_u\\&+8\sum_{k=2}^{S}
\frac{
(q_{\max})^{S - k} \, \mathbb{C}_S^k
}{
1 - (q_{\max})^S
}
\sum_{u=1}^{U} \tau_u \left\| q_u - \bar{q} \right\|^2 Z_u^2+2\sum_{u=1}^{U} \bar{\alpha}_u \sigma^2.
\end{aligned}
\end{equation}

Finally, by substituting Eq. \eqref{eqB1} and \eqref{eqB2} into Eq. \eqref{eqB}, we can obtain Lemma \ref{lemma4}.

\setcounter{equation}{0}
\renewcommand\theequation{C.\arabic{equation}}
\section{Proof of Theorem 1}
Given the smoothness property of the loss function $F(\cdot)$, the second-order Taylor expansion of $F(\cdot)$ at any training round $t \geq 0$ can be formulated as follows
\begin{equation}
\begin{aligned}
&\mathbb{E}[F(\boldsymbol{w}^t)] \leq \mathbb{E}[F(\boldsymbol{w}^{t-1})] + \mathbb{E}\left[ \langle \nabla F(\boldsymbol{w}^{t-1}), \boldsymbol{w}^{t} - \boldsymbol{w}^{t-1} \rangle \right] \\& \quad \quad+ \frac{L}{2} \mathbb{E}\left[\|\boldsymbol{w}^{t} - \boldsymbol{w}^{t-1}\|^2\right].
\end{aligned}
\end{equation}

And based on Lemmas \ref{lemma1} and \ref{lemma3}, we can obtain
\begin{equation}
\label{eqC2}
\begin{aligned}
&\mathbb{E}[F(\boldsymbol{w}^t)]\leq\mathbb{E}[F(\boldsymbol{w}^{t-1})]-\frac{\eta}{2} \cdot \mathbb{E} \left[ \left\| \nabla F(\boldsymbol{w}^{t-1}) \right\|^2 \right]\\&\quad+ \eta\sum_{u=1}^{U} \bar{\beta}_u^2 \cdot L^2 \Gamma^2 \sum_{u=1}^{U} \rho_u+L\eta^2\sum_{u=1}^{U} \bar{\alpha}_u\frac{
\displaystyle \sum_{v=1}^{V} \left( \bar{g}_{u,v}^t - \underline{g}_{u,v}^t \right)^2
}{
4\left(2^{\delta_u} - 1\right)^2
}\\&\quad+4\eta^2L^3 \Gamma^2 \sum_{u=1}^{U} \bar{\beta}_u \rho_u+8L\eta^2\, \mathbb{E} \left[
\left\| 
  \nabla F(\boldsymbol{w}^{t-1}) 
\right\|^2
\right]\\&\quad+2L\eta^2\sum_{u=1}^{U}\bar{\alpha}_u\sigma^2 + 8L\eta^2\sum_{u=1}^{U} \bar{\alpha}_u Z_u^2+\eta\cdot\chi^2_{\boldsymbol{\beta}\| \boldsymbol{\tau}} \cdot \sum_{u=1}^{U} \tau_u Z_u^2\\&\quad+8L\eta^2\sum_{k=2}^{S}
\frac{
(q_{\max})^{S - k} \, \mathbb{C}_S^k
}{
1 - (q_{\max})^S
}
\sum_{u=1}^{U} \tau_u \left\| q_u - \bar{q} \right\|^2 Z_u^2.
\end{aligned}
\end{equation}

Then, we rearrange Eq. \eqref{eqC2} and divide its both sides by $\left(\frac{\eta}{2}-8L\eta^2\right)$:
\begin{equation}
\label{eqC3}
\begin{aligned}
&\mathbb{E} \left[ \left\| \nabla F(\boldsymbol{w}^{t-1}) \right\|^2 \right]
\leq
\frac{ \mathbb{E}[F(\boldsymbol{w}^{t-1})]-\mathbb{E}[F(\boldsymbol{w}^{t})] }{ \left(\frac{\eta}{2}-8L\eta^2\right) }\\
&\quad+ \frac{ \eta \cdot \chi^2_{\boldsymbol{\beta}\| \boldsymbol{\tau}} }{ \frac{\eta}{2}-8L\eta^2 } \cdot \sum_{u=1}^{U} \tau_u Z_u^2+ \frac{ 8L\eta^2 }{ \left(\frac{\eta}{2}-8L\eta^2\right) } \sum_{u=1}^{U} \bar{\alpha}_uZ_u^2\\
&\quad+ \frac{ \eta L^2 \Gamma^2 }{ \left(\frac{\eta}{2}-8L\eta^2\right) } \left(\sum_{u=1}^{U} \bar{\beta}_u^2 \cdot \sum_{u=1}^{U} \rho_u+ 4\eta L\sum_{u=1}^{U} \bar{\beta}_u \rho_u\right)\\
&\quad+ \frac{L\eta^2}{\left(\frac{\eta}{2}-8L\eta^2\right)}\sum_{u=1}^{U} \bar{\alpha}_u\frac{
\displaystyle \sum_{v=1}^{V} \left( \bar{g}_{u,v}^t - \underline{g}_{u,v}^t \right)^2
}{
4\left(2^{\delta_u} - 1\right)^2
} \\
&\quad+ \frac{8L\eta^2}{ \frac{\eta}{2}-8L\eta^2 } 
\sum_{k=2}^{S}
\frac{
(q_{\max})^{S - k} \, \mathbb{C}_S^k
}{
1 - (q_{\max})^S
}
\sum_{u=1}^{U} \tau_u \left\| q_u - \bar{q} \right\|^2 Z_u^2\\&\quad + \frac{2 L\eta^2}{ \left(\frac{\eta}{2}-8L\eta^2\right) } \sum_{u=1}^{U} \bar{\alpha}_u\sigma^2.
\end{aligned}
\end{equation}

To guarantee the convergence of the FL process, we let $\frac{\eta}{2}-8L\eta^2 > 0$, which leads to $0<\eta<\frac{1}{16L}$. Moreover, summing above items from $t = 1$ to $\Omega$. Then, we divide both sides of Eq. \eqref{eqC3} by $\Omega$ and can obtain:
\begin{equation}
\label{eqC4}
\begin{aligned}
&\frac{1}{\Omega} \sum_{t = 1}^\Omega\mathbb{E} \left[ \left\| \nabla F(\boldsymbol{w}^{t-1}) \right\|^2 \right]
\leq
\frac{ \mathbb{E}[F(\boldsymbol{w}^{0})]-\mathbb{E}[F(\boldsymbol{w}^{*})] }{ \left(\frac{\eta}{2}-8L\eta^2\right) \Omega}\\
&\quad+ \frac{ \eta \cdot \chi^2_{\boldsymbol{\beta}\| \boldsymbol{\tau}} }{ \frac{\eta}{2}-8L\eta^2 } \cdot \sum_{u=1}^{U} \tau_u Z_u^2+ \frac{ 8L\eta^2 }{ \left(\frac{\eta}{2}-8L\eta^2\right) } \sum_{u=1}^{U} \bar{\alpha}_uZ_u^2\\
&\quad+ \frac{ \eta L^2 \Gamma^2 }{ \left(\frac{\eta}{2}-8L\eta^2\right) } \left(\sum_{u=1}^{U} \bar{\beta}_u^2 \cdot \sum_{u=1}^{U} \rho_u+ 4\eta L\sum_{u=1}^{U} \bar{\beta}_u \rho_u\right)\\
&\quad+ \frac{L\eta^2}{\left(\frac{\eta}{2}-8L\eta^2\right)\Omega} \sum_{t = 1}^\Omega\sum_{u=1}^{U} \bar{\alpha}_u\frac{
\displaystyle \sum_{v=1}^{V} \left( \bar{g}_{u,v}^t - \underline{g}_{u,v}^t \right)^2
}{
4\left(2^{\delta_u} - 1\right)^2
} \\
&\quad+ \frac{8L\eta^2}{ \frac{\eta}{2}-8L\eta^2 } 
\sum_{k=2}^{S}
\frac{
(q_{\max})^{S - k} \, \mathbb{C}_S^k
}{
1 - (q_{\max})^S
}
\sum_{u=1}^{U} \tau_u \left\| q_u - \bar{q} \right\|^2 Z_u^2\\&\quad + \frac{2 L\eta^2}{ \left(\frac{\eta}{2}-8L\eta^2\right) } \sum_{u=1}^{U} \bar{\alpha}_u\sigma^2,
\end{aligned}
\end{equation}
where $\boldsymbol{w}^{*}$ represents the optimal model.

\setcounter{equation}{0}
\renewcommand\theequation{B.\arabic{equation}}

\end{appendices}
 \bibliography{UEFLPQCS}

\begin{thebibliography}{10}
\providecommand{\url}[1]{#1}
\csname url@samestyle\endcsname
\providecommand{\newblock}{\relax}
\providecommand{\bibinfo}[2]{#2}
\providecommand{\BIBentrySTDinterwordspacing}{\spaceskip=0pt\relax}
\providecommand{\BIBentryALTinterwordstretchfactor}{4}
\providecommand{\BIBentryALTinterwordspacing}{\spaceskip=\fontdimen2\font plus
\BIBentryALTinterwordstretchfactor\fontdimen3\font minus \fontdimen4\font\relax}
\providecommand{\BIBforeignlanguage}[2]{{%
\expandafter\ifx\csname l@#1\endcsname\relax
\typeout{** WARNING: IEEEtran.bst: No hyphenation pattern has been}%
\typeout{** loaded for the language `#1'. Using the pattern for}%
\typeout{** the default language instead.}%
\else
\language=\csname l@#1\endcsname
\fi
#2}}
\providecommand{\BIBdecl}{\relax}
\BIBdecl

\bibitem{9953092}
Z.~Meng, C.~She, G.~Zhao, and D.~De~Martini, ``Sampling, communication, and prediction co-design for synchronizing the real-world device and digital model in metaverse,'' \emph{IEEE Journal on Selected Areas in Communications}, vol.~41, no.~1, pp. 288--300, 2023.

\bibitem{mcmahan2017communication}
B.~McMahan, E.~Moore, D.~Ramage, S.~Hampson, and B.~A. y~Arcas, ``Communication-efficient learning of deep networks from decentralized data,'' in \emph{International Conference on Artificial Intelligence and Statistics (AISTATS)}, Ft. Lauderdale, FL, USA, April, 2017, pp. 1273--1282.

\bibitem{9762360}
Y.~Jiang, S.~Wang, V.~Valls, B.~J. Ko, W.-H. Lee, K.~K. Leung, and L.~Tassiulas, ``Model pruning enables efficient federated learning on edge devices,'' \emph{IEEE Transactions on Neural Networks and Learning Systems}, vol.~34, no.~12, pp. 10\,374--10\,386, 2023.

\bibitem{10368103}
X.~Hou, J.~Wang, C.~Jiang, Z.~Meng, J.~Chen, and Y.~Ren, ``Efficient federated learning for metaverse via dynamic user selection, gradient quantization and resource allocation,'' \emph{IEEE Journal on Selected Areas in Communications}, vol.~42, no.~4, pp. 850--866, 2024.

\bibitem{10026255}
X.~Lin, Y.~Liu, F.~Chen, X.~Ge, and Y.~Huang, ``Joint gradient sparsification and device scheduling for federated learning,'' \emph{IEEE Transactions on Green Communications and Networking}, vol.~7, no.~3, pp. 1407--1419, 2023.

\bibitem{10599123}
J.~Wang, H.~Du, D.~Niyato, J.~Kang, S.~Cui, X.~Shen, and P.~Zhang, ``Generative ai for integrated sensing and communication: Insights from the physical layer perspective,'' \emph{IEEE Wireless Communications}, vol.~31, no.~5, pp. 246--255, 2024.

\bibitem{wang2025generative}
J.~Wang, C.~Zhao, H.~Du, G.~Sun, J.~Kang, S.~Mao, D.~Niyato, and D.~I. Kim, ``Generative ai enabled robust data augmentation for wireless sensing in isac networks,'' \emph{arXiv preprint arXiv:2502.12622}, 2025.

\bibitem{10.1145/3422622}
I.~Goodfellow, J.~Pouget-Abadie, M.~Mirza, B.~Xu, D.~Warde-Farley, S.~Ozair, A.~Courville, and Y.~Bengio, ``Generative adversarial networks,'' \emph{Commun. ACM}, vol.~63, no.~11, p. 139–144, Oct. 2020.

\bibitem{10.1145/3663364}
\BIBentryALTinterwordspacing
S.~Liang, Z.~Pan, w.~liu, J.~Yin, and M.~de~Rijke, ``A survey on variational autoencoders in recommender systems,'' \emph{ACM Comput. Surv.}, vol.~56, no.~10, Jun. 2024. [Online]. Available: \url{https://doi.org/10.1145/3663364}
\BIBentrySTDinterwordspacing

\bibitem{10333794}
Y.~Jiang, Y.~Wu, S.~Zhang, and J.~J. Yu, ``Fedvae: Trajectory privacy preserving based on federated variational autoencoder,'' in \emph{2023 IEEE 98th Vehicular Technology Conference (VTC2023-Fall)}, 2023, pp. 1--7.

\bibitem{10463181}
Y.~Xiao, X.~Li, T.~Li, R.~Wang, Y.~Pang, and G.~Wang, ``A distributed generative adversarial network for data augmentation under vertical federated learning,'' \emph{IEEE Transactions on Big Data}, vol.~11, no.~1, pp. 74--85, 2025.

\bibitem{10454003}
P.~Li, H.~Zhang, Y.~Wu, L.~Qian, R.~Yu, D.~Niyato, and X.~Shen, ``Filling the missing: Exploring generative ai for enhanced federated learning over heterogeneous mobile edge devices,'' \emph{IEEE Transactions on Mobile Computing}, vol.~23, no.~10, pp. 10\,001--10\,015, 2024.

\bibitem{10811919}
Z.~Meng, Z.~Li, X.~Hou, M.~Xu, Y.~Xia, Z.~Zhang, and S.~Song, ``Enhancing federated learning performance on heterogeneous iot devices using generative artificial intelligence with resource scheduling,'' \emph{IEEE Internet of Things Journal}, pp. 1--1, 2024.

\bibitem{9488839}
L.~Li, D.~Shi, R.~Hou, H.~Li, M.~Pan, and Z.~Han, ``To talk or to work:{ Flexible} communication compression for energy efficient federated learning over heterogeneous mobile edge devices,'' in \emph{IEEE Conference on Computer Communications (INFOCOM)}, Vancouver, BC, Canada, May, 2021, pp. 1--10.

\bibitem{10168747}
L.~Li, C.~Huang, D.~Shi, H.~Wang, X.~Zhou, M.~Shu, and M.~Pan, ``Energy and spectrum efficient federated learning via high-precision over-the-air computation,'' \emph{IEEE Transactions on Wireless Communications}, vol.~23, no.~2, pp. 1228--1242, 2024.

\bibitem{9916128}
R.~Chen, L.~Li, K.~Xue, C.~Zhang, M.~Pan, and Y.~Fang, ``Energy efficient federated learning over heterogeneous mobile devices via joint design of weight quantization and wireless transmission,'' \emph{IEEE Transactions on Mobile Computing}, vol.~22, no.~12, pp. 7451--7465, 2023.

\bibitem{10258354}
C.~Chen, B.~Jiang, S.~Liu, C.~Li, C.~Wu, and R.~Yin, ``Efficient federated learning in resource-constrained edge intelligence networks using model compression,'' \emph{IEEE Transactions on Vehicular Technology}, vol.~73, no.~2, pp. 2643--2655, 2024.

\bibitem{Li2020On}
\BIBentryALTinterwordspacing
X.~Li, K.~Huang, W.~Yang, S.~Wang, and Z.~Zhang, ``On the convergence of fedavg on non-iid data,'' in \emph{International Conference on Learning Representations}, 2020. [Online]. Available: \url{https://openreview.net/forum?id=HJxNAnVtDS}
\BIBentrySTDinterwordspacing

\bibitem{9598845}
S.~Liu, G.~Yu, R.~Yin, J.~Yuan, L.~Shen, and C.~Liu, ``Joint model pruning and device selection for communication-efficient federated edge learning,'' \emph{IEEE Transactions on Communications}, vol.~70, no.~1, pp. 231--244, 2022.

\bibitem{9210812}
M.~Chen, Z.~Yang, W.~Saad, C.~Yin, H.~V. Poor, and S.~Cui, ``A joint learning and communications framework for federated learning over wireless networks,'' \emph{IEEE Transactions on Wireless Communications}, vol.~20, no.~1, pp. 269--283, 2021.

\bibitem{stich2018sparsified}
S.~U. Stich, J.-B. Cordonnier, and M.~Jaggi, ``Sparsified sgd with memory,'' \emph{Advances in Neural Information Processing Systems (NeurIPS)}, vol.~31, December, 2018.

\bibitem{zheng2020design}
S.~Zheng, C.~Shen, and X.~Chen, ``Design and analysis of uplink and downlink communications for federated learning,'' \emph{IEEE Journal on Selected Areas in Communications}, vol.~39, no.~7, pp. 2150--2167, 2020.

\bibitem{wang2020tackling}
J.~Wang, Q.~Liu, H.~Liang, G.~Joshi, and H.~V. Poor, ``Tackling the objective inconsistency problem in heterogeneous federated optimization,'' \emph{Advances in neural information processing systems}, vol.~33, pp. 7611--7623, 2020.

\bibitem{wang2021quantized}
Y.~Wang, Y.~Xu, Q.~Shi, and T.-H. Chang, ``Quantized federated learning under transmission delay and outage constraints,'' \emph{IEEE Journal on Selected Areas in Communications}, vol.~40, no.~1, pp. 323--341, 2021.

\bibitem{krizhevsky2009learning}
A.~Krizhevsky, G.~Hinton \emph{et~al.}, ``Learning multiple layers of features from tiny images,'' 2009.

\bibitem{he2016deep}
K.~He, X.~Zhang, S.~Ren, and J.~Sun, ``Deep residual learning for image recognition,'' in \emph{Proceedings of the IEEE Conference on Computer Vision and Pattern Recognition (CVPR)}, 2016, pp. 770--778.

\bibitem{yang2023denoising}
S.~Yang, Y.~Chen, L.~Wang, S.~Liu, and Y.~Chen, ``Denoising diffusion step-aware models,'' \emph{arXiv preprint arXiv:2310.03337}, 2023.

\end{thebibliography}
\begin{IEEEbiography}[{\includegraphics[width=1.1in,height=1.33in]{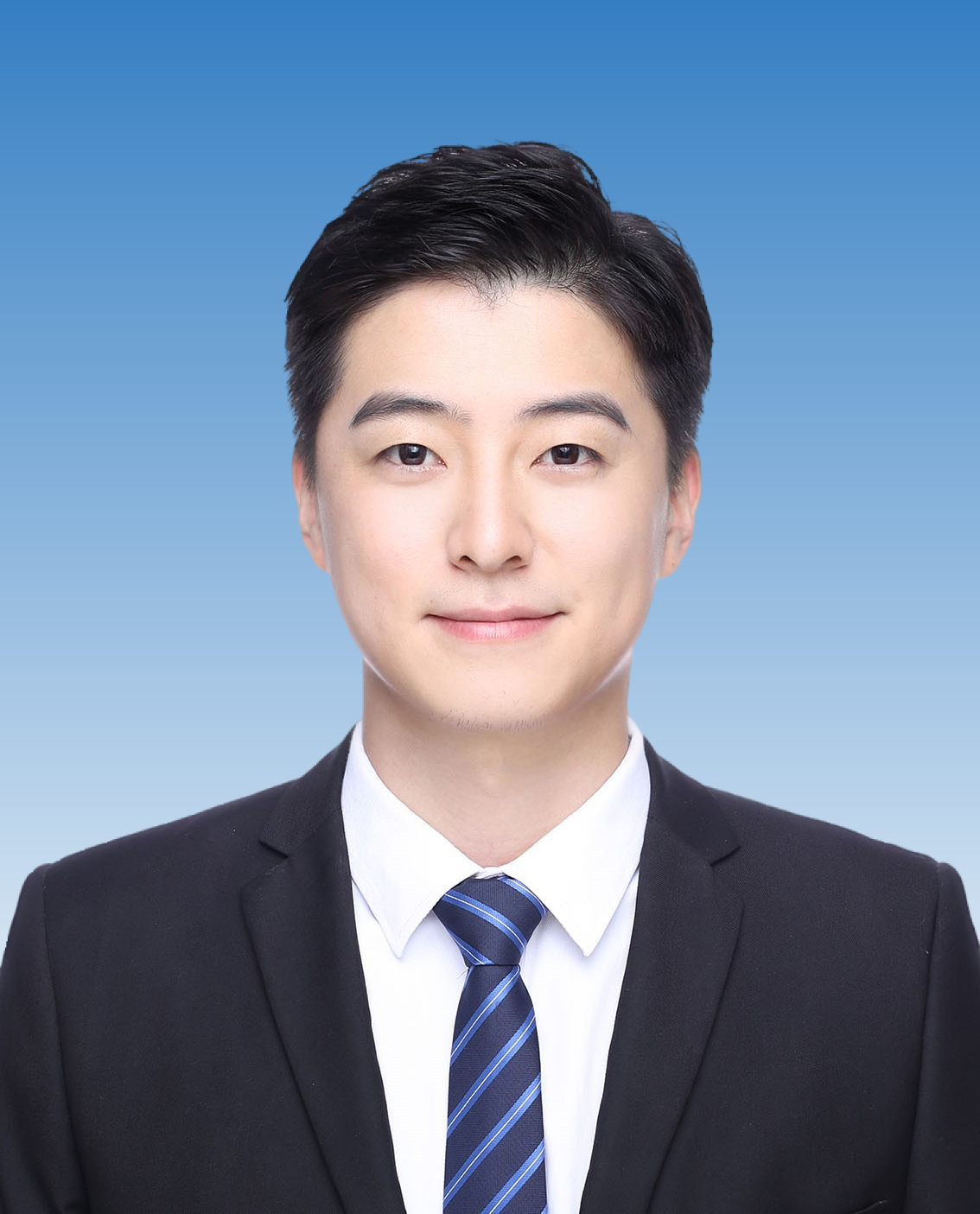}}]{\textbf{Xiangwang Hou}} (Member, IEEE) is currently a postdoctoral researcher in the Department of Electronic Engineering, Tsinghua University, Beijing, China. He received his B.S. degree from Shandong University of Technology in 2017, his M.S. degree from Xidian University in 2020, and his Ph.D. degree from Tsinghua University in 2024. From 2023 to 2024, he was a Joint Ph.D. student at the School of Computer Science and Engineering, Nanyang Technological University, Singapore, under the supervision of Prof. Dusit Niyato. From 2020 to 2021, he worked as an algorithm engineer at Huawei Technologies Co., Ltd. and at Tsinghua University. His research interests include edge intelligence, federated learning, wireless AI, and UAV/AUV networks. He was a recipient of the Best Paper Award from IEEE ICC.
  \end{IEEEbiography}
  
  \begin{IEEEbiography}[{\includegraphics[width=1.1in,height=1.33in]{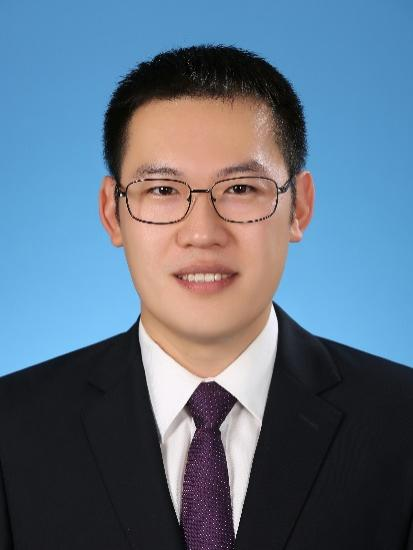}}]{\textbf{Jingjing Wang}} (Senior Member, IEEE) received his B.Sc. degree in electronic information engineering from Dalian University of Technology, Liaoning, China in 2014 and the Ph.D. degree in information and communication engineering from Tsinghua University, Beijing, China in 2019, both with the highest honors. From 2017 to 2018, he visited the next-generation wireless group chaired by Prof. Lajos Hanzo from the University of Southampton, UK. Dr. Wang is currently a Professor at the School of Cyber Science and Technology, Beihang University, Beijing, China, and also a researcher at Hangzhou Innovation Institute, Beihang University, Hangzhou, China. His research interests include AI-enhanced next-generation wireless networks, UAV networking, and swarm intelligence. He has published over 100 IEEE Journal/Conference papers. He is currently serving as an Editor for the IEEE Transactions on Vehicular Technology, IEEE Internet of Things Journal, and IEEE Wireless Communications Letters. Dr. Wang was a recipient of the Best Journal Paper Award of IEEE ComSoc Technical Committee on Green Communications \& Computing, the Best Paper Award of the IEEE ICC, and the IEEE IWCMC.
    \end{IEEEbiography}
    
  \begin{IEEEbiography}[{\includegraphics[width=1.1in,height=1.33in]{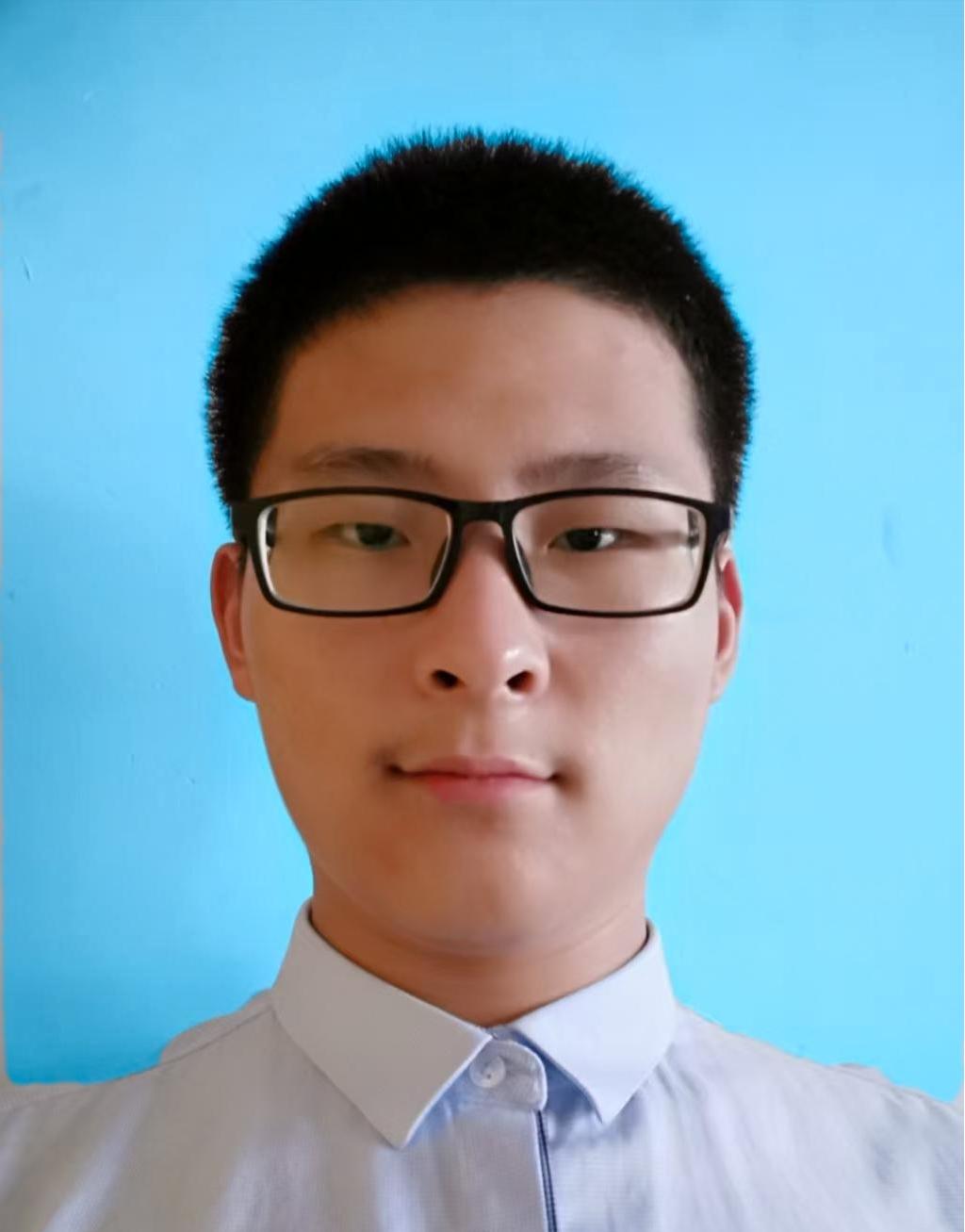}}]{\textbf{Fangming Guan}} (Student Member, IEEE) received the B.E. degree in Electronic Science and Technology from the School of Information and Electronics, Beijing Institute of Technology, Beijing, China, in 2023. He is currently pursuing the M.E. degree in Electronic Engineering with the Department of Electronic Engineering, Tsinghua University, Beijing, China. His research interests include federated learning and multi-modal large language models.
\end{IEEEbiography}

\begin{IEEEbiography}[{\includegraphics[width=1in,height=1.25in,clip,keepaspectratio]{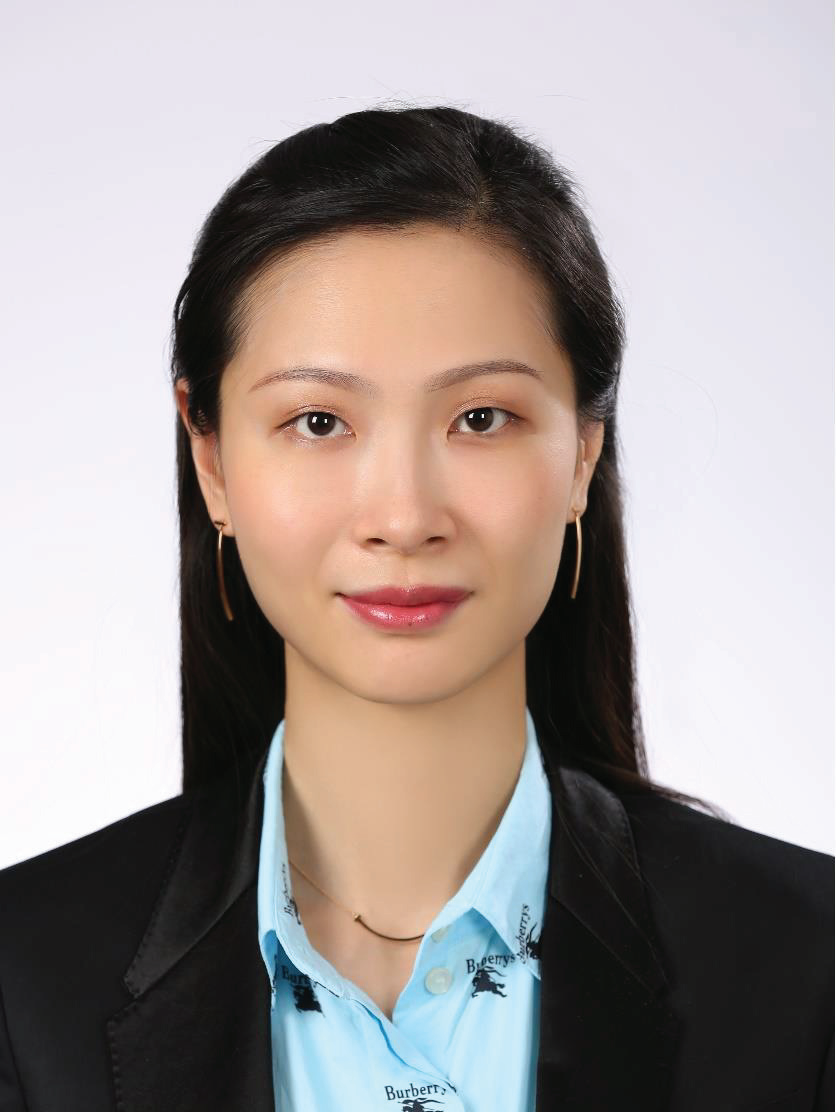}}]{Jun Du} (Senior Member, IEEE) received her B.S. in information and communication engineering from Beijing Institute of Technology, in 2009, and her M.S. and Ph.D. in information and communication engineering from Tsinghua University, Beijing, in 2014 and 2018, respectively. From Oct. 2016 - Sept. 2017, Dr. Du was a sponsored researcher, and she visited Imperial College London. Currently, she is an associate professor in the Department of Electrical Engineering, Tsinghua University. Her research interests are mainly in communications, networking, resource allocation, and system security problems of heterogeneous networks and space-based information networks. Dr. Du is the recipient of the Best Student Paper Award from IEEE GlobalSIP in 2015, the Best Paper Award from IEEE ICC 2019 and 2025, and the Best Paper Award from IWCMC in 2020.
\end{IEEEbiography}

  \begin{IEEEbiography}[{\includegraphics[width=1.1in,height=1.33in]{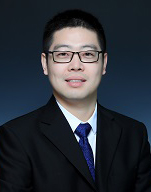}}]{\textbf{Chunxiao Jiang}} (Fellow, IEEE) is currently an associate professor in Beijing National Research Center for Information Science and Technology, Tsinghua University. He received the B.S. degree in Information Engineering from Beihang University, Beijing in 2008 and the Ph.D. degree in Electronic Engineering from Tsinghua University, Beijing in 2013, both with the highest honors. From 2011 to 2012 (as a Joint Ph.D) and 2013 to 2016 (as a Postdoc), he was in the Department of Electrical and Computer Engineering at the University of Maryland College Park under the supervision of Prof. K. J. Ray Liu. His research interests include the application of game theory, optimization, and statistical theories to communication, networking, and resource allocation problems, in particular space networks and heterogeneous networks. Dr. Jiang has served as an Editor of IEEE Transactions on Communications, IEEE Internet of Things Journal, IEEE Wireless Communications, IEEE Transactions on Network Science and Engineering, IEEE Network, IEEE Communications Letters, and a Guest Editor of IEEE Communications Magazine, IEEE Transactions on Network Science and Engineering, and IEEE Transactions on Cognitive Communications and Networking. He has also served as a member of the technical program committee as well as the Symposium Chair for a number of international conferences. Dr. Jiang is the recipient of the Best Paper Award from IEEE GLOBECOM in 2013, IEEE Communications Society Young Author Best Paper Award in 2017, the Best Paper Award from ICC 2019, IEEE VTS Early Career Award 2020, IEEE ComSoc Asia-Pacific Best Young Researcher Award 2020, IEEE VTS Distinguished Lecturer 2021, and IEEE ComSoc Best Young Professional Award in Academia 2021. He received the Chinese National Second Prize in the Technical Inventions Award in 2018. He is a Fellow of IEEE and IET.
    \end{IEEEbiography}

 \begin{IEEEbiography}[{\includegraphics[width=1.1in,height=1.33in]{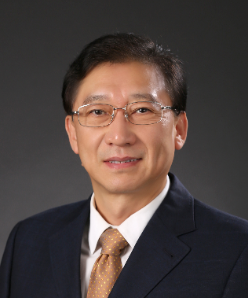}}]{\textbf{Yong Ren}} (Senior Member, IEEE) is a full professor of the Department of Electronic Engineering and serves as the director of the Complexity Engineered Systems Lab at Tsinghua University.  Moreover, he is also a guest professor of the Network and Communication Research Center at Peng Cheng Laboratory.  He received his B.S., M.S., and Ph.D. degrees in Electronic Engineering from Harbin Institute of Technology, China, in 1984, 1987, and 1994, respectively. He has published over 400 technical papers in the fields of computer networks and mobile telecommunication networks and has served as a reviewer for more than 40 international journals and conferences. His current research interests include marine information networks, swarm intelligence, and wireless AI. He has received multiple Best Paper Awards from IEEE journals and international conferences.
 \end{IEEEbiography}

\end{sloppypar}
\end{document}